\documentclass[11pt]{article}
\usepackage[T1]{fontenc}         
\usepackage{amsfonts}      
\usepackage{nicefrac}      
\usepackage{microtype}   
\usepackage[colorlinks,
            linkcolor=red,
            anchorcolor=blue,
            citecolor=blue, 
            pagebackref=true
           ]{hyperref}

\usepackage{fullpage}
\usepackage{tabularx}
\usepackage{float}
\usepackage{enumitem}
\usepackage{comment}
\usepackage{mmll}
\usepackage[colorinlistoftodos, textsize=tiny, textwidth=2.2cm, backgroundcolor=blue!10, linecolor=magenta, bordercolor=magenta]{todonotes}

\newcommand{\rev}[1]{\textcolor{black}{#1}}
\newcommand{\icml}[1]{\textcolor{black}{#1}}

\let\textstyle\relax
\allowdisplaybreaks
\title{Provable Anytime Ensemble Sampling Algorithms in Nonlinear Contextual Bandits}

\author{
    Jiazheng Sun\thanks{Duke University; e-mail: {\tt jiazheng.sun@duke.edu}}~\footnotemark[3]
    ~~
    Weixin Wang\thanks{Duke University; e-mail: {\tt weixin.wang@duke.edu}}~\thanks{Equal contribution} 
    ~~
    Pan Xu\thanks{Duke University; e-mail: {\tt pan.xu@duke.edu}}
}

\begin{document}

\date{}
\maketitle

\begin{abstract}
We provide a unified algorithmic framework for ensemble sampling in nonlinear contextual bandits and develop corresponding regret bounds for two most common nonlinear contextual bandit settings: Generalized Linear Model Ensemble Sampling (GLM-ES) for generalized linear contextual bandits and Neural Ensemble Sampling (Neural-ES) for neural contextual bandits. Both methods maintain multiple estimators for the reward model parameters via maximum likelihood estimation on randomly perturbed data. We prove high-probability frequentist regret bounds of $\widetilde{\mathcal{O}}(d^{3/2} \sqrt{T} + d^{4})$ for GLM-ES and $\widetilde{\mathcal{O}}(\widetilde{d}^{3/2} \sqrt{T})$ for Neural-ES, where $d$ is the dimension of feature vectors, $\widetilde{d}$ is the effective dimension of a neural tangent kernel (NTK) matrix and $T$ is the number of rounds. The regret bound of GLM-ES matches the state-of-the-art result of randomized exploration algorithms in generalized linear bandit setting. In the theoretical analysis, we introduce techniques that address challenges specific to nonlinear models. Practically, we remove fixed-time horizon assumption by developing anytime versions of our algorithms, suitable when $T$ is unknown. Finally, we empirically evaluate GLM-ES, Neural-ES and their anytime variants, demonstrating strong performance. Overall, our results establish ensemble sampling as a provable and practical randomized exploration approach for nonlinear contextual bandits.
\end{abstract}

\section{Introduction}

The contextual bandit is an online learning problem where an agent interacts with an environment by pulling arms, each associated with a feature vector. After each pull, the agent receives a stochastic reward whose expected value depends on the chosen arm’s feature vector. The agent’s goal is to maximize the accumulated reward. Contextual bandits provide a natural abstraction for real-world sequential decision-making problems such as content recommendation \citep{zhu2023deep} and clinical trials \citep{varatharajah2022contextual}. To maximize rewards, the agent must learn the mapping from an arm’s feature vector to its expected reward. Most prior work has focused on the linear contextual bandit setting \citep{abbasi2011improved,abeille2017linear,kveton2020randomized}, where the expected reward is assumed to be a linear function of the feature vector. While this assumption facilitates theoretical analysis and efficient implementations, it fails to capture complex relationships between features and rewards. This has motivated the study of nonlinear contextual bandits, where the expected reward is modeled as a nonlinear function of the arm’s features, e.g., through generalized linear models (GLMs) or neural networks. In the GLM setting \citep{filippi2010parametric,kveton2020randomized}, the reward is generated by applying a nonlinear function $\mu(\cdot)$ to the inner product of the feature vector and an unknown parameter vector. In more general cases where the reward cannot be expressed within the GLM structure, neural contextual bandits approximate the reward function with deep neural networks, without assuming any particular functional form \citep{jacot2018neural,zhou2020neural,zhang2021neural,xu2022neural,jia2022learning}. These nonlinear approaches have substantially improved the empirical performance of linear bandit algorithms, especially in complex environments \citep{xu2022langevin,jia2022learning}.

While nonlinear models enhance expressivity, they also complicate the design of effective exploration strategies. Existing exploration methods such as Upper Confidence Bound (UCB) \citep{abbasi2011improved} and Thompson Sampling (TS) \citep{agrawal2013thompson} heavily depend on reward structure and distributional assumptions. Extending them to nonlinear settings requires significant approximations \citep{zhou2020neural,zhang2021neural,xu2022neural,xu2022langevin}, often making the methods impractical in real-world applications. This motivates the search for exploration strategies that combine strong empirical performance with theoretical guarantees in nonlinear bandits. Ensemble sampling \citep{lu2017ensemble} has emerged as a promising class of algorithms for online decision-making problems, including bandits \citep{lee2024improved}, deep reinforcement learning \citep{osband2016deep}, and recommendation systems \citep{zhu2023deep}. Ensemble sampling maintains an ensemble of $m$ models, each trained on randomly perturbed historical data comprising arm features and rewards. At each round, one model is selected to estimate the expected rewards, and the arm with the highest estimate is chosen. Subsequently, the received reward is incorporated into the dataset, along with a new data point that has undergone random perturbation. All models are then updated accordingly.

In recent years, ensemble sampling has gained popularity due to its strong empirical performance and moderate computational cost. However, theoretical understanding has lagged behind. Existing analysis provide regret guarantees only in the linear contextual bandit setting. For example, \cite{pmlr-v336-janz26a} established a regret bound of $\widetilde{\mathcal{O}}(d^{3/2}\sqrt{T})$ for infinitely many arms with ensemble size $\Theta(d \log T)$, \cite{lee2024improved} proved a high-probability $T$-round regret bound of $\widetilde{\mathcal{O}}(d^{3/2}\sqrt{T})$ with ensemble size $\Omega(K \log T)$, where $d$ is the feature dimension and $K$ is the number of arms. While these theoretical works provide valuable insights, they are far from fully elucidating the empirical success of ensemble sampling in complex decision-making applications, where reward models are typically nonlinear in the arm features.

In this work, we extend ensemble sampling to nonlinear bandit settings with finitely many arms. Specifically, we study the two most widely used models: generalized linear bandits and neural contextual bandits. We show that ensemble sampling algorithms are provable in these nonlinear settings and \texttt{GLM-ES} achieves high-probability regret bounds matching the state-of-the-art for randomized exploration algorithms, with ensemble size logarithmic in $T$. In addition, we develop anytime versions of ensemble sampling, addressing the limitation that ensemble size and several other hyper-parameters traditionally depend on the horizon $T$. These anytime variants significantly broaden the applicability of ensemble sampling. Finally, we complement our theoretical results with  experiments on \texttt{Lin-ES}, \texttt{GLM-ES}, and \texttt{Neural-ES} against baselines. Our experiments highlight the practicality of ensemble sampling in nonlinear bandits, balancing strong performance with computational efficiency.

Our contributions are summarized as follows. 
\begin{itemize}[nosep, leftmargin=*]
\item We propose a general framework for ensemble sampling in bandit problems and introduce \texttt{GLM-ES} and \texttt{Neural-ES} as its realizations in nonlinear settings.
\item We provide theoretical analysis of \texttt{GLM-ES} and \texttt{Neural-ES}, proving high-probability regret bound of $\widetilde{\mathcal{O}}(d^{3/2} \sqrt{T} + d^{4})$ and $\widetilde{\mathcal{O}}\big(\widetilde{d}^{3/2} \sqrt{T}\big)$, respectively. The regret bound for \texttt{GLM-ES} matches the state-of-the-art for randomized exploration algorithms. To the best of our knowledge, these are the first high-probability regret bounds for ensemble sampling in nonlinear bandit settings.
\item For generalized linear bandits, we optimize the warm-up procedure in existing literature, reducing regret from $d^9$ to $d^{4}$. This improvement also applies to perturbed-history type of exploration strategies. We further remove the need for adaptive reward perturbations, simplifying the design and improving the efficiency of our algorithm.
\item We develop anytime variants of ensemble sampling and prove that their asymptotic cumulative regret guarantees are preserved.
\item We conduct empirical evaluations comparing cumulative regret and computational cost with baselines, demonstrating the practicality of ensemble sampling and their anytime variants.
\end{itemize}

\section{Related Work}

In this section, we briefly summarize the recent works in nonlinear contextual bandit settings with randomized exploration algorithms.

\paragraph{Randomized Exploration}

Randomized exploration strategies add controlled randomness to promote exploration of actions with high uncertainty. In sequential decision making problems, randomized exploration strategies often outperform deterministic strategies such as Upper Confidence Bound (UCB) \citep{chu2011contextual,lattimore2020bandit} by preventing early convergence to suboptimal actions \citep{jin2021mots,jin2023thompson}. Among such methods, Thompson Sampling (TS) \citep{thompson1933likelihood} is a key approach for multi-armed bandits \citep{agrawal2017near}, contextual bandits \citep{agrawal2013thompson}, and RL \citep{osband2013more}. TS maintains a posterior over model parameters, updated each round from a prior (e.g., Gaussian) and observed rewards \citep{agrawal2013thompson}, and samples a parameter from this posterior for arm selection. Despite its simplicity, many TS variants rely on exact posteriors or accurate Laplace approximations, which can be costly. To address this, approximate sampling methods such as Langevin Monte Carlo (LMC) \citep{xu2022langevin,hsu2024randomized}, Stochastic gradient Langevin dynamics (SGLD) variants \citep{mazumdar2020approximate,zheng2024accelerating} and variational inference \citep{pmlr-v235-clavier24a} have been developed and applied to various problem settings, including multi-armed bandits with non-conjugate or highly nonlinear rewards, nonlinear contextual bandits and RL \citep{ishfaq2024provable,ishfaq2024more,hsu2024randomized}. Another important method is perturb-history exploration (PHE) method, which involves introducing random perturbations in the historical data to approximate posterior sampling, making it applicable to complex reward distributions \citep{kveton2020randomized, ishfaq2021randomized}. Ensemble sampling maintains a small set of independently perturbed model replicas and selects arms using a randomly chosen replica \citep{lu2017ensemble}. Follow-up work provided theory for the linear contextual bandit setting. \cite{qin2022analysis} gave the first regret bound. \cite{janz2024ensemble} tightened guarantees with an ensemble of size $\Theta(d\log T)$ for linear bandits with infinitely many arms. \texttt{Lin-ES} \citep{lee2024improved} further improved the regret to $\widetilde O(d^{3/2}\sqrt{T})$ and clarified its connection to \texttt{Lin-PHE}.

\paragraph{Generalized Linear Bandits}

Generalized linear contextual bandits model rewards via a link function of a linear predictor, extending linear bandits to a more general setting. Early work introduced \texttt{GLM-UCB} and proved regret guarantees under standard regularity conditions \citep{filippi2010parametric}. Subsequent advances focused on optimality and efficiency. \cite{li2017provably} gave provably optimal algorithms with refined confidence sets. \cite{ding2021efficient} combined online stochastic gradient updates with Thompson Sampling for scalable inference. \cite{kveton2020randomized} developed randomized exploration for GLMs with sharper analyses. Perturbation-based methods provide practical alternatives: linearly perturbed loss minimization yields simple, sampling-free exploration with strong guarantees \citep{janz2024exploration}, and PHE adapts perturb-history exploration to sub-Gaussian GLMs \citep{liu2023glm}. \cite{sawarni2024generalized} analyze GLMs under limited adaptivity (batched policies) with communication and deployment constraints. Anytime-valid confidence sequences for GLMs enable principled UCB/TS decisions and valid sequential inference \citep{lee2024a}. At large horizons, One-pass update methods achieve near-optimal regret with single-pass, low-memory updates \citep{zhang2025generalized}.

\paragraph{Neural Bandits}

Neural bandits combine deep neural networks (DNNs) with contextual bandit algorithms. This setting leverages the representation power of DNNs and insights from neural tangent kernel (NTK) theory \citep{jacot2018neural}. \texttt{Neural-UCB} \citep{zhou2020neural} builds confidence sets using DNN-derived random features to enable UCB-style exploration. \texttt{Neural-TS} \citep{zhang2021neural} extends Thompson Sampling to this setting by using a neural estimator to approximate the posterior over rewards. \texttt{Neural-LCB} \citep{nguyen-tang2022offline} studies offline neural bandits and uses a neural lower confidence bound to take pessimistic decisions under uncertainty. To reduce the compute burden of explicit exploration, \texttt{Neural-PHE} \citep{jia2022learning} learns a neural bandit model with perturbed rewards, avoiding separate exploration updates. Subsequent work explores added networks for exploitation \citep{ban2022eenet}, provable guarantees with smooth activations \citep{salgia2023provably}, and extensions to combinatorial selection \citep{hwang2023combinatorial, atalar2025neural, wang2025neural}, where the learner selects a subset (e.g., multiple arms under constraints) each round and receives a corresponding reward.
Recently, \texttt{NeuSquareCB} and \texttt{NeuFastCB} as proposed in \citep{deb2024contextual} use randomly sampled arm in the arm-selection process to implement exploration, letting exploration randomness enter through the arm-selection distribution. The algorithms introduced in this section are important baselines for the neural contextual bandit setup.

\section{Problem Setting}

Contextual bandits form a broad class of sequential decision making problems where the player chooses an action from an observed action set based on interaction history. Each action is associated with a feature vector (context).
At round $t$, the player observes an action set $\cX_t\subseteq\RR^d$, where we assume for any $X \in \cX_t$, $\|X\|_2 \leq 1$.
The agent then selects an arm (action) $X_t\in\cX_t$ and the environment immediately reveals a reward $Y_t$.
In this work, we consider the setting that the action set is finite with $K$ arms and fixed across different rounds. Since we are considering a fixed arm set in this work, for simplicity, we use $\mathcal{X}$ to denote the fixed arm set.
{We assume that the reward $Y_{t}$ is generated by an unknown distribution with mean $h(X): \RR^d \rightarrow \RR$.}
In general, the form of $h(X)$ is unknown. One important special case is that we set reward model to be a linear function $h(X)=X^\top\theta^*$, then $h(X)$ is parameterized using a vector $\theta^* \in\RR^d$ and we have the standard linear contextual bandit setting~\citep{chu2011contextual,abbasi2011improved,agrawal2013thompson}.

In this work, we consider nonlinear contextual bandits where the reward model $h(X)$ is a nonlinear function of feature vector $X$. We focus on two most common nonlinear settings: (1) generalized linear bandits with $h(X)=\mu\big(X^\top\theta_{\text{GLM}}^*\big)$, where $\theta_{\text{GLM}}^{*} \in \RR^{d}$ is the true parameter with $\|\theta_{\text{GLM}}^{*}\|_{2} \leq S$ and $\mu(\cdot)$ is a strictly increasing link function~\citep{li2017provably, kveton2020randomized}; %
(2) neural contextual bandits where no assumptions are made about the form of $h(X)$, we use a neural network $f(X; \theta_{\text{Neural}})$ to approximate $h(X)$, where $\theta_{\text{Neural}}$ is the concatenation of all weight parameters in the neural network and its dimension is determined by the structure of the neural network~\citep{zhou2020neural,zhang2021neural,jia2022learning}.

The goal of a bandit algorithm is to maximize the cumulative reward over a horizon $T$, equivalently to minimize the pseudo-regret~\citep{lattimore2020bandit}
\begin{align}\label{def:pseudo_regret}
    \rev{R(T)=\sum_{t=1}^{T}\big(h(X^*)-h(X_t)\big),}
\end{align}
where $X_t\in\cX$ is the arm played at round $t$, and $\rev{X^*=\argmax_{X\in\cX}h(X)}$ is the arm with the highest expected reward. To minimize the cumulative regret $R(T)$, the agent needs to collect information and learn the true reward model $h(X)$ from interactions with the environment.

\paragraph{Notations}

We adopt the following standard notations throughout this paper. The set $\{1, 2, ..., n\}$ is denoted by $[n]$.
For any positive semi-definite matrix $M$, we use $\lambda_{\text{max}}(M)$ and $\lambda_{\text{min}}(M) \geq 0$ to denote maximum and minimum eigenvalues of $M$.
The 2-norm of a symmetric matrix $M$ is defined as $\| M \|_{2} = \max_{i} | \lambda_{i}(M) |$. For any positive semi-definite matrices $M_{1}$ and $M_{2}$, $M_{1} \preceq M_{2}$ if and only if $X^{\top} M_{1} X \leq X^{\top} M_{2} X$ for all $X \in \mathbb{R}^{d}$. All vectors are column vectors. For any vector $X$, we use the following vector norms: $|| X ||_{2} = \sqrt{X^{\top}X}$, $|| X ||_{M} = \sqrt{X^{\top} M X}$. The indicator function that event $\mathcal{E}$ occurs is $\ind \{\mathcal{E}\}$.
We use $\widetilde{\mathcal{O}}$ for the big-O notation up to logarithmic factors.
We define the filtration $\mathcal{F}_{t}^{\prime} := \sigma(X_{1},...,X_{t},Y_{1},...,Y_{t})$ as the $\sigma$-algebra generated by the pulled arms and observed rewards by the end of round $t\in [T]$, $\mathcal{F}_{t} := \sigma\big(X_{1},...,X_{t},Y_{1},...,Y_{t}, j_{1},...,j_{t}, \{Z_{1}^{j}\}_{j=1}^{m}, ...,  \{Z_{t}^{j}\}_{j=1}^{m}\big)$ as the $\sigma$-algebra generated by the pulled arms, observed rewards, chosen model and added perturbations by the end of round $t\in [T]$, and $\mathcal{F}_{t}^{-} := \sigma\big(X_{1},...,X_{t},Y_{1},...,Y_{t-1}, j_{1},...,j_{t}, \{Z_{1}^{j}\}_{j=1}^{m}, ...,  \{Z_{t-1}^{j}\}_{j=1}^{m}\big)$ as the $\sigma$-algebra before observing the reward at round $t\in [T]$.

\section{Ensemble Sampling for Nonlinear Contextual Bandits}

We apply the design principle of ensemble sampling to nonlinear contextual bandits, extending previous works on linear ensemble sampling (\texttt{Lin-ES}) to broader applications. In this section, we first present a unified algorithm framework for ensemble sampling in nonlinear contextual bandit, then we focus on two common nonlinear cases: 1) generalized linear contextual bandits; 2) neural contextual bandits, and respectively provide algorithms \texttt{GLM-ES} and \texttt{Neural-ES}.

\subsection{Unified Algorithm Framework}

Ensemble sampling follows the randomized exploration principle and exploration is realized through adding perturbations to observed rewards.
Therefore, the perturbed history $\mathcal{D}_{t} = \{(X_{l}, Y_{l} + Z_{l})\}_{l=1}^{t}$ is utilized to estimate the true mean reward model $h(X)$ and choose the best arm, where $\{X_{l}\}_{l=1}^{t}$ are pulled arms, $\{Y_{l}\}_{l=1}^{t}$ are observed rewards and $\{Z_{l}\}_{l=1}^{t}$ are perturbations. %

In particular, we maintain an ensemble of $m$ perturbed models, each with different perturbed history $\mathcal{D}_{t}^{j}$, where $j \in [m]$ is the model index. At each round, we randomly select one model to estimate the mean reward of each arm in the arm set, then select the arm $X_{t}$ that maximizes the estimated mean reward. We use $f_{t}(X)$ to denote the estimated mean reward of arm $X$ from the chosen model $j_{t}$ at round $t$. After observing the reward $Y_{t}$ from the environment, each model in the ensemble is updated incrementally based on history $\mathcal{D}_{t}^{j}$.
At each round, we only sample one perturbation $Z_{t}^{j}$ for each model $j \in[m]$, the previous perturbations $\{Z_{l}^{j}\}_{l=1}^{t-1}$ are not re-sampled.

\begin{remark}
The design of ensemble sampling is similar to that of perturb-history exploration (PHE)-based algorithms. In PHE-based algorithms \citep{kveton2020perturbed,kveton2020randomized}, we only keep one model, but the entire perturbation sequence $\{Z_{l}\}_{l=1}^{t}$ is freshly sampled at each round. As a result, the per-round computational cost due to sampling increases linearly in $t$ and accumulates to $\Theta\big(T^{2}\big)$ after $T$ rounds, the algorithm becomes impractical for large $T$. Ensemble sampling can significantly reduce the computational cost by keeping previous perturbations, the per-round sampling cost remains constant for any $t$.
\end{remark}

A unified algorithmic framework is given in \Cref{alg:es_nonlinear}. For the generalized linear and neural bandit settings, we specify (i) a warm-up/initialization step (\icml{Line 2}) and (ii) a parameter-estimation loss function $L_{\text{nonlin}}$ (Line 9). The following sections detail these two algorithmic instantiations.

\begin{algorithm}[H]
\caption{Ensemble Sampling \icml{Framework} for Nonlinear Contextual Bandits}
\begin{algorithmic}[1]  \label{alg:es_nonlinear}
\STATE \textbf{Input:} ensemble size $m$, regularization parameter $\lambda$, reward-perturbation distribution $\mathcal{P}_{R}$ on $\mathbb{R}$, number of warm-up exploration rounds $\tau$ {and warm-up strategy}
\STATE Initialization and warm-up for the first $\tau$ rounds \hfill $\triangleleft$ \textcolor{red}{GLM-ES} or \textcolor{red}{Neural-ES} \label{line:warm_up}

\FOR{$t = \tau + 1, ..., T$}
    \STATE Sample $j_{t}$ uniformly from $\text{[}m\text{]}$
    \STATE Pull arm $X_{t} \leftarrow \text{argmax}_{X\in\mathcal{X}} f_{t}(X)$ and receive reward $Y_{t}$
    \FOR{$j=1, 2, ..., m$}
        \STATE Sample $Z_{t}^{j} \sim \mathcal{P}_{R}$

        \STATE Update \icml{$\mathcal{D}_{t}^{j} \leftarrow \mathcal{D}_{t-1}^{j} \cup \{\big(X_{t}, Y_{t} + Z_{t}^{j}\big)\}$}
        
        \STATE $\icml{\theta_{t}^{j}} \leftarrow \text{argmin}_{\theta} \, L_{\text{nonlin}}(\theta {, \lambda}; \mathcal{D}_{t}^{j})$ \hfill $\triangleleft$ \textcolor{red}{GLM-ES} or \textcolor{red}{Neural-ES} \label{line:theta_estimate}
    \ENDFOR
\ENDFOR
\end{algorithmic}
\end{algorithm}

\subsection{Ensemble Sampling for Generalized Linear Model (GLM-ES)}

We first consider the generalized linear model (GLM) and provide the algorithm design of \texttt{GLM-ES}. %
We make the following assumption on the true reward model.
\begin{assumption}[Exponential family distribution]
\label{assumption:exponential_family}
Given feature vector $X \in \RR^{d}$, the conditional distribution of reward $Y \in \mathbb{R}$ follows exponential family distribution:
\begin{align}
    \label{equ:exponential_family_main}
    p_{\theta^{*}_{\text{GLM}}}(Y \vert X) =
    b_{0}(Y) \, \text{exp} \big[Y\cdot X^{\top}\theta^{*}_{\text{GLM}} - b\big(X^{\top}\theta^{*}_{\text{GLM}}\big)\big],
\end{align}
where $\theta^{*}_{\text{GLM}} \in \RR^{d}$ is the true parameter.
We assume the link function $\mu (\cdot)$ is known to the agent, where the link function is defined as
\begin{align*}
    \mu\big(X^{\top} \theta\big) := \mathbb{E}_{\theta}\big[Y|X\big] = \dot{b}\big(X^{\top} \theta\big).
\end{align*}
\end{assumption}

\paragraph{Generalized Linear Model with Regularization} 
{According to \Cref{assumption:exponential_family}, reward $Y_{t}$ is generated from exponential family distribution with conditional expected value $\mu\big(X_{t}^{\top} \theta_{\text{GLM}}^{*}\big)$, where the link function $\mu (\cdot)$ is known to the agent.}
Given observed data set $\mathcal{D}_{t} = \{(X_{l}, Y_{l})\}_{l=1}^{t}$, the $\lambda$-regularized negative log-likelihood of $\mathcal{D}_{t}$ under parameter $\theta$ is defined as follows:
\begin{align}
\label{equ:regularized_negative_log_likelihood}
    L_{\text{GLM}}(\theta; \mathcal{D}_{t}) := \frac{\lambda}{2}\|\theta\|_{2}^2-\sum_{l=1}^t\big(Y_l \cdot X_l^{\top} \theta-b(X_l^{\top} \theta)\big).
\end{align}

\paragraph{Algorithm Design}
In \texttt{GLM-ES}, we maintain an ensemble of $m$ models (estimators), each model is parametrized by parameter $\theta_{t}^{j}$, $j \in [m]$, which is an estimation of $\theta^{*}_{\text{GLM}}$ based on perturbed history at the end of round $t$.
We start the \texttt{GLM-ES} algorithm by sampling the initial perturbation $Z_{0}^{j} \sim \mathcal{P}_{I}$ on $\mathbb{R}^{d}$ for each ensemble member $j \in [m]$.
We then use the $\lambda$-regularized negative log-likelihood to obtain the parameter estimation,
\begin{align}  \label{equ:glm_es_estimate}
    \theta_{t}^{j} :=
    \text{argmin}_{\theta \in \mathbb{R}^{d}}
    \Big(L_{\text{GLM}}\big(\theta; \{X_{l}, Y_{l} + Z_{l}^{j}\}_{l=1}^{t}\big) - \sqrt{\lambda / \dot{\mu}_{\min}} \langle Z_{0}^{j}, \theta \rangle\Big),
\end{align}
where $Z_{l}^{j} \in \mathbb{R}$ are perturbations i.i.d. sampled from distribution $\mathcal{P}_{R}$, $Z_{0}^{j} \in \mathbb{R}^{d}$ is initial perturbation sampled from distribution $\mathcal{P}_{I}$, $\dot{\mu}_{\min} > 0$ is the minimum value of the derivative of the link function and is known to the agent.
We use a warm-up procedure that approximates a G-optimal design, which is detailed as \Cref{alg:warm_up_glm} in \Cref{appendix_sub:warm_up_glm}.

\begin{remark}
The algorithm \texttt{GLM-ES} is an application of ensemble sampling in the GLM setting. Compared to \texttt{Lin-ES}, a warm-up procedure \Cref{alg:warm_up_glm} is required to guarantee that optimism is satisfied with constant probability. In the generalized linear bandit literature, the number of rounds in warm-up procedure is typically chosen by enforcing a lower bound on the minimum eigenvalue of the empirical feature covariance matrix \citep{li2017provably, kveton2020randomized, liu2023glm}. However, most works do not specify the required order for the number of warm-up rounds, effectively treating it as an assumption. We propose a practical warm-up scheme that directly controls the uncertainty level (see \Cref{lem:warm_up_procedure} for details). The same procedure is also used in \cite{liu2023glm}. Additionally, compared to the algorithm design in \cite{liu2023glm}, we removed the requirement for adapted perturbation on rewards and reduced requirements on the number of rounds in the warm-up procedure, making the algorithm simpler and more efficient.
\end{remark}

\subsection{Ensemble Sampling for Neural Contextual Bandit (Neural-ES)}
\label{section_sub:neural_es_design}

We now introduce neural contextual bandit setting and introduce the algorithm \texttt{Neural-ES}.

\paragraph{Neural Contextual Bandit}  
In neural contextual bandit, we only assume that the true mean reward $h(\cdot)$ is bounded, and we use a deep neural network (DNN) to approximate $h(\cdot)$.
We adopt the following fully connected neural network $f(X; \theta)$ to approximate $h(X)$:
\begin{align*}
    f(X; \theta) = \sqrt{N} \, W_{L} \phi \big( W_{L - 1} \phi \big( \cdot\cdot\cdot \phi(W_{1}X) \big) \big),
\end{align*}
where $N$ and $L$ are the width and depth of the neural network, $\phi(x) = \text{ReLU} (x)$, $W_{l}$ are learnable parameter matrices, and $\theta = [\text{vec}(W_{1}), \cdot\cdot\cdot, \text{vec}(W_{L})] \in \mathbb{R}^{d'}$ is the concatenation of all learnable parameters.
The dimension of $\theta$ is $d' = N + Nd + N^{2}(L-2)$.
For simplicity of theoretical analysis, we design the neural network such that each layer has the same width $N$.

To learn the parameters in the neural network, we define the following $\lambda$-regularized loss function:
\begin{align*}
    L_{\text{Neural}}(\theta; \mathcal{D}_{t}) = \frac{1}{2} & \sum_{l=1}^{\icml{t}} \big( f(X_{l}; \theta) - Y_{l} \big)^{2}
    + \frac{1}{2} \lambda N \, ||\theta - \theta_{0}||_{2}^{2},
\end{align*}
where parameter $\theta_{0}$ is randomly sampled at initialization. The parameter $\theta$ is estimated using gradient descent to minimize the loss function. We use learning rate $\eta$ and number of steps $J$ in the gradient descent in neural network learning.

\paragraph{Algorithm Design}
We follow the unified algorithm framework (\Cref{alg:es_nonlinear}) to design \texttt{Neural-ES}.
We maintain $m$ different models to approximate the true mapping $h (\cdot)$, each model is a deep neural network with the same structure.
{We first initialize the neural network $\theta_{0} = [\text{vec}(W_{1}^{0}), \cdot\cdot\cdot, \text{vec}(W_{L}^{0})] \in \mathbb{R}^{d'}$ using random parameters sampled from Gaussian distribution: for $2 \leq l \leq L-1$, $W_{l}^{0} = (W, 0; 0, W)$ with each entry in $W \in \mathbb{R}^{N/2 \times N/2}$ sampled from $\mathcal{N}(0, 4/N)$; $W_{1}^{0} = (W, 0; 0, W)$ with each entry in $W \in \mathbb{R}^{N/2 \times d/2}$ sampled from $\mathcal{N}(0, 4/N)$; $W_{L}^{0} = (\mathbf{w}^{\top}, -\mathbf{w}^{\top})$ with each entry in $\mathbf{w} \in \mathbb{R}^{1 \times N/2}$ sampled from $\mathcal{N} (0, 2/N)$.}
Note that the initialization $\theta_{0}$ is shared across all models in the ensemble.

We then perform a simple warm-up by pulling each arm once.
At each round \icml{$t$}, we uniformly randomly choose one model $j_{t}$ from the ensemble and choose arm $X_{t}$ which maximizes the learned function \rev{$f(X; \theta_{t-1}^{j_{t}})$} and receive reward $Y_{t}$.
We use gradient descent on the $\lambda$-regularized loss function to \rev{update} the parameter estimation \rev{for each model}:
\begin{align}  \label{equ:neural_es_estimate}
    \theta_{t}^{j} \approx
    \text{argmin}_{\theta \in \mathbb{R}^{d'}}
    L_{\text{Neural}}(\theta; \{X_{l}, Y_{l} + Z_{l}^{j}\}_{l=1}^{t}).
\end{align}
{Note that there is optimization error due to finite learning rate and optimization steps, $\theta_{t}^{j}$ is not the exact minimizer of the loss function. We explicitly quantify this error in the theoretical analysis.}

\begin{remark}   
The most relevant design is Neural Bandit with Perturbed Reward (\texttt{NPR}) proposed by \citet{jia2022learning}, except that our \texttt{Neural-ES} keeps an ensemble of models and updates the perturbations incrementally instead of resampling all perturbations at each round.
While this design could consume more memory, different models can be updated in parallel if we distribute the ensemble in $m$ different machines. Since the computational cost of updating one model is reduced compared to \texttt{NPR}, our algorithm can accelerate the overall learning process in this parallelized setting.
\end{remark}

\section{Theoretical Analysis}
In this section, we provide theoretical guarantees of the proposed algorithms.

\subsection{Regret Bound of GLM-ES}

To analyze \texttt{GLM-ES}, we first lay down the following  assumptions commonly used in the literature of generalized linear bandits.
The assumption on the derivative of link function $\mu(\cdot)$ is standard in the GLM setting \citep{li2017provably,kveton2020randomized}, while the $M$-self-concordant assumption is recently proposed and applies to a broad class of functions \citep{liu2024almost}.

\begin{assumption}[Link function derivative]
\label{assum:mu}
The link function $\mu (\cdot)$ is strictly increasing and the derivative of $\mu (\cdot)$ is bounded as follows: $\dot{\mu}(s) > 0,\forall s \in \mathbb{R}$ and $0 < \dot{\mu}_{\text{min}} \leq \dot{\mu}(s) \leq \dot{\mu}_{\text{max}}$.
\end{assumption}

\begin{assumption} [$M$-Self-concordant]
\label{assum:M_self_concordant}
    The link function $\mu (\cdot)$ is $M$-self-concordant with a constant $M>0$ known to the agent: $|\ddot{\mu}(u)| \leq M \dot{\mu}(u), \forall u \in \mathbb{R}$.
\end{assumption}

\begin{remark}[\icml{$M$-self-concordant assumption}]
    We adopt the $M$-self-concordant assumption as a mild, notation-simplifying condition for exponential family models. In fact, for common reward distributions (Gaussian, exponential, Poisson, and Beta), it has been shown that \Cref{assum:M_self_concordant} automatically holds \citep[Table 3.1]{liu2023glm}.
    {This assumption is mainly used to bridge $\dot{\mu}(X^{\top}\theta)$ and $\dot{\mu}(X^{\top}\theta')$ for $(\theta, \theta')$ pairs, which is crucial in our theoretical analysis.}
    This assumption has also been used by \citet[Theorem 2]{janz2024exploration}.
\end{remark}

\begin{remark}[Removal of the regularity assumption]
We use $\lambda$-regularized negative log-likelihood in \eqref{equ:regularized_negative_log_likelihood}. This is used to remove the additional regularity assumptions in existing papers \citep{li2017provably, kveton2020randomized}: (i) there exists a constant $\sigma_0>0$ such that $\lambda_{\min }\big(\mathbb{E}[\frac{1}{K} \sum_{a \in[K]} X_{t, a} X_{t, a}^{\prime}]\big) \geq \sigma_0^2$ for all $t$, and (ii) arm context vectors $\{X_{t, a} | a \in[K]\} \subset \mathbb{R}^d$ are i.i.d. drawn. These two assumptions are usually used to guarantee that $V_t$ is invertible. By deploying the $\lambda$-regularized negative log-likelihood in \eqref{equ:regularized_negative_log_likelihood}, our analysis does not require these assumptions.
\end{remark}

We present the frequentist high-probability regret bound of \texttt{GLM-ES} as follows. The complete proof and exact expression of the regret bound is presented in \Cref{appendix:glm_es_bound}.
\begin{theorem} [Regret Bound for \texttt{GLM-ES}]  \label{theorem:glm-es}
    Fix $\delta \in (0, 1]$. Assume $| \mathcal{X} | = K < \infty$ and run \texttt{GLM-ES} with regularization parameter
    $\lambda = 1 \vee \dot{\mu}_{\min} \vee d \vee \log(1/\delta)$, ensemble size $m = \Omega \big( \min\{K \log T, d \log T\} + \log 1/\delta\big)$, perturbation distribution $\mathcal{P}_{R} = \mathcal{N}(0, \sigma_{R}^{2})$, initial perturbation distribution $\mathcal{P}_{I} = \mathcal{N}(0, \sigma_{R}^{2}I_{d})$, where {$\sigma_{R} = \Theta\big((d\log T)^{1/2}\big)$, warm-up rounds $\tau = \Theta\big(d^{4} (\log T)^{3}\big)$}.
    Then, with probability at least $1 - 4\delta$, the cumulative regret of \texttt{GLM-ES} is bounded by
    \begin{align}  \label{equ:regret_bound_bigO}
        R(T) = \widetilde{\mathcal{O}}\big(d^{3/2} \sqrt{T} + d^{4}\big).
    \end{align}
\end{theorem}

\begin{remark}
When we have $T = \Omega\big(d^5\big)$, the regret bound of \texttt{GLM-ES} becomes $\widetilde{\mathcal{O}}(d^{3/2} \sqrt{T})$, which matches the result of \texttt{GLM-TSL}, $\texttt{GLM-FPL}$ \citep{kveton2020randomized} and \texttt{EVILL} \citep{janz2024exploration}, %
achieving a state-of-the-art theoretical guarantee for randomized exploration algorithms in the GLM setting. Compared with \citet{liu2023glm}, we also improve the number of rounds of warm-up from $d^{9}$ to $d^{4}$. We also develop novel analysis techniques to avoid the adapted perturbation requirement as in \cite{liu2023glm}, where the distribution of perturbation changes adaptively in each round. This makes our algorithm more efficient and easier to implement. 
\end{remark}

\subsection{Regret Bound of Neural-ES}

For the regret analysis of \texttt{Neural-ES}, we first introduce the following notations.
A detailed introduction to neural tangent kernel is provided by \citet{jacot2018neural}, here we only list the important notations for our theoretical results.
We use $\mathbf{H}$ to denote the neural tangent kernel (NTK) matrix defined on the context set $\mathcal{X}$ and $\mathbf{h} = ( h(X_{1}), ..., h(X_{K}) )$. Then, we use notation $S_{\text{Neural}}$ as the upper bound of $\sqrt{2\mathbf{h}^{\top}\mathbf{H}^{-1}\mathbf{h}}$.
The effective dimension $\widetilde{d}$ of the NTK matrix is defined as
\begin{align}
    \label{equ:effective_dim_def}
    \widetilde{d} = \frac{\text{log} \, \text{det}(I + T\mathbf{H} / \lambda)}{\text{log} (1 + TK/\lambda)}.
\end{align}

We need the following sub-Gaussian assumption for the random noise in the reward observations.
\begin{assumption}[Sub-Gaussian random noise]
\label{assumption:sub_gaussian}
{
Given feature vector $X \in \RR^{d}$, the reward $Y_{t}$ is generated by $Y_{t} = h(X_{t}) + \eta_{t}$. There exists $\sigma > 0$ such that $\eta_{t}$ is $\mathcal{F}_{t}^{-}$-conditionally $\sigma$-sub-Gaussian: $\forall t \in [T]$, we have $\mathbb{E} [\exp(s\eta_{t}) | \mathcal{F}_{t}^{-}] \leq \exp(\sigma^{2}s^{2}/2) $
holds almost surely for all $s \in \mathbb{R}$.
}
\end{assumption}
We also need the following common assumption on the bound of the NTK matrix and arm set.
\begin{assumption}
    \label{assumption:neural_x}
    {
    We use $\mathbf{H}$ to denote the neural tangent kernel (NTK) matrix on the context set. We assume that $\mathbf{H} \succeq \lambda_{0} I$.
    Moreover, assume that for any $X \in \mathcal{X}$, $||X||_{2} \leq 1$ and $[X]_{j} = [X]_{j+d/2}$.
    }
\end{assumption}
\begin{remark}
    {
    \Cref{assumption:neural_x} is a mild and widely adopted assumption (such as in \texttt{Neural-PHE} \citep{jia2022learning} and \texttt{Neural-UCB} \citep{zhou2020neural}). For any context $X$ with $||X||_{2} \leq 1$, we can construct $X' = [X^{\top}, X^{\top}]^{\top} / \sqrt{2}$ to satisfy this condition, with $d \rightarrow 2d$ for the system.
    }
\end{remark}
Additionally, we need the following assumption for space coverage.
\begin{assumption}
    \label{assumption:coverage}
    {
    Define the subspace $\mathcal{S} = \text{span} \{g(X; \theta_{0})/\sqrt{N}: X \in \mathcal{X}\} \subseteq \mathbb{R}^{d'}$, where $g(X; \theta_{0}) := \nabla_{\theta} f(X; \theta_{0})$. We assume that $\sum_{X \in \mathcal{X}} g(X; \theta_{0})g(X; \theta_{0})^{\top} / N \succeq \rho I_{\mathcal{S}}$, where $\rho > 0$ is a constant.
    }
\end{assumption}

Now we present the frequentist high-probability regret bound for \texttt{Neural-ES} as follows. The complete proof and exact expression of the regret bound is presented in \Cref{appendix:neural_es_bound}.
\begin{theorem} [Regret Bound for Neural-ES] \label{theorem:neural-es}
    Fix $\delta \in (0, 1]$.
    Let $\widetilde{d}$ be the effective dimension of the neural tangent kernel matrix. %
    Assume $| \mathcal{X} | = K < \infty$ and run \texttt{Neural-ES} with ensemble size $m = \Omega \big( K \text{log} T + \log 1/\delta \big)$, regularization parameter $\lambda \geq \text{max}\{1, S_{\text{Neural}}^{-2}\}$, perturbation distribution $\mathcal{P}_{R} = \mathcal{N}(0, \sigma_{R}^{2})$ with $\sigma_{R} = \alpha_{T} \sqrt{(\lambda + \rho)/\rho}$, $\alpha_{T} = \sigma \sqrt{\widetilde{d} \log \Big(1 + TL^{2}/\lambda \widetilde{d}\Big)} + \sqrt{\lambda} S_{\text{Neural}}$.
    Set learning rate $\eta = C_{1} (NTL + N\lambda)^{-1}$, neural network width $N[\text{log} N]^{-3} \geq C_{2} \, \text{max} \big\{ TL^{12}\lambda^{-1}, T^{7}\lambda^{-8}L^{18}(\lambda + LT)^{6}, L^{21}T^{7}\lambda^{-7}(1 + \sqrt{T/\lambda})^{6} \big\}$, where $C_{1}, C_{2} > 0$ are constants.
    Then, with probability at least $1 - 4\delta$, the cumulative regret of \texttt{Neural-ES} is bounded by
    \begin{align}  \label{equ:regret_bound_neural_bigO}
        R(T) = \widetilde{\mathcal{O}}\big(\widetilde{d}^{3/2} \sqrt{T}\big).
    \end{align}
\end{theorem}

\begin{remark}
    To the best of our knowledge, this is the first high-probability regret bound for ensemble sampling in neural contextual bandit setting.
    Comparing with \texttt{Neural-TS} \citep{zhang2021neural} and \texttt{Neural-UCB} \citep{zhou2020neural} which achieved $\widetilde{\mathcal{O}}\big(\widetilde{d} \sqrt{T}\big)$ high-probability regret bound, we acknowledge that our result has an extra $\widetilde{d}^{1/2}$ factor and is less sharp.
    The extra $\widetilde{d}^{1/2}$ factor comes from the fact that our perturbations are memorized rather than freshly sampled at each round, resulting in less tight concentration bound of perturbations. Note that we obtained the same dependency on $d$ in linear and generalized linear bandit settings, indicating that this is a common issue for ensemble sampling theoretical analysis.
    We leave whether the $\widetilde{\mathcal{O}}\big(\widetilde{d} \sqrt{T}\big)$ high probability regret bound could be achieved in ensemble sampling as an open question for future work.
\end{remark}
\begin{remark}
    Compared to prior works in the neural contextual bandit literature, \texttt{Neural-ES} presents unique practical advantages and better scalability.
    Compared to \texttt{Neural-PHE}, the sampling cost per model is reduced from $\Theta(T^{2})$ to $\Theta(T)$ because we do not resample all perturbations in each round.
    Compared to \texttt{Neural-TS} and \texttt{Neural-UCB}, we do not need to construct high-probability confidence sets for exploration, which is difficult to implement and typically involves very high computational cost.
\end{remark}
\begin{remark}
    In previous works of \texttt{Neural-TS} \citep{zhang2021neural} and \texttt{Neural-UCB} \citep{zhou2020neural}, it is a known issue that the effective dimension $\widetilde{d}$ can grow as $\widetilde{d} = \Omega(\sqrt{T})$ \citep{deb2024contextual}. However, in these previous works, the authors assumed that the arm set changes at each round, thus the NTK matrix is $TK \times TK$. In our work, we considered a fixed arm set, thus the effective dimension is bounded by $\widetilde{d} \leq \max \{K, \lambda_{\max}(\mathbf{H})\}$, which is independent of $T$.
\end{remark}

\section{Extension to Anytime Algorithms}

In this section, we present two approaches to extend the ensemble sampling algorithms into anytime algorithms. The first approach is to apply the doubling trick \citep{besson2018doubling} and fully reset the algorithm when entering a new epoch. The second approach is to preserve all interaction histories while adjusting the hyper-parameters for different epochs. Both approaches keeps the asymptotic behavior of regret bound, while the second non-reset method achieves lower regret bound by constant factors.
In this section, we present the algorithm design and theoretical guarantees of the two anytime variants of ensemble sampling, the analysis of regret bound are presented in \Cref{appendix:doubling_trick} and \Cref{appendix:doubling_trick_new}, respectively .

In the doubling trick approach, we choose a sequence of time steps $\{T_{i}\} = \{T_{0},\, T_{1},\, T_{2}, \, ...\}$ and fully restart the original (non-anytime) algorithm when we reach $t = T_{i}+1$.
Therefore, after each reset, the algorithm runs from $T_{i} + 1$ until $\text{min} \, \{T_{i+1},\, T\}$, and we can initialize the $T$-dependent parameters for $L_{i+1} = T_{i+1} - T_{i}$ rounds.
The number of rounds follows the sequence $\{L_{i}\} = \{T_{0},\, T_{1} - T_{0},\, T_{2} - T_{1}, \, ...\}$.
The doubling trick approach treats the original non-anytime algorithm as a black box, thus we can easily extend both \texttt{GLM-ES} and \texttt{Neural-ES} into anytime algorithms using the same method.
Our main theory of the anytime variant of ensemble sampling with doubling trick is as follows.

\begin{theorem} [Regret Bound with Doubling Trick]
    \label{theorem:doubling-trick}
    Set sequence $T_{i} = \lfloor T_{0} b^{i} \rfloor$, where $T_{0} \geq 100$ and $b = (3 + \sqrt{5})/2 \approx 2.6$.
    Use $\mathcal{R}(T, \delta)$ to denote the regret bound of a non-anytime algorithm with $\mathcal{R}(T, \delta) = \mathcal{O}\big(\sqrt{T} (\text{log}T)^{3/2}\big)$, and $R^{DT}(T, \delta)$ to denote its corresponding anytime variant using doubling trick.
    Then, by applying doubling trick, the cumulative regret is bounded by 
    \begin{align*}
        \mathcal{R}^{DT} (T, \delta) \leq \, 3.3 \, \mathcal{R} (T, \delta).
    \end{align*}
\end{theorem}

Applying \Cref{theorem:doubling-trick}, when the total round is $T$ (unknown to the agent), the anytime version of \texttt{GLM-ES} and \texttt{Neural-ES} achieves the same order of regret bound with probability at least $1 - \delta \, \text{log}T$.

\begin{remark}
    By directly applying doubling trick, we obtain the same asymptotic behavior of regret bound. By properly setting the sequence $\{T_{i}\}$, extending ensemble sampling to anytime algorithm comes with the cost of only a constant factor. This constant factor is determined by the parameters $T_{0}$ and $b$, the choice of these parameters can considerably affect the empirical performance. We provide more comprehensive discussions on how to set $T_{0}$ and $b$ in \Cref{appendix:experiments}.
\end{remark}

In addition to the traditional doubling-trick, we provide an anytime extension for \texttt{GLM-ES} without fully resetting the algorithm for each epoch. The algorithm design is detailed in \Cref{alg:es_anytime}.
In \Cref{alg:es_anytime}, we set the parameters as follows:
\begin{align}  \label{equ:param_req}
    m_{i} = \Omega \big( \min\{K \log T_{i}, d \log T_{i}\} + \log 1/\delta\big),\quad
    \sigma_{R}^{(i)} = \Theta\big((d\log T_{i})^{1/2}\big),\quad
    \tau_{i} = \Theta\big(d^{4} (\log T_{i})^{3}\big).
\end{align}

\begin{algorithm}[H]
\caption{Anytime Ensemble Sampling for Generalized Linear Bandits without Reset}
\begin{algorithmic}[1]  \label{alg:es_anytime}
\STATE \textbf{Input:} regularization parameter $\lambda$, anytime schedule $T_{0}$, $b$
\FOR{epoch $i$ = 0, 1, ...}
    \STATE Set $T_{i} = \lfloor T_{0} b^{i} \rfloor$, set $\sigma_{R}^{(i)}$, $m_{i}$ and $\tau_{i}$ according to \eqref{equ:param_req}
    \STATE Resample $Z_{0}^{j}$, $\{Z_{l}^{j}\}_{l=1}^{T_{i-1}}$, $j \in [m_{i}]$ with variance $\sigma_{R}^{(i)}$
    \STATE Compute $\theta_{T_{i-1}}^{j}$, $j \in [m_{i}]$ according to \eqref{equ:glm_es_estimate} as the initial value for the epoch
    \STATE Warm-up for the first $\tau_{i} - \tau_{i-1}$ rounds
    
    \FOR{$t = T_{i-1} + \tau_{i} - \tau_{i-1} + 1, ..., \min\{T_{i}, T\}$}
        \STATE Sample $j_{t}$ uniformly from $\text{[}m_{i}\text{]}$
        \STATE Pull arm $X_{t} \leftarrow \text{argmax}_{X\in\mathcal{X}} f_{t}(X)$ and receive reward $Y_{t}$
        \FOR{$j=1, 2, ..., m_{i}$}
            \STATE Sample $Z_{t}^{j} \sim \mathcal{P}_{R}^{(i)}$
    
            \STATE Update \icml{$\mathcal{D}_{t}^{j} \leftarrow \mathcal{D}_{t-1}^{j} \cup \{\big(X_{t}, Y_{t} + Z_{t}^{j}\big)\}$}
            
            \STATE $\theta_{t}^{j} \leftarrow \text{argmin}_{\theta} \, \Big(L_{\text{GLM}}\big(\theta; \{X_{l}, Y_{l} + Z_{l}^{j}\}_{l=1}^{t}\big) - \sqrt{\lambda / \dot{\mu}_{\min}} \langle Z_{0}^{j}, \theta \rangle\Big)$ 
        \ENDFOR
    \ENDFOR
\ENDFOR
\end{algorithmic}
\end{algorithm}

We have the following regret bound for the non-reset anytime extension for \texttt{GLM-ES}.

\begin{theorem} [Regret Bound for Anytime \texttt{GLM-ES} without Reset]  \label{theorem:glm-es-anytime}
    Fix $\delta \in (0, 1]$. Assume $| \mathcal{X} | = K < \infty$ and run \texttt{GLM-ES} with regularization parameter
    $\lambda = 1 \vee \dot{\mu}_{\min} \vee d \vee \log(1/\delta)$. For each epoch $i$, set ensemble size $m_{i} = \Omega \big( \min\{K \log T_{i}, d \log T_{i}\} + \log 1/\delta\big)$, perturbation distribution $\mathcal{P}_{R}^{(i)} = \mathcal{N}(0, \sigma_{R}^{(i)2})$, initial perturbation distribution $\mathcal{P}_{I}^{(i)} = \mathcal{N}(0, \sigma_{R}^{(i)2}I_{d})$, where $\sigma_{R}^{(i)} = \Theta\big((d\log T_{i})^{1/2}\big)$, warm-up rounds $\tau_{i} = \Theta\big(d^{4} (\log T_{i})^{3}\big)$.
    Use anytime parameters $T_{0}$ and $b$ such that epoch $i$ runs in $[T_{0} b^{i-1}, T_{0}b^{i}]$.
    Then, with probability at least $1 - \delta \log T$, the cumulative regret of \texttt{GLM-ES} is bounded by
    \begin{align}
        R^{(A)}(T) = \Big( \frac{\log (b T_{0})}{\log T_{0}} \Big)^{3} R(T),
    \end{align}
    where $R(T)$ is the regret bound of (non-anytime) \texttt{GLM-ES} provided in \Cref{theorem:glm-es}.
    By choosing anytime parameters $T_{0} \geq 100$ and $b = (3 + \sqrt{5})/2 \approx 2.6$, $R^{(A)}(T) \leq 1.78 R(T)$.
\end{theorem}

\begin{remark}
    The anytime extension of \texttt{GLM-ES} without fully reset achieves lower constant factor compared to doubling trick result \Cref{theorem:doubling-trick} due to the fact that we preserved the interaction history prior to epoch $i$, while doubling trick fully restarts the algorithm and all interaction history before the restart is discarded.
\end{remark}

\begin{remark}
    While the non-reset anytime extension works elegantly for \texttt{GLM-ES}, it cannot be directly adapted to \texttt{Neural-ES} due to the fact that the requirements for the neural network width and learning rate are $T$-dependent. After entering a new epoch and adjusting the neural network width and learning rate, the neural network has to be re-trained from scratch and the previous parameters cannot be meaningfully preserved in the new epoch.
\end{remark}

\section{Experiments}  \label{section:exp}

We conduct experiments to demonstrate the practicality of ensemble sampling and its anytime variants in bandit settings.
We use synthetic linear bandit environment to test \texttt{Lin-ES}, logistic bandit to test \texttt{GLM-ES}, then distance bandit and quadratic form bandit to test \texttt{Neural-ES}.
Additional experiments on real-world datasets and performance sensitivity with respect to hyper-parameter setups are included in \Cref{appendix:experiments}.
Our result shows that ensemble sampling achieves competitive cumulative regret with reduced computational cost.

\begin{figure*}[t]
\centering
    \subfigure[Linear Bandit.]{\includegraphics[width=.32\textwidth]{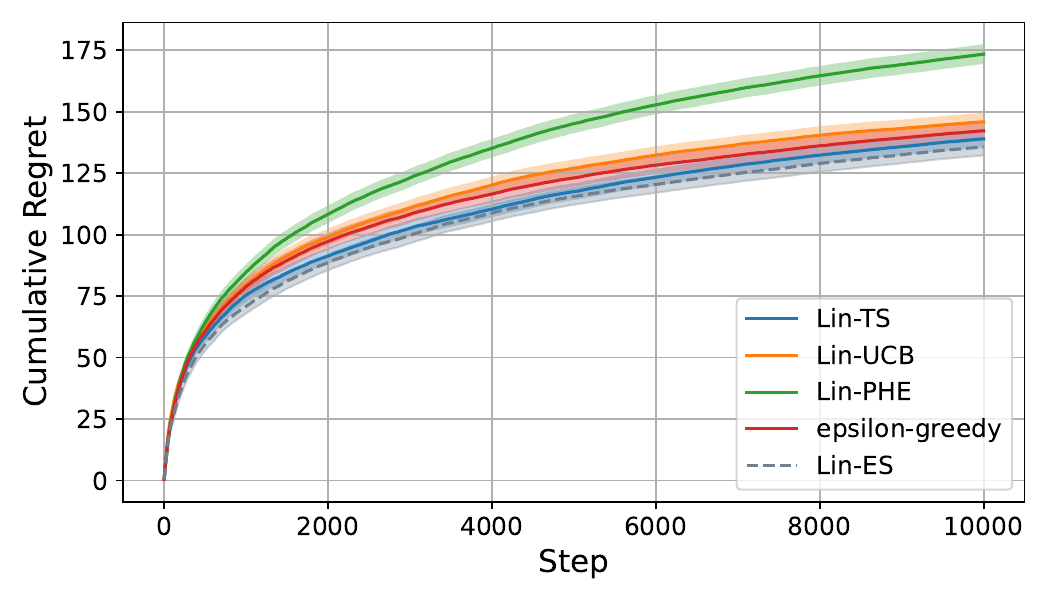}\label{fig:linear_bandit}}
    \subfigure[Logistic Bandit.]{\includegraphics[width=.32\textwidth]{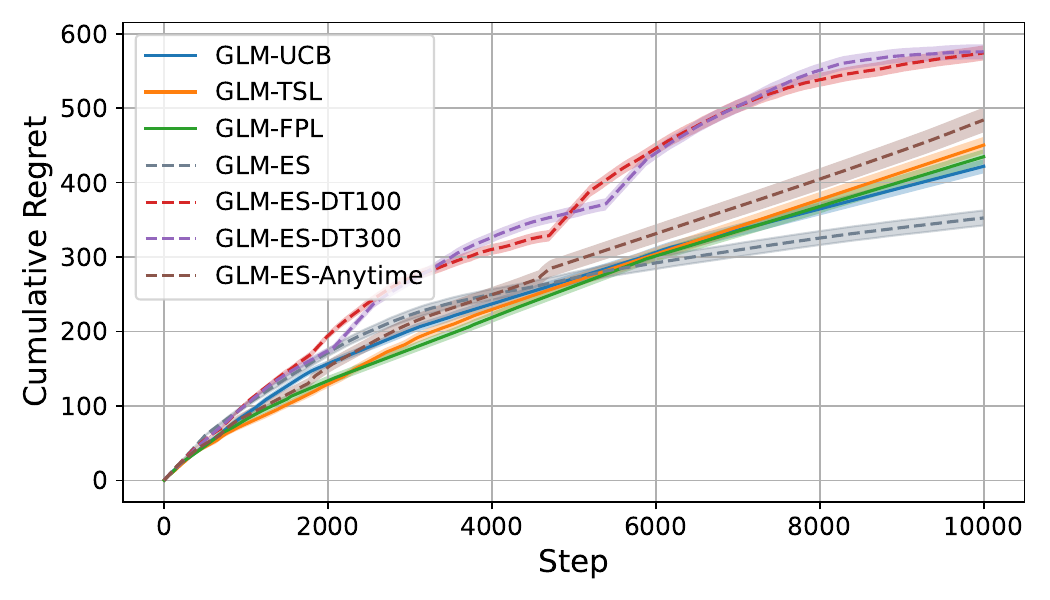}\label{fig:logistic_bandit}} 
    \subfigure[Runtime comparison for logistic bandit.]{\includegraphics[width=.32\textwidth]{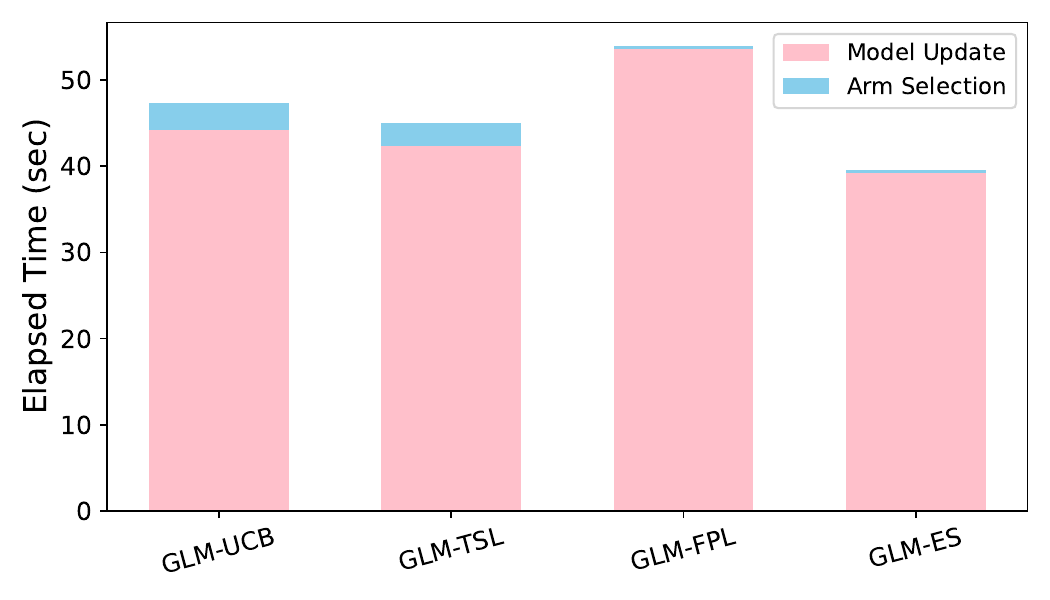}\label{fig:runtime_log}} 
    \subfigure[Distance Bandit.]{\includegraphics[width=.32\textwidth]{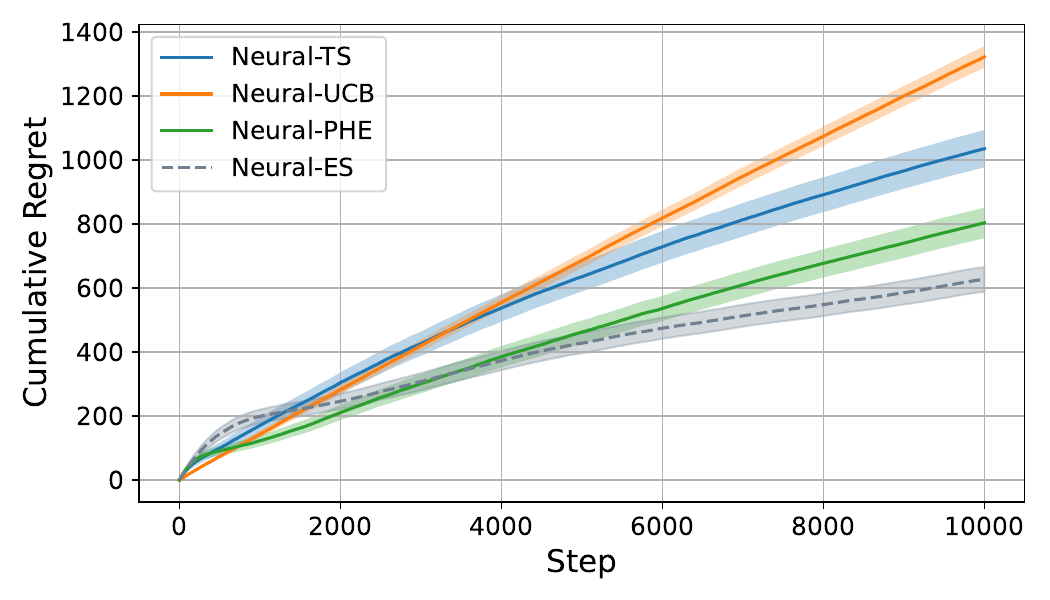}\label{fig:distance_bandit}}
    \subfigure[Quadratic Bandit.]{\includegraphics[width=.32\textwidth]{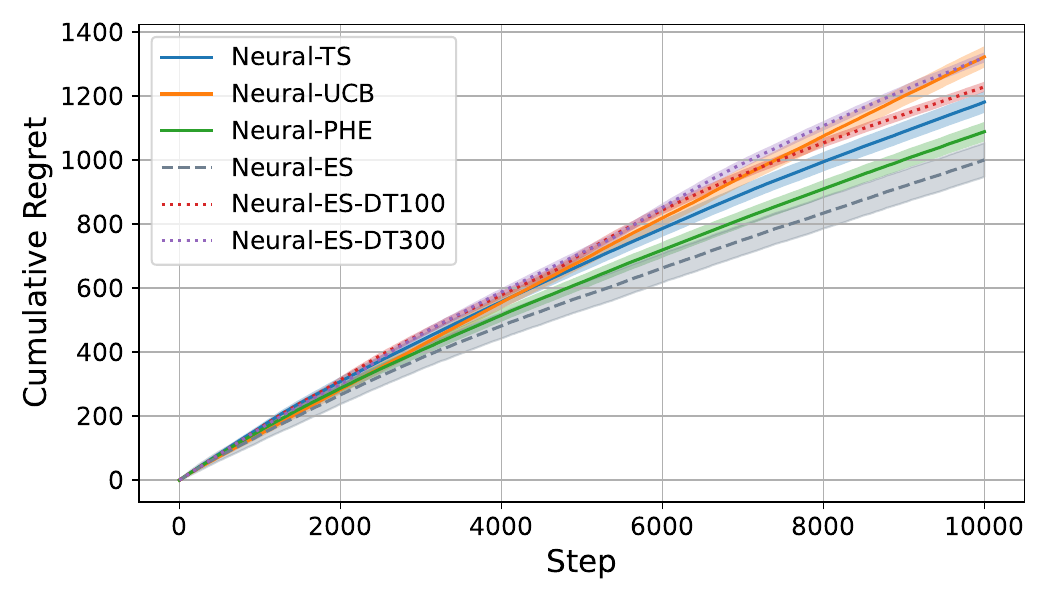}\label{fig:quadratic_bandit}}
    \subfigure[Runtime comparison for quadratic bandit.]{\includegraphics[width=.32\textwidth]{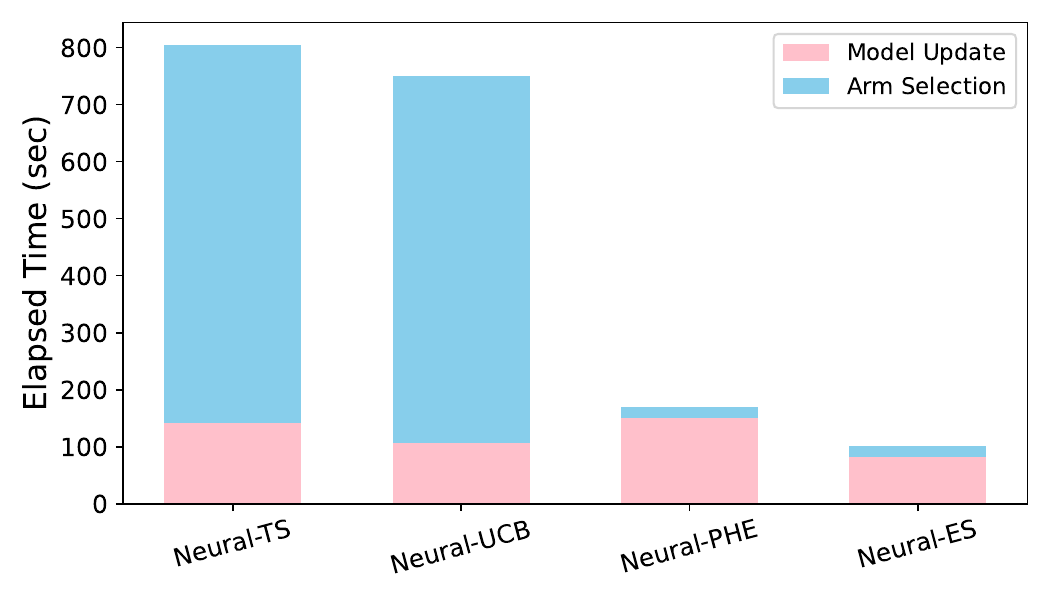}\label{fig:runtime_neural}} 
    \caption{Experiment results in various \icml{synthetic contextual} bandit settings. \icml{The results are averaged over 50 repetitions. The shaded area represents the standard deviation of the cumulative regret.} %
    \label{fig:bandit}}
\end{figure*}

\subsection{Environment Setup}

\paragraph{Linear Bandit}
We assess the empirical performance of \texttt{Lin-ES} using linear bandit environment, where the reward is generated by $Y_{t} = X_{t}^{\top} \theta^{*} + \eta_{t}$.
Specifically, we set the number of arms $K = 50$, dimension of feature vector $d = 20$ and total steps $T = 10^{4}$.
The noise of reward are generated from Gaussian distribution $\eta_{t} \sim \mathcal{N}(0, \sigma^{2})$ with $\sigma = 0.5$.

\paragraph{Generalized Linear Bandit}
We assess the empirical performance of \texttt{GLM-ES} using logistic bandit environment, where the reward is given by $Y_{t} = 1 / \big(1 + \text{exp} (-X_{t}^{\top}\theta^{*})\big) + \eta_{t}$.
We set the number of arms $K = 50$, dimension of feature vector $d = 20$, total steps $T = 10^{4}$. The noise of reward are generated from Gaussian distribution $\eta_{t} \sim \mathcal{N}(0, \sigma^{2})$ with $\sigma = 0.5$.

\paragraph{Neural Bandit}
We test \texttt{Neural-ES} on two nonlinear reward models: (i) distance bandit: $h_{1}(X) = - ||X - \theta^{*}||_{2}$; and (ii) quadratic bandit: $h_{2}(X) = 10^{-2} ( X^{\top}AA^{\top}X )$,
where $A$ is a $d \times d$ matrix with each entry randomly sampled from $\mathcal{N}(0, 1)$. In the experiments, we set number of arms $K = 50$, dimension of feature vector $d = 20$ and total steps $T = 10^{4}$. The noise of reward are generated from Gaussian distribution $\eta_{t} \sim \mathcal{N}(0, \sigma^{2})$ with $\sigma = 0.5$.

\subsection{Implementation Details of Algorithms}

\paragraph{Lin-ES}
In the \texttt{Lin-ES} algorithm, we set ensemble size $m = 25$, regularization $\lambda = 1.0$, and perturbation distribution $\mathcal{N}(0, \sigma_{R}^{2})$ with $\sigma_{R} = 0.1$.
There is no warm-up procedure in \texttt{Lin-ES}.

\paragraph{GLM-ES}
We implement \texttt{GLM-ES} assuming that the link function $\mu(\cdot)$ is known to the agent. The parameter $\theta$ is learned through gradient descent of the $\lambda$-regularized negative likelihood.
Specifically, we use 100 iterations of gradient descent with step size 0.01.
In the non-anytime experiment, we set ensemble size $m = 10$, perturbation distribution $\mathcal{N}(0, \sigma_{R}^{2})$ with $\sigma_{R} = 0.1$, warm-up steps $\tau = 500$ and regularization $\lambda = 1.0$.
In the doubling trick experiments, for epoch $i$, the number of rounds is $L_{i} = T_{i} - T_{i-1}$,  we set ensemble size $m_{i} = 2 \times \text{log}L_{i}$, perturbation distribution $\mathcal{N}(0, \sigma_{R}^{(i)2})$ with $\sigma_{R}^{(i)} = 0.02 \times (\log L_{i})^{1/2}$, and warm-up steps $\tau_{i} = 100 \times (\log L_{i})^{3}$.
In the non-reset anytime experiment, for epoch $i$, we set ensemble size $m_{i} = 2 \times \log T_{i}$, perturbation distribution $\mathcal{N}(0, \sigma_{R}^{(i)2})$ with $\sigma_{R}^{(i)} = 0.02 \times (\log T_{i})^{1/2}$, and warm-up steps $\tau_{i} = 100 \times (\log T_{i})^{3}$.

\paragraph{Neural-ES}
We implement the fully connected neural network $f(X; \theta)$ using $L=3$ layers, the width is set to $N = 20$ for each layer and we use ReLU as the activation function. The network structure is the same for \texttt{Neural-ES}, \texttt{Neural-TS}, \texttt{Neural-UCB} and \texttt{Neural-PHE}.
We optimize the loss function using gradient descent with 100 steps and learning rate 0.01.
In the non-anytime experiment, we set ensemble size $m = 10$, perturbation distribution $\mathcal{N}(0, \sigma_{R}^{2})$ with $\sigma_{R} = 0.1$, warm-up steps $\tau = 50$ and regularization $\lambda = 1.0$.
In the doubling trick experiment, for epoch $i$, the number of rounds is $L_{i} = T_{i} - T_{i-1}$, we set ensemble size $m_{i} = 2 \times \log L_{i}$, perturbation distribution $\mathcal{N}(0, \sigma_{R}^{(i)2})$ with $\sigma_{R}^{(i)} = 0.02 \times (\log L_{i})^{1/2}$, and warm-up steps $\tau = 50$.

\begin{remark}
    The width condition in \Cref{theorem:neural-es} is a sufficient over-parameterization condition used to control the NTK approximation and is highly conservative, as is common in provable neural-bandit analysis.
    As a common practice in neural contextual bandit research \citep{jia2022learning}, we use moderate-sized neural networks for our empirical evaluations.
    Additionally, according to experiments on width-sensitivity study \Cref{appendix:hyper_parameters}, a moderate width is sufficient in the evaluated environments and that further increasing the width yields limited empirical improvement.
\end{remark}

\subsection{Results and Discussions}

The empirical results for synthetic contextual bandits are plotted in \Cref{fig:bandit}.
The cumulative regret results of ensemble sampling are plotted in dashed lines. Specifically, ``DT100'' and ``DT300'' indicate that we choose $T_{0}=100$ and $T_{0}=300$ in doubling trick, respectively.
We set $b = (3 + \sqrt{5})/2 \approx 2.6$ for all doubling trick simulations.
For the non-reset anytime \texttt{GLM-ES} \Cref{alg:es_anytime} (GLM-ES-Anytime), we use $T_{0} = 100$ and $b = (3 + \sqrt{5})/2$.
The $m$ models in ensemble sampling algorithms are updated in parallel.
All numerical results are averaged over 50 problem instances.
For fair comparison, all linear/generalized linear bandit algorithm evaluations are conducted on two Intel Xeon Platinum 8375C CPUs with 32GB memory, all neural bandit algorithm evaluations are conducted on an NVIDIA T4 GPU.  The maximum time budget for each simulation is 10,000 seconds.

According to simulation results, the cumulative regret of ensemble sampling is competitive compared to baseline algorithms in both linear and nonlinear bandit settings.
\icml{From \Cref{fig:runtime_neural}, comparing with \texttt{Neural-TS} and \texttt{Neural-UCB}, the perturbed-history-based \texttt{Neural-PHE} and \texttt{Neural-ES} involve much less computational cost at inference (arm selection) due to the fact that we do not need to construct confidence sets for the reward of each arm. The computational cost of \texttt{Neural-ES} at inference is similar to \texttt{Neural-PHE}, while the model update is considerably faster because perturbations are kept in history and the model is updated incrementally.}
Additionally, the anytime variants of ensemble sampling demonstrate their practicality by achieving competitive cumulative regret while removing the requirement of $T$ in prior.

Overall, our experiment result \icml{demonstrates} that ensemble sampling is practical in both linear and nonlinear bandit settings, its cumulative regret is similar to or outperforms baseline models in different bandit environments.
\icml{Our additional experiments in \Cref{appendix:experiments} further verifies that ensemble sampling-based algorithms are competitive in cumulative regret and computational efficiency in real-world and more complex environments.}

\section{Conclusion}

In this work, we studied ensemble sampling in  nonlinear contextual bandit settings. We proposed a general framework of algorithm design for ensemble sampling in bandit problems, then discussed two realizations \texttt{GLM-ES} and \texttt{Neural-ES} for generalized linear bandit and neural bandit settings, respectively.
We proved high-probability regret bound $\widetilde{\mathcal{O}}(d^{3/2} \sqrt{T} + d^{4})$ for \texttt{GLM-ES} and $\widetilde{\mathcal{O}}(\widetilde{d}^{3/2} \sqrt{T})$ for \texttt{Neural-ES}. To the best of our knowledge, these are the first high-probability regret bounds for ensemble sampling in nonlinear bandit settings. The regret bound of \texttt{GLM-ES} matches the state-of-the-art result of randomized exploration algorithms in generalized linear bandit setup. We used synthetic bandit environments and real-world datasets to evaluate the performance of the proposed algorithms in terms of cumulative regret and computational cost. The empirical results demonstrate that ensemble sampling and its anytime variants can achieve competitive cumulative regret with considerably reduced computational cost. Our work establishes ensemble sampling as a provable and practical algorithm framework in bandit problems.

\clearpage

\appendix

\section{Proof of Regret Bound of GLM-ES}  \label{appendix:glm_es_bound}

{In this section, we prove the claimed high-probability regret bound for \texttt{GLM-ES} \Cref{theorem:glm-es}. The structure of the proof is as follows. In \Cref{appendix_sub:glm_es_preliminary}, we introduce the notations used throughout the analysis, as well as useful technical lemmas obtained directly from the generalized linear model setup and \Cref{assumption:exponential_family}, \Cref{assum:mu}, \Cref{assum:M_self_concordant}. In \Cref{appendix_sub:glm_es_lemmas}, we list the main technical lemmas for concentration and optimism, which are required for the proof of the regret bound. The proof of these technical lemmas are provided in \Cref{appendix:glm_es_lemmas}. With these technical lemmas, in \Cref{appendix_sub:glm_es_proof}, we present the proof of the high-probability regret bound \Cref{theorem:glm-es}.}

\subsection{Preliminary Analysis}  \label{appendix_sub:glm_es_preliminary}

In this section, we present basic technical lemmas directly obtained from the definition and assumptions of the GLM setting.
The technical lemmas and notations introduced in this section are fundamental and will be extensively utilized in the following analysis.
The discussion of link function and negative log-likelihood function is standard in the generalized linear bandit literature. We also {adopt} the definitions and technical lemmas of the secant approximation discussed in \cite{janz2024exploration} in our analysis. {Our analysis of \texttt{GLM-ES} builds on these previous works, we list the definitions and technical lemmas of the secant approximation here to make our proof self-contained.} The introduction of the auxiliary matrix \eqref{equ:gt_definition} is novel to our proof and is the key to remove the requirement for adaptive perturbations.

In the GLM setting with exponential family distribution \Cref{assumption:exponential_family}, given feature vector $X \in \RR^{d}$, the conditional distribution of reward {$Y \in \mathbb{R}$} is given by
\begin{align}
    \label{equ:exponential_family}
    p_{\theta}(Y \vert X) =
    b_{0}(Y) \, \text{exp} \big[Y\cdot X^{\top}\theta - b(X^{\top}\theta)\big],
\end{align}
where $\theta \in \RR^{d}$ is a fixed parameter for a given model.
We define the link function $\mu (\cdot)$ as
\begin{align*}
    \mu (\cdot) := \dot{b} (\cdot),
\end{align*}
the conditional expected value of $Y$ is given by {$\mathbb{E}_{\theta}\big[Y|X\big] = \dot{b}\big(X^{\top} \theta\big) =  \mu\big(X^{\top} \theta\big)$}.
{We make additional assumptions as in \Cref{assum:mu}} that the link function $\mu (\cdot)$ is strictly increasing and the derivative of $\mu$ is bounded as follows:
\begin{align*}
    \dot{\mu}(s) > 0,
    \quad \forall s \in \mathbb{R};
    \quad\quad 0 < \dot{\mu}_{\text{min}} \leq \dot{\mu}(s) \leq \dot{\mu}_{\text{max}},
    {
    \quad\forall s \in \mathbb{R}.
    }
\end{align*}
{We assume that $\mu(\cdot)$ is known to the agent.}
Using the mean value theorem, for each pair $s_{1} < s_{2}$, we have 
\begin{align*}
    \mu(s_{2}) - \mu(s_{1}) =
    \dot{\mu}(s_{0}) (s_{2} - s_{1}),
    \quad s_{0} \in (s_{1}, s_{2}).
\end{align*}
Therefore, we have that
\begin{align*}
    \dot{\mu}_{\text{min}} (s_{2} - s_{1}) \leq \mu(s_{2}) - \mu(s_{1}) \leq \dot{\mu}_{\text{max}} (s_{2} - s_{1}).
\end{align*}
We add regularization to the algorithm design for the analysis of the warm-up procedure.
Given input data set $\mathcal{D}_{\icml{t}}=\{(X_l,Y_l)\}_{l=1}^t$, we define the following $\lambda$-regularized negative log-likelihood:
\begin{align}  \label{equ:glm_likelihood}
    L_{\text{GLM}}(\theta; \mathcal{D}_{t})
    &:= \frac{\lambda}{2}{\|\theta\|_{2}^{2}} - \sum_{l=1}^{t} \text{log} \, p_{\theta}(Y_{l} | X_{l}) \notag \\
    &= \frac{\lambda}{2}{\|\theta\|_{2}^{2}}-\sum_{l=1}^t\big(Y_l \cdot X_l^{\top} \theta-b(X_l^{\top} \theta)\big) - \sum_{l=1}^{t} \text{log} \, b_{0}(Y_{l}).
\end{align}
Then, for $t \in [T]$, we define the regularized maximum likelihood estimation (MLE) at the end of round $t$ as follows:
\begin{align}
\label{equ:mle}
    \bar{\theta}_t := \argmin_{\theta \in \mathbb{R}^d} \, L_{\text{GLM}}(\theta; \cD_{\icml{t}}).
\end{align}
Based on \eqref{equ:glm_likelihood}, we can compute the gradient
\begin{align*}
    \nabla_\theta L_{\text{GLM}}(\theta; \cD_{\icml{t}})=\lambda \theta + \sum_{l=1}^t \mu\big(X_l^{\top} \theta\big) X_l - \sum_{l=1}^t X_l Y_l.
\end{align*}
These definitions of {regularized} negative log-likelihood and MLE are standard in the generalized bandit literature. {Due to the non-linearity of the link function, in general, we do not have a closed form solution for $\bar{\theta}_{t}$.}

We adopt the following definitions and technical lemmas discussed in \cite{janz2024exploration}. Our analysis of \texttt{GLM-ES} builds on these previous works.
We define function $f_t(\theta)$ as
\begin{align}  \label{equ:ft_definition}
    f_t(\theta) := \sum_{l=1}^t \mu\big(X_l^{\top} \theta\big) X_l + \lambda \theta,
\end{align}
then the $\lambda$-regularized negative log-likelihood can be written as:
\begin{align*}
    \nabla_\theta L_{\text{GLM}}(\theta; \cD_{\icml{t}})=\lambda \theta + \sum_{l=1}^t \mu(X_l^{\top} \theta) X_l - \sum_{l=1}^t X_l Y_l
    = f_t(\theta) - \sum_{l=1}^t X_l Y_l.
\end{align*}
Note that $\bar{\theta}_t = \argmin_{\theta \in \mathbb{R}^d} L_{\text{GLM}}(\theta; \cD_{t})$ {is a minimizer}, thus $\nabla_\theta L_{\text{GLM}}(\bar{\theta}_{t}; \cD_{t}) = 0$, we have
\begin{align*}
    f_t(\bar{\theta}_t) = \sum_{l=1}^t X_l Y_l.
\end{align*}
We further compute the Hessian matrix as follows:
\begin{align}
\label{equ:hessian_definition}
    H(\theta; \cD_{t}):= \nabla_\theta^2 L_{\text{GLM}}(\theta ; \cD_{t})=\sum_{l=1}^t \dot{\mu}(X_l^{\top} \theta) X_l X_l^{\top}+\lambda I_{d}.
\end{align}
For simplicity of notation, we define the Hessian matrix at true parameter $\theta^{*}$ and MLE $\bar{\theta}_{t}$ as
\begin{align}
    \label{equ:ht_definition}
    H_t := H(\theta^*; \cD_{t}), \quad \bar{H}_t := H(\bar{\theta}_t; \cD_{t}).
\end{align}
Next, we define a secant approximation of Hessian matrix $H(\theta ; \cD_{\icml{t}})$, which was first introduced in \cite{janz2024exploration}. We define the secant function as follows:
\begin{align}
    \label{equ:alpha_def}
    \alpha(s, s^{\prime}):= \begin{cases}\dot{\mu}(s) \qquad \qquad \text { if } s=s^{\prime}, \\
    \frac{\mu(s)-\mu(s^{\prime})}{s-s^{\prime}} \qquad \text { otherwise. }\end{cases}
\end{align}
Note that $\lim_{s \rightarrow s^\prime} \alpha(s,s^\prime) = \dot{\mu}(s)$, $\alpha(s, s')$ is an approximation of $\dot{\mu}(s)$. Since $\mu(\cdot)$ is strictly increasing, we have $\alpha(s, s') > 0$ for all pairs of $(s, s')$. Then, we {define} the following approximation of $H(\theta ; \cD_{t})$:
\begin{align}
    \label{equ:q_definition}
    Q(\theta, \theta^{\prime} ; \cD_{\icml{t}}):=\sum_{l=1}^t \alpha\big(X_l^{\top} \theta, X_l^{\top} \theta^{\prime}\big) X_l X_l^{\top}+\lambda I_{d}.
\end{align}
Note that $Q(\theta, \theta^{\prime} ; \cD_{\icml{t}}) \rightarrow H(\theta ; \cD_{\icml{t}})$ as $\theta \rightarrow \theta^\prime$.

{To quantify the distance between $\alpha(s, s')$ and $\dot{\mu}(s)$, we need the $M$-self-concordant condition \Cref{assum:M_self_concordant} and the following technical lemma that provides upper and lower bounds for the secant function $ \alpha(s, s^{\prime})$.}
\begin{lemma} [Corollary 2 in \cite{sun2019generalized}]
    \label{lem:bounds_of_alpha}
    Given that $\mu(\cdot)$ is $M$-self-concordant, for $s, s^\prime \in \mathbb{R}$, we have
    \begin{align}
        \dot{\mu}(s) h_{s}\big(-M|s-s^{\prime}|\big) \leq \alpha(s, s^{\prime}) \leq \dot{\mu}(s) h_{s}\big(M|s-s^{\prime}|\big),
    \end{align}
    where we define
    \begin{align*}
        h_{s}(x) := \begin{cases}\frac{e^x-1}{x} \qquad \text { if } x \neq 0, \\
        1 \qquad \qquad x = 0.\end{cases}
    \end{align*}
\end{lemma}
{The introduced function $h_{s}(x)$ is strictly increasing: $h_{s}'(x) > 0, \, \forall x\in\mathbb{R}$, and $0 < h_{s}(x) < \infty$.}
{Based on \Cref{lem:bounds_of_alpha} and the fact that}
\begin{align*}
    {
    \frac{1 - e^{-x}}{x} \geq \frac{1}{1+x},\quad \forall x > 0,
    }
\end{align*}
we have
\begin{align*}
    {
    \alpha(s, s') \geq \dot{\mu}(s) h_{s}\big(-M|s-s^{\prime}|\big)
    = \dot{\mu}(s) \frac{1 - \text{exp}\big(-M|s-s'|\big)}{M|s-s'|}
    \geq \dot{\mu}(s) \frac{1}{1 + M|s-s'|}.
    }
\end{align*}
Based on this inequality and the definitions \eqref{equ:hessian_definition} and \eqref{equ:q_definition}, we obtain the following inequality on $H(\theta;\mathcal{D}_{t})$ in terms of $Q(\theta,\theta^{\prime}; \mathcal{D}_{t})$. This result is adapted from Lemma 3.9 in \cite{liu2023glm}.
\begin{lemma}
    \label{lem:upper_bound_of_H}
    For any $\theta, \theta^{\prime} \in \mathbb{R}^d$, we have
    \begin{align}
        H(\theta ; \mathcal{D}_{\icml{t}}) 
        \preceq\big(1+M D(\theta-\theta^{\prime})\big) Q(\theta, \theta^{\prime};\mathcal{D}_{t})
        =\big(1+M D(\theta-\theta^{\prime})\big) Q(\theta^{\prime},\theta;\mathcal{D}_{t}),
    \end{align}
    where $D(v)$ is defined as $D(v)=\max _{X \in \mathcal{X}}|X^{\top} v|$ for input vector $v \in \mathbb{R}^d$.
\end{lemma}
\begin{proof}
    From the inequality
    \begin{align*}
        \dot{\mu}(s) \leq \alpha(s, s') \big(1 + M|s-s'|\big)
        \quad\rightarrow\quad
        \dot{\mu}\big(X^{\top}\theta\big) \leq \alpha\big(X^{\top}\theta, X^{\top}\theta'\big) \big(1 + M|X^{\top}\theta-X^{\top}\theta'|\big),
    \end{align*}
    and the relation
    \begin{align*}
        |X^{\top}\theta-X^{\top}\theta'| \leq D(\theta - \theta'),
    \end{align*}
    we obtain
    \begin{align*}
        H(\theta; \cD_{t})=\sum_{l=1}^t \dot{\mu}(X_l^{\top} \theta) X_l X_l^{\top}+\lambda I_{d}
        \preceq \sum_{l=1}^t \alpha\big(X^{\top}\theta, X^{\top}\theta'\big) \big(1 + M D(\theta - \theta')\big)  X_l X_l^{\top}+\lambda I_{d}.
    \end{align*}
    Since $D(\theta - \theta') > 0$, we further have
    \begin{align}
        H(\theta; \cD_{t}) \preceq& \big(1 + M D(\theta - \theta')\big) \Big( \sum_{l=1}^t \alpha\big(X^{\top}\theta, X^{\top}\theta'\big)  X_l X_l^{\top}+\lambda I_{d} \Big) \notag \\
        =& \big(1+M D(\theta-\theta^{\prime})\big) Q(\theta, \theta^{\prime};\mathcal{D}_{t}).
    \end{align}
\end{proof}
From the definition of $D(v)$, we note the following inequality:
\begin{align}
    \label{equ:dv_bound}
    D(v)=\max _{X \in \mathcal{X}}\big|X^{\top} v\big|
    \leq \Big(\max _{X \in \mathcal{X}} \vert\vert X \vert\vert_{2} \Big) \vert\vert v \vert\vert_{2} \leq \vert\vert v \vert\vert_{2},
\end{align}
where we used the assumption $||X||_{2} \leq 1$.
This inequality is very useful in the following analysis.
By definition \eqref{equ:ft_definition}, we can directly obtain the following relation, which is critical in the later analysis:
\begin{align}
\label{equ:similar_mean_value_theorem}
    f_t(\theta)-f_t(\theta^{\prime})
    &=\sum_{l=1}^{t} \Big( \mu\big(X_{l}^{\top}\theta\big) - \mu\big(X_{l}^{\top}\theta'
    \big) \Big) X_{l} + \lambda \big( \theta - \theta' \big) \notag  \\
    &=Q\big(\theta, \theta^{\prime} ; \cD_{t}\big)(\theta-\theta^{\prime}) \notag \\
    &=Q\big(\theta^{\prime}, \theta ; \cD_{t}\big)(\theta-\theta^{\prime}).
\end{align}
According to this relation, finding {the upper bound of the distance between $\theta$ and $\theta'$} can be decomposed to bounding $|f_t(\theta)-f_t(\theta^{\prime})|$ and $Q\big(\theta, \theta^{\prime} ; \cD_{\icml{t}}\big)$, which provides convenience in the following analysis.
\icml{The definitions of $f_t(\theta)$ and secant approximation are adapted from \cite{janz2024exploration} and we list the results above to make our proof self-contained.}

In addition, we define an auxiliary matrix $G_t$ as follows:
\begin{align}  \label{equ:gt_definition}
    G_t= \sum_{l=1}^t X_l X_l^{\top}+\frac{\lambda}{\dot{\mu}_{\text{min}}} I_{d}.
\end{align}
The introduction of $G_{t}$ is unique to our proof and is the key to remove the requirement for adaptive perturbation distributions as in previous perturbed history generalized linear bandit algorithms \citep{liu2023glm}.
Since we assumed that $0 < \dot{\mu}_{\text{min}} \leq \dot{\mu}(\cdot) \leq \dot{\mu}_{\text{max}}$, we have
\begin{align}
\label{equ:bounds_of_H_by_G}
    \dot{\mu}_{\text{min}} G_t \preceq H(\theta ; \cD_{t}) \preceq \dot{\mu}_{\text{max}} G_t, \quad \forall \theta \in \mathbb{R}^d.
\end{align}

\subsection{Technical Lemmas}  \label{appendix_sub:glm_es_lemmas}

In this section, we list the main technical lemmas required in the analysis of regret bound. These technical lemmas are proved in \Cref{appendix:glm_es_lemmas}.
A key idea of analyzing GLM is that we can use the lower and upper bound of $\dot{\mu}(\cdot)$ to find similarities between GLM and linear bandit.
Despite the similarities between linear bandit and generalized linear bandit settings, the GLM setting is considerably more challenging to analyze due to the fact that we do not have a closed form solution for parameter estimate $\bar{\theta}_{t}$ and $\theta_{t}^{j}$.
In our analysis, we use the $M$-concordant assumption as in \Cref{assum:M_self_concordant} with novel analysis tools to overcome these difficulties.

We list the concentration results of the estimated parameters {$\{\theta_{t}^{j}\}_{j=1}^{m}$} as follows.
{For simplicity,} we use $\theta^{*} = \theta_{\text{GLM}}^{*}$ for the true parameter of the generalized linear model, $\bar{\theta}_{t}$ for the MLE and $\theta_{t}^{j}$ for the estimate of model $j \in [m]$ in the ensemble at the end of round $t\in[T]$, $j_{t}$ for the chosen model at round $t$ and $\theta_{t} = \theta_{t-1}^{j_{t}}$ for the chosen parameter at round $t\in[T]$.
{Note that at the beginning of round $t$, the parameter estimates $\{\theta_{t-1}^{j}\}_{j=1}^{m}$ are computed from interactions $\mathcal{D}_{t-1}=\{(X_l,Y_l)\}_{l=1}^{t-1}$, thus the chosen parameter is expressed as $\theta_{t} = \theta_{t-1}^{j_{t}}$.  As specified in \Cref{theorem:glm-es}, the perturbations are i.i.d. sampled from distribution $\mathcal{P}_{R} = \mathcal{N}(0, \sigma_{R}^{2})$, where $\sigma_{R}$ is a constant.}
\begin{lemma}[Concentration of MLE]
\label{lemma:concentration_mle}
Fix $\delta \in (0, 1)$. Define the following parameters:
\begin{align*}
    \gamma_{t}(\delta, \lambda) &:= \sqrt{\lambda} \bigg( \frac{1}{2M} + S\bigg) + \frac{\rev{4M}}{\sqrt{\lambda}} \bigg(d + \frac{d}{2} \, \text{log} \bigg( 1 + \frac{t \dot{\mu}_{\text{max}}}{d \lambda} \bigg) + \text{log}\frac{1}{\delta} \bigg), \\
    \beta_{t}(\delta, \lambda) &:= M \frac{\gamma_{t}(\delta, \lambda)^2}{\lambda} +  \sqrt{\frac{\gamma_t(\delta, \lambda)^2}{\lambda}}.
\end{align*}
For $t \in {[T]}$, define a sequence of events
\begin{align*}
    \bar{\mathcal{E}}_{t} := \big\{ \|\bar{\theta}_{{t-1}}-\theta^*\|_{H_{{t-1}}} \leq (1+M \beta_T) \,\gamma_T \big\},
\end{align*}
and their intersection $\bar{\mathcal{E}} = \bigcap_{t=1}^{T} \bar{\mathcal{E}}_{t}$. Then, we have $\mathbb{P}(\bar{\mathcal{E}}) \geq 1 - \delta$.
\end{lemma}

\begin{lemma}[Concentration of Perturbations]
\label{lemma:concentration_pert}
Fix $\delta \in (0, 1)$ and $\lambda \geq 1$. Define the following parameter:
\begin{align*}
    \widetilde{\gamma}_{t} (\delta, \lambda) = \sqrt{d \log \bigg(1+\frac{t \dot{\mu}_{\text{min}}}{d \lambda}\bigg)+2 \log \frac{2T}{\delta}} + \sqrt{d} + \sqrt{2 \log \frac{2T}{\delta}}.
\end{align*}
For each round $t \in [T]$, define the following event
\begin{align*}
    \widetilde{\mathcal{E}}_{1, t} := \big\{ \big\|\rev{\theta_{t-1}^{j_{t}}- \bar{\theta}_{t-1}\big\|_{H_{t-1}}} \leq \dot{\mu}_{\text{max}}^{1/2} \, \dot{\mu}_{\text{min}}^{-1} \, \sigma_{R} \, \widetilde{\gamma}_{T} \big\}.
\end{align*}
Then, $\widetilde{\mathcal{E}}_{1, t}$ holds with probability at least $1 - \delta/T$.
\end{lemma}

Combining the concentration results \Cref{lemma:concentration_mle} and \Cref{lemma:concentration_pert}, we have the following result of the concentration of {the chosen parameter estimate $\theta_{t-1}^{j_{t}}$ at round $t\in [T]$}.

\begin{lemma}[Concentration of Estimation \rev{$\theta_{t-1}^{j_{t}}$}]
\label{lemma:concentration_glm_es}
Define a sequence of events as
\begin{align*}
    \mathcal{E}_{1, t} := \big\{ \rev{\big\|\theta_{t-1}^{j_{t}}- \theta^{*}\big\|_{H_{t-1}}} \leq \gamma \big\}, \quad \text{where} \,\, \gamma = (1+M \beta_T)\gamma_{T} + \dot{\mu}_{\text{max}}^{1/2} \, \dot{\mu}_{\text{min}}^{-1} \, \sigma_{R} \, \widetilde{\gamma}_{T},
\end{align*}
and their intersections $\mathcal{E}_{1} = \bigcap_{t=1}^{T} \mathcal{E}_{1, t}$. Then, $\mathcal{E}_{1}$ holds with probability at least $1 -  2\delta$.
\end{lemma}

Next, the following lemma demonstrates that optimism holds with constant probability in \texttt{GLM-ES}.
\begin{lemma}[Optimism]
\label{lemma:optimism_glm}
Let constant $p_{N} = 1 - \Phi(1) \approx 0.16$, and define a sequence of events as
\begin{align*}
    \mathcal{E}_{2, t} = \Big\{ \mathbb{P} \Big(X^{*\top}\theta^{*} \leq X_{t}^{\top} \rev{\theta_{t-1}^{j_{t}}} \text{ and } \widetilde{\mathcal{E}}_{1, t}  \,\vert\, \mathcal{F}_{t-1}\Big) \geq p_{N}/4 \Big\},
\end{align*}
and their intersection $\mathcal{E}_{2} = \bigcap_{t=\tau}^{T} \mathcal{E}_{2, t}$.
{Then, with ensemble size $m \geq \frac{8}{p_{N}^{2}} \Big( K \text{log}T + \text{log}\frac{1}{\delta} \Big)$, regularization $\lambda = 1 \vee d \vee \log(1/\delta)$, perturbation variance $\sigma_{R} = \widetilde{\Omega} \big(d^{1/2}\big)$ and warm-up rounds $\tau = \widetilde{\Omega}\big(d^{4}\big)$}, $\mathcal{E}_{2}$ holds with probability at least $1 - \delta$.
\end{lemma}

\begin{lemma}[Optimism without $K$-dependency]
\label{lemma:optimism_glm_d}
Let constant $p_{N} = 1 - \Phi(1) \approx 0.16$, and define a sequence of events as
\begin{align*}
    \mathcal{E}_{2, t} = \Big\{ \mathbb{P} \Big(X^{*\top}\theta^{*} \leq X_{t}^{\top} \rev{\theta_{t-1}^{j_{t}}} \text{ and } \widetilde{\mathcal{E}}_{1, t}  \,\vert\, \mathcal{F}_{t-1}\Big) \geq p_{N} / 4 \Big\},
\end{align*}
and their intersection $\mathcal{E}_{2} = \bigcap_{t=\tau}^{T} \mathcal{E}_{2, t}$.
Then, with ensemble size $m \geq 2000 d \log (T / \delta)$, regularization $\lambda = 1 \vee \dot{\mu}_{\min} \vee d \vee \log(1/\delta)$, perturbation variance $\sigma_{R} = \widetilde{\Omega} \big(d^{1/2}\big)$ and warm-up rounds $\tau = \widetilde{\Omega}\big(d^{4}\big)$, $\mathcal{E}_{2}$ holds with probability at least $1 - \delta$.
\end{lemma}

We need the following lemma to introduce the notations of $\theta_{t}^{-}$ and $X_{t}^{-}$, which will be used in \Cref{appendix_sub:glm_es_proof}.
These definitions are inspired by Lemma 6 in \cite{lee2024improved} for linear bandit analysis, here we extend the technical lemma to generalized linear bandit setting.
\begin{lemma}  
\label{lemma:gt_define}
Define $J(\theta) = \text{sup}_{X \in \mathcal{X}} \, \mu(X^{\top} \theta)$ as the highest reward under parameter $\theta \in \mathbb{R}^{d}$ and arm set $\mathcal{X}$. Let $\Theta_{t} = \{\theta \in \mathbb{R}^{d} \,\, \vert \,\, \vert\vert \theta - \theta^{*} \vert\vert_{H_{t-1}} \leq \gamma\}$, where $\gamma > 0$ is a constant, {$\theta^{*}$ is the true model parameter}. Define $\theta_{t}^{-} = \text{argmin}_{\theta \in \Theta_{t}} \, J(\theta)$ as parameter associated with lowest reward and $X_{t}^{-} = \text{argmax}_{X \in \mathcal{X}} \, \mu (X^\top \theta_{t}^{-})$ the corresponding chosen arm. For any $\theta \in \mathbb{R}^{d}$ and an event $\mathcal{E}'$, we introduce the following notation:
\begin{align*}
    g_{t} (\theta, \mathcal{E}') = \big(J(\theta) - J(\theta_{t}^{-})\big) \ind\{\mathcal{E}'\}.
\end{align*}
{Then, $g_{t}(\theta^{*}, \mathcal{E}') \geq 0$ holds for any event $\mathcal{E}' \in \mathcal{F}_{t}$.}
Additionally, under event $\mathcal{E}_{1, t}$, $\vert\vert \theta_{t} - \theta^{*} \vert\vert_{H_{t-1}} \leq \gamma$, thus $g_{t} (\theta_{t}, \mathcal{E}'') \geq 0$ holds almost surely for any event such that $\mathcal{E}'' \subset \mathcal{E}_{1, t}$.
\end{lemma}

Finally, we need the following lemma to calculate the summation in \Cref{appendix_sub:glm_es_proof}.
\begin{lemma} [Lemma 2 in \cite{li2017provably}]
\label{lemma:bound_sum}
Let $\{X_{t}\}_{t=1}^{\infty}$ be a sequence in $\mathbb{R}^{d}$ satisfying $\vert\vert X_{t} \vert\vert \leq 1$. Define $X_{0} = \mathbf{0}$ and $V_{t-1} = \lambda \, I_{d} +  \sum_{l=0}^{t-1}X_{l}X_{l}^{\top}$. Suppose there is an integer $n$ such that $\lambda_{\text{min}} (V_{n}) \geq 1$, then for all $T > n$,
\begin{align*}
    \sum_{t=n+1}^{T} \vert\vert X_{t} \vert\vert_{V_{t-1}^{-1}} \leq \sqrt{2 (T-n) d \,\, \text{log}\bigg(1 + \frac{T}{d \lambda}\bigg)}.
\end{align*}
\end{lemma}

\subsection{Regret Analysis}  
\label{appendix_sub:glm_es_proof}
With the technical lemmas listed above, we prove \Cref{theorem:glm-es} in this section.
{Our proof follows the general framework of regret bound analysis in generalized linear bandit problems with randomized exploration strategy, building on the high-probability concentration result and constant probability of optimism of the estimated parameter $\theta_{t}$.
Specifically, we adapt the framework in \citep{lee2024improved} to prove the constant probability of optimism. While their analysis is for linear bandit setting, we extend this method to generalized bandit setting.}
Parts of our analysis (\eqref{equ:gt_lower} and \eqref{equ:gt_upper}) are direct extensions to the generalized linear bandit setting built on previous works in \cite{lee2024improved}, and we include the detailed derivations here to make our proof self-contained.

Note that most of the added difficulties in the generalized linear bandit setting are encapsulated in the technical lemmas listed in \Cref{appendix_sub:glm_es_lemmas}. The main novelties in our analysis are within the proof of these technical lemmas, {especially the concentration result of MLE and sufficient condition of optimism}. The proof of the high-probability regret bound is mostly following the common procedure in the contextual bandit literature once we have gathered the concentration and optimism guarantees in the analysis.

\begin{proof}[Proof of Theorem \ref{theorem:glm-es}]
The cumulative regret is defined as
\begin{align*}
    R(T) = \sum_{t=1}^{T}
    \Big(\mu \big(X^{*\top}\theta^{*} \big) - \mu \big(X_{t}^{\top}\theta^{*} \big)\Big).
\end{align*}
In the warm-up procedure, the regret for each round is bounded as
\begin{align*}
    \Delta_{\text{max}} = \text{sup}_{X \in \mathcal{X}} \,\, \mu \big(X^{\top}\theta^{*} \big) - \text{inf}_{X \in \mathcal{X}} \,\, \mu \big(X^{\top}\theta^{*} \big).
\end{align*}
Therefore, we have
\begin{align*}
    R(T) \leq \tau \Delta_{\text{{max}}} +  \sum_{t=\tau+1}^{T}
    \Big( \mu(X^{*\top}\theta^{*} \big) - \mu \big(X_{t}^{\top}\theta^{*} \big)\Big).
\end{align*}
We only consider $t \in [\tau + 1, T]$ in the following discussions.

We first bound the instantaneous regret at time $t$.
Define event {$\mathcal{E}$ as $\mathcal{E} = \bigcap_{t=\tau+1}^{T} \big( \mathcal{E}_{1, t} \cap \mathcal{E}_{2, t} \big)$}, we consider the per-round regret under event $\mathcal{E}$:
\begin{align*}
    \Big(\mu\big(X^{*\top}\theta^{*}\big) - \mu\big(X_{t}^{\top}\theta^{*}\big)\Big) \ind\{\mathcal{E}\} = \Big(\mu\big(X^{*\top}\theta^{*}\big) - \mu\big(X_{t}^{\top}\theta_{t}\big)\Big) \ind\{\mathcal{E}\} + \Big(\mu\big(X_{t}^{\top}\theta_{t}\big) - \mu\big(X_{t}^{\top}\theta^{*}\big)\Big) \ind\{\mathcal{E}\}.
\end{align*}
{Recall that $\theta_{t} = \theta_{t-1}^{j_{t}}$ is the chosen estimate for round $t$.}
From the expression, \icml{the first term} is related to optimism, \icml{the second term} is related to concentration of $\theta_{t}$.
The concentration is directly bounded under $\mathcal{E}$:
\begin{align*}
    \Big(\mu\big(X_{t}^{\top}\theta_{t}\big) - \mu\big(X_{t}^{\top}\theta^{*}\big)\Big) \ind\{\mathcal{E}\}
    &\leq \, \dot{\mu}_{\text{max}} \big\vert X_{t}^{\top}\theta_{t} - X_{t}^{\top}\theta^{*} \big\vert \ind\{\mathcal{E}\} \\
    &\leq \dot{\mu}_{\text{max}} \vert\vert X_{t} \vert\vert_{H_{t-1}^{-1}} \vert\vert \theta_{t} - \theta^{*} \vert\vert_{H_{t-1}} \ind\{\mathcal{E}\} \\
    &\leq \dot{\mu}_{\text{max}} \gamma \vert\vert X_{t} \vert\vert_{H_{t-1}^{-1}},
\end{align*}
{where we applied \Cref{lemma:concentration_glm_es}.}
Now we consider the optimism condition. Using the definitions in \Cref{lemma:gt_define}, we have
\begin{align*}
    &\Big(\mu\big(X^{*\top}\theta^{*}\big) - \mu\big(X_{t}^{\top}\theta_{t}\big)\Big)
    \ind\{\mathcal{E}\} \\
    &= \Big(\mu\big(X^{*\top}\theta^{*}\big) - \mu\big(X_{t}^{-\top}\theta_{t}^{-}\big)\Big)
    \ind\{\mathcal{E}\} +
    \Big(\mu\big(X_{t}^{-\top}\theta_{t}^{-}\big) - \mu\big(X_{t}^{\top}\theta_{t}\big)\Big)
    \ind\{\mathcal{E}\} \\
    &= \big(J(\theta^{*}) - J(\theta_{t}^{-})\big)
    \ind\{\mathcal{E}\} +
    \big(J(\theta_{t}^{-}) - J(\theta_{t})\big)
    \ind\{\mathcal{E}\} \\
    &=  g_{t}(\theta^{*}, \mathcal{E}) - g_{t}(\theta_{t}, \mathcal{E}) \\
    &\leq g_{t}(\theta^{*}, \mathcal{E}) \\
    &\leq  g_{t}(\theta^{*}, \mathcal{E}_{2, t}).
\end{align*}
We use Markov's inequality to bound $g_{t}(\theta^{*}, \mathcal{E}_{2, t})$:
\begin{align*}
    g_{t}(\theta^{*}, \mathcal{E}_{2, t})
    \mathbb{P}\Big(g_{t}(\theta_{t}, \mathcal{E}_{1, t} \cap \mathcal{E}_{2, t}) \geq g_{t}(\theta^{*}, \mathcal{E}_{2, t})  \,\, \vert \,\, \mathcal{F}_{t-1} \Big)
    \leq \mathbb{E} \big[ g_{t}(\theta_{t}, \mathcal{E}_{1, t} \cap \mathcal{E}_{2, t}) \,\, \vert \,\, \mathcal{F}_{t-1} \big].
\end{align*}
Now we need the lower bound for $\mathbb{P}\big(g_{t}(\theta_{t}, \mathcal{E}_{1, t} \cap \mathcal{E}_{2, t}) \geq g_{t}(\theta^{*}, \mathcal{E}_{2, t})  \,\, \vert \,\, \mathcal{F}_{t-1} \big)$ (we relate this to the probability of optimism) and the upper bound for $g_{t}(\theta_{t}, \mathcal{E}_{1, t} \cap \mathcal{E}_{2, t})$.
The derivations of the lower bound \eqref{equ:gt_lower} and upper bound \eqref{equ:gt_upper} are direct extensions to generalized linear bandit settings adapted from \cite{lee2024improved}, we include the detailed derivations here to make our proof self-contained.

Using the assumption that $\mu(\cdot)$ is strictly increasing, under the event $\big(X^{*\top}\theta^{*} \leq X_{t}^{\top} \theta_{t}\big) \wedge \mathcal{E}_{1, t} \wedge \mathcal{E}_{2, t}$, \icml{$X^{*\top}\theta^{*} \leq \, X_{t}^{\top} \theta_{t}$ is equivalent to $ \big(\mu(X^{*\top}\theta^{*}) - \mu(X_{t}^{-\top} \theta_{t}^{-})\big) \ind\{\mathcal{E}_{2, t}\} \leq \big(\mu(X_{t}^{\top} \theta_{t}) - \mu(X_{t}^{-\top} \theta_{t}^{-})\big) \ind\{\mathcal{E}_{1, t} \cap \mathcal{E}_{2, t}\}$, thus we have}
\begin{align*}
    X^{*\top}\theta^{*} \leq \, X_{t}^{\top} \theta_{t} \quad
    \leftrightarrow \quad \, g_{t}(\theta^{*}, \mathcal{E}_{2, t}) \leq g_{t}(\theta_{t}, \mathcal{E}_{1, t} \cap \mathcal{E}_{2, t}).
\end{align*}
Therefore, we have
\begin{align}
    \label{equ:gt_lower}
    \mathbb{P}\Big(g_{t}\Big(\theta_{t}, \mathcal{E}_{1, t} \cap \mathcal{E}_{2, t}\Big) \geq g_{t}(\theta^{*}, \mathcal{E}_{2, t})  \,\, \vert \,\, \mathcal{F}_{t-1} \Big)
    \geq p_{N} / 4.
\end{align}
Next, we consider the upper bound of $g_{t}(\theta_{t}, \mathcal{E}_{1, t} \cap \mathcal{E}_{2, t})$.
Using the definition, we have %
\begin{align}
    \label{equ:gt_upper}
    g_{t}(\theta_{t}, \mathcal{E}_{1, t} \cap \mathcal{E}_{2, t})
    =& \, \big(\mu(X_{t}^{\top} \theta_{t}) - \mu(X_{t}^{-\top} \theta_{t}^{-})\big) \ind\{\mathcal{E}_{1, t} \cap \mathcal{E}_{2, t}\} \notag \\
    \leq& \, \big(\mu(X_{t}^{\top} \theta_{t}) - \mu(X_{t}^{\top} \theta_{t}^{-})\big) \ind\{\mathcal{E}_{1, t} \cap \mathcal{E}_{2, t}\} \notag \\
    \leq& \, \dot{\mu}_{\text{max}} \big\vert X_{t}^{\top} \theta_{t} - X_{t}^{\top} \theta_{t}^{-} \big\vert \, \ind\{\mathcal{E}_{1, t} \cap \mathcal{E}_{2, t}\} \notag \\
    \leq& \, \dot{\mu}_{\text{max}} \vert\vert X_{t}\vert\vert_{H_{t-1}^{-1}} \vert\vert \theta_{t} - \theta_{t}^{-} \vert\vert_{H_{t-1}} \ind\{\mathcal{E}_{1, t} \cap \mathcal{E}_{2, t}\} \notag \\
    \leq& \, 2\gamma\dot{\mu}_{\text{max}} \vert\vert X_{t}\vert\vert_{H_{t-1}^{-1}}.
\end{align}
Using the results above, the instantaneous regret at time $t \in [\tau+1, T]$ is bounded as:
\begin{align*}
    &\Big(\mu(X^{*\top}\theta^{*}) - \mu(X_{t}^{\top}\theta^{*})\Big)
    \ind\{\mathcal{E}\} \\
    &\leq \gamma\dot{\mu}_{\text{max}} \vert\vert X_{t} \vert\vert_{H_{t-1}^{-1}}
    + \frac{8\gamma\dot{\mu}_{\text{max}}}{p_{N}} \mathbb{E}\Big[\vert\vert X_{t} \vert\vert_{H_{t-1}^{-1}}  \vert  \mathcal{F}_{t-1}\Big] \\
    &= \gamma\dot{\mu}_{\text{max}} \bigg(1 + \frac{8}{p_{N}}\bigg) \vert\vert X_{t} \vert\vert_{H_{t-1}^{-1}}
    + \frac{8\gamma\dot{\mu}_{\text{max}}}{p_{N}}
    \Big(\mathbb{E}\Big[\vert\vert X_{t} \vert\vert_{H_{t-1}^{-1}}  \vert  \mathcal{F}_{t-1}\Big] - \vert\vert X_{t} \vert\vert_{H_{t-1}^{-1}}\Big).
\end{align*}
The cumulative regret {under $\mathcal{E}$} is bounded as:
\begin{align}
\label{equ:two_summation_regret}
    R(T) \ind\{\mathcal{E}\} \notag
    &\leq \tau \Delta_{\text{max}} + \gamma\dot{\mu}_{\text{max}} \bigg(1 + \frac{8}{p_{N}}\bigg) \sum_{t=\tau + 1}^{T} \vert\vert X_{t} \vert\vert_{H_{t-1}^{-1}} \\
    &\qquad+ \frac{8\gamma\dot{\mu}_{\text{max}}}{p_{N}} \sum_{t=\tau + 1}^{T}
    \Big(\mathbb{E}\Big[\vert\vert X_{t} \vert\vert_{H_{t-1}^{-1}} \vert \mathcal{F}_{t-1}\Big] - \vert\vert X_{t} \vert\vert_{H_{t-1}^{-1}}\Big).
\end{align}

Next, we calculate the two summations in \eqref{equ:two_summation_regret}.
{From \eqref{equ:bounds_of_H_by_G}, $\dot{\mu}_{\text{min}} G_t \preceq H(\theta ; \cD_t) \preceq \dot{\mu}_{\text{max}} G_t$. Then, for any $X \in \mathbb{R}^d$, we have}
\begin{align*}
    {
    \frac{1}{\sqrt{\dot{\mu}_{\text{max}}}} || X ||_{G_{t-1}^{-1}}
    \leq || X ||_{H_{t-1}^{-1}}
    \leq \frac{1}{\sqrt{\dot{\mu}_{\text{min}}}} || X ||_{G_{t-1}^{-1}}.
    }
\end{align*}
Therefore, we can write
\begin{align*}
    {
    \sum_{t=\tau + 1}^{T} \vert\vert X_{t} \vert\vert_{H_{t-1}^{-1}}
    \leq \sum_{t=\tau + 1}^{T} \vert\vert X_{t} \vert\vert_{(\dot{\mu}_{\text{min}} G_{t-1})^{-1}}
    = \frac{1}{\sqrt{\dot{\mu}_{\text{min}}}} \sum_{t=\tau + 1}^{T} \vert\vert X_{t} \vert\vert_{G_{t-1}^{-1}}.
    }
\end{align*}
{From the fact that $\dot{\mu}_{\text{min}} G_t \succeq \lambda I$, we have $\lambda_{\text{min}} (\dot{\mu}_{\text{min}} G_t) \geq \lambda$.
With the parameter setup in \Cref{theorem:glm-es}, regularization parameter $\lambda = 1 \vee (2dM/S) \log \big(e \sqrt{1+T \dot{\mu}_{\text{max}} / d} \vee 1 / \delta\big)$, thus $\lambda_{\text{min}} (\dot{\mu}_{\text{min}} G_t) \geq 1$.
The first summation in \eqref{equ:two_summation_regret} is bounded as follows by using \Cref{lemma:bound_sum}:}
\begin{align}
    \label{equ:summation_glm_1}
    {
    \sum_{t=\tau + 1}^{T} \vert\vert X_{t} \vert\vert_{H_{t-1}^{-1}}
    \leq \sum_{t=\tau + 1}^{T} \vert\vert X_{t} \vert\vert_{(\dot{\mu}_{\text{min}} G_{t-1})^{-1}}
    \leq \frac{1}{\sqrt{\dot{\mu}_{\text{min}}}} \sqrt{2(T-\tau)d \,\, \text{log}\bigg(1 + \frac{T}{d \lambda}\bigg)}.
    }
\end{align}

Now we consider the second summation in \eqref{equ:two_summation_regret}. Note that
\begin{align*}
    0 \leq \vert\vert X_{t} \vert\vert_{H_{t-1}^{-1}} \leq \sqrt{\lambda_{\text{max}}\big(H_{t-1}^{-1}\big)} \vert\vert X_{t} \vert\vert_{2} \leq 1,
\end{align*}
{where we used $||X||_{2} \leq 1$ and $\lambda_{\text{max}}\big(H_{t-1}^{-1}\big) \leq 1/\lambda \leq 1$.}
According to Azuma-Hoeffding inequality, we can bound the second summation in \eqref{equ:two_summation_regret} by
\begin{align}
    \label{equ:summation_glm_2}
    \sum_{t=\tau + 1}^{T}
    \Big(\mathbb{E}\Big[\vert\vert X_{t} \vert\vert_{H_{t-1}^{-1}} \,\, \vert \,\, \mathcal{F}_{t-1}\Big] - \vert\vert X_{t} \vert\vert_{H_{t-1}^{-1}}\Big)
    \leq \sqrt{\frac{T-\tau}{2} \text{log}\frac{1}{\delta}},
\end{align}
with probability at least $1-\delta$. {Combining the summations \eqref{equ:summation_glm_1} and \eqref{equ:summation_glm_2}, with probability at least $1 - \delta$, we have}
\begin{align*}
    &R(T) \ind\{\mathcal{E}\} \\
    &\leq \gamma \frac{\dot{\mu}_{\text{max}}}{\sqrt{\dot{\mu}_{\text{min}}}} \bigg(1 + \frac{8}{p_{N}}\bigg)
    \sqrt{2d(T-\tau) \,\, \text{log}\bigg(1 + \frac{T}{d \lambda}\bigg)}
    + \frac{4\gamma\dot{\mu}_{\text{max}}}{p_{N}}
    \sqrt{2(T-\tau) \text{log}\frac{1}{\delta}}
    + \tau \Delta_{\text{max}}.
\end{align*}
{According to \Cref{lemma:concentration_glm_es} and \Cref{lemma:optimism_glm}, event $\mathcal{E} = \bigcap_{t=\tau+1}^{T} \big( \mathcal{E}_{1, t} \cap \mathcal{E}_{2, t} \big)$ holds with probability at least $1 - 3\delta$. Therefore,}
\begin{align}  \label{equ:glm_bound_exact}
    R(T) \leq \gamma \frac{\dot{\mu}_{\text{max}}}{\sqrt{\dot{\mu}_{\text{min}}}} \bigg(1 + \frac{8}{p_{N}}\bigg)
    \sqrt{2d(T-\tau) \,\, \text{log}\bigg(1 + \frac{T}{d \lambda}\bigg)}
    + \frac{4\gamma\dot{\mu}_{\text{max}}}{p_{N}}
    \sqrt{2(T-\tau) \text{log}\frac{1}{\delta}}
    + \tau \Delta_{\text{max}}
\end{align}
holds with probability at least $1 - 4\delta$.
Combining the fact that the warm up rounds $\tau = \widetilde{\Omega}(d^{4})$ and we set regularization $\lambda = 1 \vee d \vee \log(1/\delta)$, perturbation variance $\sigma_{R} = \widetilde{\Theta}\big(d^{1/2}\big)$, we have $\gamma = \widetilde{\mathcal{O}}(d)$. The final regret bound is
\begin{align*}
    R(T) = \widetilde{\mathcal{O}}\big(d^{3/2} \sqrt{T} + d^{4}\big),
\end{align*}
which holds with probability at least $1 - 4\delta$.
This completes the proof.
\end{proof}

\vspace{3mm}

\section{Proof of Technical Lemmas in GLM-ES}  \label{appendix:glm_es_lemmas}

\subsection{Proof of \Cref{lemma:concentration_mle} (Concentration of MLE)}

In this section, we study the distance between the MLE $\bar{\theta}_{t}$ and the true parameter $\theta^{*}$.
Our main goal is to obtain a high probability upper bound of $\|\theta^* - \bar{\theta}_t\|_{H_{t}}$.
Recall from \eqref{equ:hessian_definition} and \eqref{equ:ht_definition}, the definitions of $H_{t}$ and $\bar{H}_{t}$ are as follows:
\begin{align*}
    H_{t} &\,:= H(\theta^{*}; \cD_{t})= \nabla_\theta^2 L(\theta^{*} ; \cD_{t})=\sum_{l=1}^t \dot{\mu}\big(X_l^{\top} \theta^{*}\big) X_l X_l^{\top}+\lambda I_{d}, \\
    \bar{H}_{t} &\,:= H(\bar{\theta}_{t}; \cD_t)= \nabla_\theta^2 L(\bar{\theta}_{t} ; \cD_t)=\sum_{l=1}^t \dot{\mu}\big(X_l^{\top} \bar{\theta}_{t}\big) X_l X_l^{\top}+\lambda I_{d}.
\end{align*}
Recall that we established the following relation in \eqref{equ:similar_mean_value_theorem}:
\begin{align*}
    f_t(\theta)-f_t(\theta^{\prime})
    =Q\big(\theta^{\prime}, \theta ; \cD_{t}\big)(\theta-\theta^{\prime})
    \quad\rightarrow\quad
    \theta - \theta' = Q^{-1}\big(\theta^{\prime}, \theta ; \cD_{t}\big) \big(f_t(\theta)-f_t(\theta^{\prime})\big).
\end{align*}
Therefore, to obtain the upper bound of $\|\theta^* - \bar{\theta}_t\|_{H_{t}}$, we first focus on upper bounding $\|f_t(\bar{\theta}_t)-f_t(\theta^*)\|_{H_t^{-1}}$.
The proof of the upper bound requires Theorem 2 in \cite{janz2024exploration} and is partially adapted from the proof of Lemma 4 in \cite{janz2024exploration}. We include the detailed derivations to make our proof self-contained. Note that our analysis of the sub-exponential condition differs from the analysis in \cite{janz2024exploration} and is more concise.

Recall from the definition \eqref{equ:ft_definition}, we have that
\begin{align*}
    f_t(\bar{\theta}_t)=\sum_{l=1}^t X_l Y_l, \quad f_t(\theta^*)=\sum_{l=1}^t \mu\big(X_l^{\top} \theta^*\big) X_l+\lambda \theta^*.
\end{align*}
By defining $\epsilon_l=Y_l-\mu\big(X_l^{\top} \theta^*\big)$, $S_t=\sum_{l=1}^t \epsilon_l X_l$, we have
\begin{align*}
    f_t(\bar{\theta}_t) - f_t(\theta^*)
    = \sum_{l=1}^t \Big(Y_l - \mu\big(X_l^{\top} \theta^{*}\big)\Big) X_l - \lambda \theta^*
    = S_t - \lambda \theta^*.
\end{align*}
Therefore, we have the following decomposition:
\begin{align*}
    \|f_t(\bar{\theta}_t)-f_t(\theta^*)\|_{H_t^{-1}} \leq\|S_t\|_{H_t^{-1}}+\lambda\|\theta^*\|_{H_t^{-1}} \leq\|S_t\|_{H_t^{-1}}+\sqrt{\lambda} S,
\end{align*}
where we used the fact that $\lambda I_{d} \preceq H_t$ and {$\|\theta^{*}\|_{2} \leq S$}.

Next, we analyze the upper bound of $\|S_t\|_{H_t^{-1}}$.
We need the following result from \cite{janz2024exploration}.
\begin{lemma} [Theorem 2 in \cite{janz2024exploration}]  \label{lemma:concentration_janz}
    Fix $\lambda, M > 0$.
    Let $(X_{t})_{t \in \mathbb{N^{+}}}$ be a $B_{2}^{d}$-valued random sequence, $(\nu_{t})_{t \in \mathbb{N}}$ be a non-negative valued random sequence.
    Let $\mathbb{F}' = (\mathbb{F}'_{t})_{t \in \mathbb{N}}$ be filtration such that (i) $(X_{t})_{t \in \mathbb{N^{+}}}$ is $\mathbb{F}'$-predictable and (ii) $(Y_{t})_{t \in \mathbb{N^{+}}}$ and $(\nu_{t})_{t \in \mathbb{N}}$ are $\mathbb{F}'$-adapted.
    Let $\epsilon_{t} = Y_{t} - \mathbb{E}[Y_{t} \,\vert\, \mathbb{F}'_{t-1}]$ and assume that the following condition holds:
    \begin{align}  \label{equ:sub_exponential}
        \mathbb{E} [\text{exp}(s\epsilon_{t}) \,\vert\, \mathbb{F}'_{t-1}] \leq \text{exp}(s^{2}\nu_{t-1}),
        \quad \forall \vert s \vert \leq 1/M
        \text{ and } t \in \mathbb{N}^{+}.
    \end{align}
    Then, for $\widetilde{H}_{t} = \sum_{l=1}^{t} \nu_{l-1} X_{l}X_{l}^{\top} + \lambda I_{d}$ and $S_{t} = \sum_{l=1}^{t} \epsilon_{l} X_{l}$ and any $\delta > 0$,
    \begin{align*}
        \mathbb{P} \bigg( \exists t \in \mathbb{N}^{+}:
        \vert\vert S_{t} \vert\vert_{\widetilde{H}_{t}^{-1}} \geq \frac{\sqrt{\lambda}}{2M} + \frac{2M}{\sqrt{\lambda}} \text{log} \Big( \frac{\text{det}(\widetilde{H}_{t})^{1/2} \lambda^{-d/2}}{\delta} \Big) + \frac{2M}{\sqrt{\lambda}} d \,\text{log}(2) \bigg)
        \leq \delta.
    \end{align*}
\end{lemma}
According to previous discussions, we choose the sequences in the lemma as follows:
\begin{align*}
    \nu_{t-1} \rightarrow \dot{\mu}\big(X_{t}^{\top} \theta^{*}\big)
    \quad \Rightarrow \quad
    \widetilde{H}_{t} \rightarrow H_{t}.
\end{align*}
To apply \Cref{lemma:concentration_janz}, we need to prove that $\epsilon_{t}$ satisfies the sub-exponential condition in \eqref{equ:sub_exponential}.
Note that now we have $\epsilon_{t} = Y_{t} - \mu(X_{t}^{\top} \theta^{*})$, we need to utilize the assumption that the reward $Y$ given feature vector $X$ has an exponential-family distribution to prove this {sub-exponential} property.

From {the definition of conditional expected value}, we can write (we use notation $u := X^{\top} \theta^{*}$ for simplicity)
\begin{align*}
    {
    \mathbb{E} [\text{exp}(s\epsilon_{t}) \,\vert\, \mathbb{F}'_{t-1}]
    = \int_{Y} \text{exp} \big(s(Y - \mu(u))\big) \, p(Y \,\vert\, u) \, \text{d}Y.
    }
\end{align*}
From {the definition of exponential family \eqref{equ:exponential_family}}, %
\begin{align*}
    {
    p_{\theta}(Y \,\vert\, u) = b_{0}(Y) \, \text{exp}\big[Y \cdot u - b(u)\big]
    \quad\rightarrow\quad
    \int_{Y}b_{0}(Y) \, \text{exp}(uY) \, \text{d}Y = \text{exp}\big(b(u)\big).
    }
\end{align*}
Therefore, we have
\begin{align*}
    \mathbb{E} [\text{exp}(s\epsilon_{t}) \,\vert\, \mathbb{F}'_{t-1}]
    =&{ \, \text{exp} \big(-s \mu(u) - b(u)\big) \int_{Y} b_{0}(Y) \, \text{exp}((s+u)Y) \, \text{d}Y} \\
    =& \, \text{exp} \big(-s \mu(u) - b(u) + b(s+u)\big).
\end{align*}
Next, observing the expression of \eqref{equ:sub_exponential}, \rev{we need to prove that}
\begin{align*}
    -s \mu(u) - b(u) + b(s+u) \leq s^{2} \dot{\mu}(u)
    \quad\rightarrow\quad
    b(s+u) - b(u) \leq s \mu(u) + s^{2} \dot{\mu}(u).
\end{align*}
{Recall from \Cref{appendix_sub:glm_es_preliminary},} we have $\mu(\cdot) = \dot{b}(\cdot)$, the expression above is equivalent to
\begin{align*}
    b(s+u) - b(u) \leq s \dot{b}(u) + s^{2} \ddot{b}(u).
\end{align*}
Using Taylor expansion, we can write
\begin{align*}
    b(s+u) - b(u) = s \dot{b}(u) + \frac{s^{2}}{2} \ddot{b}(c),
    \quad c \in [u, s+u].
\end{align*}
{Then, the expression is equivalent to}
\begin{align*}
    {
    \ddot{b}(c) \leq 2 \ddot{b}(u).
    }
\end{align*}
By self-concordance assumption, \rev{for $\vert s \vert \leq \log(2)/M$}, we can bound $\ddot{b}(c)$ as
\begin{align*}
    \ddot{b}(c) \leq \ddot{b}(u) \, \text{exp} (M |c-u|)
    \leq \ddot{b}(u) \, \text{exp} (M|s|) \rev{\leq 2 \cdot \ddot{b}(u)}.
\end{align*}
Therefore, we proved that $\mathbb{E} [\text{exp}(s\epsilon_{t}) \,\vert\, \mathbb{F}'_{t-1}] \leq \text{exp}(s^{2}\nu_{t-1})$ holds for $\vert s \vert \leq \log(2)/M$, 
the sub-exponential condition holds in this setting.

Applying \Cref{lemma:concentration_janz} \rev{with $M \rightarrow M / \log(2)$}, we have that
\begin{align*}
    \rev{
    \vert\vert S_{t} \vert\vert_{H_{t}^{-1}} \leq \frac{\sqrt{\lambda}}{2M} + \frac{4M}{\sqrt{\lambda}} \text{log} \bigg( \frac{\text{det}(H_{t})^{1/2} \lambda^{-d/2}}{\delta} \bigg) + \frac{2M}{\sqrt{\lambda}} d,
    \quad \forall t \in [T]
    }
\end{align*}
holds with probability at least $1 - \delta$.
From the definition of $H_{t}$ \eqref{equ:hessian_definition}, we have
\begin{align*}
    \text{det}(H_{t}) \leq \lambda^{d} \bigg( 1 + \frac{t \dot{\mu}_{\text{max}}}{\lambda d} \bigg)^{d}.
\end{align*}
Therefore, we can write the high probability upper bound as
\begin{align*}
    \vert\vert S_{t} \vert\vert_{H_{t}^{-1}} \leq \frac{\sqrt{\lambda}}{2M} + \frac{\rev{4M}}{\sqrt{\lambda}} \Big(d + \frac{d}{2} \text{log} \Big( 1 + \frac{t \dot{\mu}_{\text{max}}}{\lambda d} \Big) + \text{log}\frac{1}{\delta} \Big).
\end{align*}
We can then write the upper bound of $\|f_t(\bar{\theta}_t)-f_t(\theta^*)\|_{H_t^{-1}}$ as follows:
\begin{align}  
\label{equ:gamma_t_def}
    \|f_t(\bar{\theta}_t)-f_t(\theta^*)\|_{H_t^{-1}}
    \leq \sqrt{\lambda} \Big( \frac{1}{2M} + S\Big) + \frac{\rev{4M}}{\sqrt{\lambda}} \Big(d + \frac{d}{2} \text{log}\Big( 1 + \frac{t \dot{\mu}_{\text{max}}}{\lambda d} \Big) + \text{log}\frac{1}{\delta} \Big) =: \gamma_{t}(\delta, \lambda)
\end{align}
holds for all $t \in [T]$ with probability at least $1 - \delta$.

Now we are ready to bound $\|\theta^* - \bar{\theta}_t\|_{H_{t}}$.
Recall from \eqref{equ:similar_mean_value_theorem} that we have the following relations:
\begin{align*}
    f_{t}(\theta) - f_{t}(\theta') = Q(\theta', \theta; \mathcal{D}_{t}) (\theta - \theta'),
    \quad
    Q(\theta, \theta^{\prime} ; \cD_t) = \sum_{l=1}^t \alpha\big(X_l^{\top} \theta, X_l^{\top} \theta^{\prime}\big) X_l X_l^{\top}+\lambda I_{d}.
\end{align*}
We denote $Q_t= Q(\bar{\theta}_t, \theta^*; \cD_t)$, then $\bar{\theta}_{t} - \theta^{*} = Q_{t}^{-1} \big( f_{t}(\bar{\theta}_{t}) - f_{t}(\theta^{*}) \big)$.
We have the following bound:
\begin{align*}
    {\big\|\bar{\theta}_t-\theta^*\big\|_{2}^2} &= \big\|Q_t^{-1}\big(f_t(\bar{\theta}_t)-f_t(\theta^*)\big)\big\|_{2}^2 \\
    &\leq \frac{1}{\lambda} \|f_t(\bar{\theta}_t)-f_t(\theta^*)\|_{Q_t^{-1}}^2 \\
    &\leq \big(1+M D(\bar{\theta}_t-\theta^*)\big) \frac{1}{\lambda} \big\|f_t(\bar{\theta}_t)-f_t(\theta^*)\big\|_{H_t^{-1}}^2 \\
    &\leq \big(1+M \|\bar{\theta}_t-\theta^*\|_{2}\big) \frac{1}{\lambda} \gamma_{t}(\delta, \lambda)^2.
\end{align*}
{The first inequality holds due to the fact $\lambda I_{d} \preceq Q_t$.
The second inequality holds because $H_{t} = H(\theta^{*}; \mathcal{D}_{t}) \preceq \big(1 + MD(\theta^{*} - \bar{\theta}_{t})\big) Q(\theta^{*}, \bar{\theta}_{t}; \mathcal{D}_{t}) = Q_{t}$ (\Cref{lem:upper_bound_of_H}).
The last inequality holds because $D(v) \leq \|v\|_{2}$ for $v \in \mathbb{R}^d$ \eqref{equ:dv_bound} and is under the assumption $\|f_t(\bar{\theta}_t)-f_t(\theta^*)\|_{H_t^{-1}} \leq \gamma_{t}(\delta, \lambda)$ \eqref{equ:gamma_t_def}.}
Note that for any $b,c \geq 0, x \in \mathbb{R}$, if $x^2 \leq bx+c$, we have $x \leq b+\sqrt{c}$. Based on this result, we can solve
\begin{align}
\label{equ:MLE_error_bound_L2}
    {\|\bar{\theta}_t-\theta^*\|_{2}} \leq M \frac{\gamma_{t}(\delta, \lambda)^2}{\lambda} +  \sqrt{\frac{\gamma_t(\delta, \lambda)^2}{\lambda}} =: \beta_{t}(\delta, \lambda).
\end{align}
Then, we can write
\begin{align*}
    \|\bar{\theta}_t-\theta^*\|_{H_t} &= \big\|Q_t^{-1}\big(f_t(\bar{\theta}_t)-f_t(\theta^*)\big)\big\|_{H_t} \\
    &= \sqrt{\big(f_t(\bar{\theta}_t)-f_t(\theta^*)\big)^\top Q_t^{-1} H_t  Q_t^{-1} \big(f_t(\bar{\theta}_t)-f_t(\theta^*)\big)} \\
    &\leq \big(1+M D(\bar{\theta}_t-\theta^*)\big)\big\|f_t(\bar{\theta}_t)-f_t(\theta^*)\big\|_{H_t^{-1}} \\
    &\leq \big(1+M {\|\bar{\theta}_t-\theta^*\|_{2}}\big)\gamma_{t}(\delta, \lambda) \\
    &\leq \big(1+M \beta_{t}(\delta, \lambda)\big)\gamma_{t}(\delta, \lambda).
\end{align*}
{Similar to the previous derivation, the first inequality holds because $H_{t} \preceq Q_{t}$ (\Cref{lem:upper_bound_of_H}).
The second inequality holds because $D(v) \leq \|v\|_{2}$ for $v \in \mathbb{R}^d$ and is under the assumption $\|f_t(\bar{\theta}_t)-f_t(\theta^*)\|_{H_t^{-1}} \leq \gamma_{t}(\delta, \lambda)$.
The last inequality holds because of \eqref{equ:MLE_error_bound_L2}.}
{Therefore, since \eqref{equ:gamma_t_def} holds for all $t \in [T]$ with probability at least $1 - \delta$,} we have
\begin{align*}
    \|\bar{\theta}_t-\theta^*\|_{H_t} \leq \big(1+M \beta_{t}(\delta, \lambda)\big)\gamma_{t}(\delta, \lambda),
    \quad \forall t \in [T]
\end{align*}
holds with probability at least $1 - \delta$.

While this concentration result of MLE is valid and tight, we need $\|\bar{\theta}_t-\theta^*\|_{H_t}$ to be bounded by a constant to prove the regret bound.
Therefore, we need to take the upper bound of $\gamma_{t}(\delta, \lambda)$ {and $\beta_{t}(\delta, \lambda)$}.
From the definitions \eqref{equ:gamma_t_def} and \eqref{equ:MLE_error_bound_L2}, both $\gamma_{t}(\delta, \lambda)$ and $\beta_{t}(\delta, \lambda)$ increase as $t$ increases, thus $\gamma_{t}(\delta, \lambda) \leq \gamma_{T}(\delta, \lambda)$, $\beta_{t}(\delta, \lambda) \leq \beta_{T}(\delta, \lambda)$.
Therefore, we have the following concentration result:
\begin{align*}
    \|\bar{\theta}_t-\theta^*\|_{H_t} \leq \big(1+M \beta_{T}\big)\gamma_{T},
    \quad \forall t \in [T]
\end{align*}
holds with probability at least $1 - \delta$.
This completes the proof.

\subsection{Proof of \Cref{lemma:concentration_pert} (Concentration of Perturbation)}

In this section, our main goal is to find a high probability upper bound of $\|\theta_{t-1}^{j_{t}} - \bar{\theta}_{t-1}\|_{H_{t-1}}$, where $j_{t}$ is the selected ensemble element at round $t$, $\theta_{t-1}^j = \argmin_{\theta \in \mathbb{R}^d} L(\theta; \mathcal{D}_{t-1}^j)$ is the estimate of $\theta^{*}$ from ensemble element $j$ at the beginning of round $t$, $\mathcal{D}_{t-1}^j = \{(X_l,Y_l+Z_l^j)\}_{l=1}^{t-1} \cup \{Z_{0}^{j}\}$ is the perturbed dataset used by element $j$ for optimization, and $Z_0^j \sim \mathcal{N}(0,\sigma_{R}^{2} I_{d})$, $Z_l^j \sim \mathcal{N}(0,\sigma_{R}^2)$ are i.i.d. perturbations of rewards.
To achieve this, we first upper bound the distance $\|\theta_{t}^{j} - \bar{\theta}_{t}\|_{H_{t}}$ for fixed $j \in [m]$, then add the randomness of selecting $j$ from the ensemble to prove \Cref{lemma:concentration_pert}.

Recall from the definition \eqref{equ:ft_definition}, we have
\begin{align*}
    f_t(\bar{\theta}_t)&=\sum_{l=1}^t \mu(X_l^{\top} \bar{\theta}_t) X_l+\lambda \bar{\theta}_t=\sum_{l=1}^t X_l Y_l, \\
    f_t(\theta_t^j)&=\sum_{l=1}^t \mu(X_l^{\top} \theta_t^j) X_l+\lambda \theta_t^j=\sum_{l=1}^t X_l(Y_l+Z_l^j) + \sqrt{\frac{\lambda}{\dot{\mu}_{\min}}} Z_{0}^{j},
\end{align*}
where we consider fixed $j \in [m]$.
Then, we have
\begin{align*}
    f_t(\theta_t^j) - f_t(\bar{\theta}_t) = \sum_{l=1}^t Z_l^j X_l + \sqrt{\frac{\lambda}{\dot{\mu}_{\min}}} Z_{0}^{j}.
\end{align*}
We use the following well-known technical lemma to provide a high probability upper bound of the summation.
\begin{lemma}[Theorem 1 of \cite{abbasi2011improved}]  
\label{lemma:martingale}
Let $\{\mathcal{F}_{t}\}_{t=0}^{\infty}$ be a filtration. Let $\{\eta_{t}\}_{t=1}^{\infty}$ be a sequence of real-valued random variables such that $\eta_{t}$ is $\mathcal{F}_{t}$-measurable and $\mathcal{F}_{t-1}$-conditionally $\sigma$-sub-Gaussian for some $\sigma \geq 0$. Let $\{X_{t}\}_{t=1}^{\infty}$ be a sequence of $\mathbb{R}^{d}$-valued random vectors such that $X_{t}$ is $\mathcal{F}_{t-1}$-measurable and $\vert\vert X_{t} \vert\vert_{2} \leq 1$ almost surely for all $t \geq 1$. Fix $\lambda \geq 1$. Let $V_{t} = \lambda I_{d} + \sum_{l=1}^{t} X_{l}X_{l}^{\top}$. Then, for any $\delta \in (0, 1]$, with probability at least $1 - \delta$, the following inequality holds for all $t \geq 0$:
\begin{align*}
    \bigg\| \sum_{l=1}^{t} \eta_{l} X_{l} \bigg\|_{V_{t}^{-1}}
    \leq \sigma \sqrt{d \, \text{log}\bigg( 1 + \frac{t}{d\lambda} \bigg) + 2 \text{log}\frac{1}{\delta}}.
\end{align*}
\end{lemma}

Recall the definition of $G_{t}$ in \eqref{equ:gt_definition}: $G_t= \sum_{l=1}^t X_l X_l^{\top}+\lambda I_{d} / \dot{\mu}_{\text{min}}$.
Applying \Cref{lemma:martingale}, we have that
\begin{align}
\label{equ:concentration_of_perturbation}
    \big\|f_t(\theta_t^j) - f_t(\bar{\theta}_t)\big\|_{G_t^{-1}}
    &= \bigg\|\sum_{l=1}^t Z_l^j X_l + \sqrt{\lambda / \dot{\mu}_{\min}} Z_{0}^{j}\bigg\|_{G_t^{-1}} \notag \\
    &\leq \bigg\|\sum_{l=1}^t Z_l^j X_l\bigg\|_{G_t^{-1}} + \sqrt{\lambda / \dot{\mu}_{\min}} \big\| Z_{0}^{j}\big\|_{G_t^{-1}} \notag \\
    &\leq \sigma_{R} \sqrt{d \log \bigg(1+\frac{t \dot{\mu}_{\text{min}}}{d \lambda}\bigg)+2 \log \frac{2T}{\delta}} + \sigma_{R} \Big( \sqrt{d} + \sqrt{2 \log \frac{2T}{\delta}} \Big) \notag \\
    &=: \sigma_{R} \, \widetilde{\gamma}_t(\delta, \lambda)
\end{align}
holds for all $t\geq 0$ with probability at least $1-\delta/T$.
Here we used
\begin{align*}
    \sqrt{\lambda / \dot{\mu}_{\min}}\big\| Z_{0}^{j}\big\|_{G_t^{-1}}
    \leq \big\|Z_{0}^{j}\big\|_{2}
    \leq \sigma_{R} \big( \sqrt{d} + \sqrt{2 \log (2T / \delta)} \big)
\end{align*}
holds with probability at least $\delta / (2T)$.

Next, we use this result to prove a high probability upper bound of $\|\theta_t^j - \bar{\theta}_t\|_{H_t}$.
Recall from \eqref{equ:similar_mean_value_theorem}, we have
\begin{align*}
    f_t(\theta_t^j) - f_t(\bar{\theta}_t) = Q(\theta_t^j, \bar{\theta}_t; \cD_t)(\theta_t^j - \bar{\theta}_t) =: \widetilde{Q}_t^j(\theta_t^j-\bar{\theta}_t).
\end{align*}
From the definition \eqref{equ:q_definition},
\begin{align}
    \label{equ:q_tj_definition}
    \widetilde{Q}_t^j=Q(\theta_t^j, \bar{\theta}_t ; \cD_t)=\sum_{l=1}^t \alpha\big(X_l^{\top} \theta_t^j, X_l^{\top} \bar{\theta}_t\big) X_l X_l^{\top}+\lambda I_{d}.
\end{align}
Based on the definition of $\alpha(\cdot,\cdot)$ \eqref{equ:alpha_def} and the mean value theorem, there exists $\zeta(l) = a  X_l^{\top} \theta_t^j + (1-a) X_l^{\top} \bar{\theta}_t \in \mathbb{R}$ with $0<a<1$ such that
\begin{align*}
    \alpha\big(X_l^{\top} \theta_t^j, X_l^{\top} \bar{\theta}_t\big)= \frac{\mu(X_l^{\top} \theta_t^j) - \mu(X_l^{\top} \bar{\theta}_t)}{X_l^{\top} \theta_t^j - X_l^{\top} \bar{\theta}_t} = \dot{\mu}(\zeta(l))
    {
    \quad\rightarrow\quad
    \widetilde{Q}_t^j = \sum_{l=1}^{t} \dot{\mu}(\zeta(l)) X_{l}X_{l}^{\top} + \lambda I_{d}.
    }
\end{align*}
Note that we have assumed $0 < \dot{\mu}_{\text{min}} \leq \dot{\mu}(\cdot) \leq \dot{\mu}_{\text{max}}$, thus we have
\begin{align*}
    \dot{\mu}_{\text{min}} G_t \preceq \widetilde{Q}_t^j \preceq \dot{\mu}_{\text{max}} G_t.
\end{align*}

Combining the discussions above and the fact $\dot{\mu}_{\text{min}} G_t \preceq H(\theta ; \cD_t) \preceq \dot{\mu}_{\text{max}} G_t$, when \eqref{equ:concentration_of_perturbation} holds, we have the following upper bound:
\begin{align*}
    \big\|\theta_t^j- \bar{\theta}_t\big\|_{H_t} &= \big\|{\big(\widetilde{Q}^j_t\big)}^{-1}\big(f_t(\theta_t^j)-f_t(\bar{\theta}_t)\big)\big\|_{H_t} \\
    &= \sqrt{\big(f_t(\theta_t^j)-f_t(\bar{\theta}_t)\big)^\top {\big(\widetilde{Q}^j_t\big)}^{-1} H_t  {\big(\widetilde{Q}^j_t\big)}^{-1} \big(f_t(\theta_t^j)-f_t(\bar{\theta}_t)\big)} \\
    &\leq \dot{\mu}_{\text{max}}^{1/2} \, \dot{\mu}_{\text{min}}^{-1} \, \big\|f_t(\theta_t^j)-f_t(\bar{\theta}_t)\big\|_{G_t^{-1}} \\
    &\leq \dot{\mu}_{\text{max}}^{1/2} \, \dot{\mu}_{\text{min}}^{-1} \, \sigma_{R} \, \widetilde{\gamma}_{t} (\delta, \lambda) \\
    &\leq \dot{\mu}_{\text{max}}^{1/2} \, \dot{\mu}_{\text{min}}^{-1} \, \sigma_{R} \, \widetilde{\gamma}_{T}.
\end{align*}
The last inequality holds because according to the definition of $\gamma_{t}(\delta, \lambda)$ \eqref{equ:concentration_of_perturbation}, for $t \in [T]$, we have $\gamma_{t} \leq \gamma_{T}$.

The result above is for fixed ensemble element $j \in [m]$.
We further define the following event for each round $t$ (note that $\big\{\theta_{t}^{j}\big\}_{j=1}^{m}$ are estimated at the end of round $t$ and will be used for inference at round $(t+1)$):
\begin{align*}
    {
    \widetilde{\mathcal{E}}_{1, t} :=
    \Big\{ \big\|\theta_{t-1}^{j_{t}}- \bar{\theta}_{t-1}\big\|_{H_{t-1}}
    \leq \dot{\mu}_{\text{max}}^{1/2} \, \dot{\mu}_{\text{min}}^{-1} \, \sigma_{R} \, \widetilde{\gamma}_{T} \Big\}.
    }
\end{align*}
Since $j_{t}$ is sampled independently, we have
\begin{align*}
    \mathbb{P} \big(\widetilde{\mathcal{E}}_{1, t}\big)
    = \sum_{j=1}^{m} \mathbb{P} \Big( j_{t}=j, \big\|\theta_{t-1}^{j}- \bar{\theta}_{t-1}\big\|_{H_{t-1}}
    \leq \dot{\mu}_{\text{max}}^{1/2} \, \dot{\mu}_{\text{min}}^{-1} \, \sigma_{R} \, \widetilde{\gamma}_{T} \Big)
    \geq 1 - \frac{\delta}{T}.
\end{align*}
Therefore, for each round $t \in [T]$, the concentration of perturbation holds with high probability.
This completes the proof.

\subsection{Sufficient Condition of Optimism}
\label{appendix_sub:sufficient}

In this section, our main goal is to obtain a sufficient condition of optimism {$X_{t}^{\top} \theta_{t-1}^{j_{t}} \geq X^{*\top} \theta^{*}$, where $X_{t}$ is the chosen arm at round $t$ and $\theta_{t-1}^{j_{t}}$ is the selected estimated parameter at round $t$.
The motivation of finding a sufficient condition is that we need to find a lower bound of probability of optimism for each round, but obtaining this probability bound directly is very challenging. To overcome these difficulties, we need to find a sufficient condition for optimism expressed in terms of $\{X_{l}\}_{l=1}^{t-1}$ and $\{Z_{l}^{j}\}_{l=1}^{t-1}$.} By analyzing the probability of this sufficient condition, we can obtain a lower bound of probability of optimism.

{In this section, we consider a fixed ensemble element $j \in [m]$ and obtain a sufficient condition for optimism $X_{t}^{\top} \theta_{t-1}^{j} \geq X^{*\top} \theta^{*}$, where $\theta_{t-1}^{j}$ is the estimated parameter at the beginning of round $t$, $X_{t} = \text{argmax}_{X \in \mathcal{X}} \, \mu\big(X^{\top} \theta_{t-1}^{j}\big)$ is the selected arm assuming that we use $\theta_{t-1}^{j}$ as the estimated parameter. First, we have the following decomposition:}
\begin{align}
\label{equ:expression_optimism}
    X_{t}^{\top} \theta_{t-1}^j - X^{*\top} \theta^{*} &\geq X^{*\top} \theta_{t-1}^j - X^{*\top} \theta^{*} \notag \\
    &= X^{*\top} \big(\widetilde{\theta}_{t-1}^j + \bar{\theta}_{t-1} - \theta^{*}\big) \notag \\
    &= X^{*\top} \widetilde{\theta}_{t-1}^j + X^{*\top} \big(\bar{\theta}_{t-1} - \theta^{*}\big),
\end{align}
where $\widetilde{\theta}_{t-1}^j := \theta_{t-1}^{j} - \bar{\theta}_{t-1}$, and we used the fact that $X_{t} = \text{argmax}_{X \in \mathcal{X}} \,\mu\big(X^{\top}\theta_{t-1}^{j}\big)$.
Next, we introduce the following notations {for the perturbation and pulled arm sequences}:
\begin{align*}
    \mathbf{Z}_{t}^{j} = \big(Z_{0}^{j}, Z_{1}^{j}, \cdots, Z_{t}^{j} \big)^{\top} \in \mathbb{R}^{d+t},
    \quad
    \Phi_{t} = \big(\sqrt{\lambda/\dot{\mu}_{\text{min}}} I_d , X_{1}, \cdots, X_{t}\big) \in \mathbb{R}^{d \times (d+t)},
\end{align*}
then we have $\Phi_{t} \mathbf{Z}_{t}^{j} = \sum_{l=1}^{t} X_{l} Z_{l}^{j} + \sqrt{\lambda / \dot{\mu}_{\min}} Z_{0}^{j}$.
Recall from \eqref{equ:q_tj_definition} that we have
\begin{align*}
    \widetilde{\theta}_{t}^j = \theta_t^j-\bar{\theta}_t = \big(\widetilde{Q}_t^j\big)^{-1} \big(f_t(\theta_t^j) - f_t(\bar{\theta}_t)\big) = \big(\widetilde{Q}_t^j\big)^{-1} \Big(\sum_{l=1}^{t} X_{l} Z_{l}^{j} + \sqrt{\lambda / \dot{\mu}_{\min}} Z_{0}^{j}\Big),
\end{align*}
we then obtain the relation $\widetilde{\theta}_{t}^j = \big(\widetilde{Q}_t^j\big)^{-1} \Phi_{t} \mathbf{Z}_{t}^{j}$.

We further define
\begin{align*}
    U_{t}^\top = X^{*\top}H_{t}^{-1} \Phi_{t} \in \mathbb{R}^{d+t},
\end{align*}
where $H_{t}$ is defined in \eqref{equ:ht_definition}.
According to the definition of $\Phi_{t}$, we have
\begin{align*}
    \Phi_{t}\Phi_{t}^{\top}
    = \sum_{l=1}^{t}X_{l}X_{l}^{\top} + \frac{\lambda}{\dot{\mu}_{\text{min}}} I_{d} \succeq \frac{1}{\dot{\mu}_{\text{max}}} H_{t}
    \quad\rightarrow\quad
    H_{t}^{-1} \preceq \dot{\mu}_{\text{max}} H_{t}^{-1}\Phi_{t}\Phi_{t}^{\top}H_{t}^{-1}.
\end{align*}
Then, we have
\begin{align*}
    \|X^*\|_{H_{t}^{-1}}^2 = {X^*}^\top H_{t}^{-1} X^* \leq \dot{\mu}_{\text{max}} X^{*\top} H_{t}^{-1} \Phi_{t} \Phi_{t}^\top H_{t}^{-1} X^* = \dot{\mu}_{\text{max}} \|U_t\|_2^2.
\end{align*}
This relates the norm of $X^{*}$ to $U_{t}$.

With these notations and results, we can continue to analyze \eqref{equ:expression_optimism}.
First, we can directly bound $X^{*\top} \big(\bar{\theta}_{t-1} - \theta^{*}\big)$ as follows:
\begin{align*}
    X^{*\top} \big(\bar{\theta}_{t-1} - \theta^{*}\big)
    \geq - \vert\vert X^{*} \vert\vert_{H_{t-1}^{-1}}
    \vert\vert \bar{\theta}_{t-1} - \theta^{*} \vert\vert_{H_{t-1}}
    \geq - \sqrt{\dot{\mu}_{\text{max}}} \|U_t\|_2 \, \vert\vert \bar{\theta}_{t-1} - \theta^{*} \vert\vert_{H_{t-1}}.
\end{align*}

Next, we analyze $X^{*\top} \widetilde{\theta}_{t}^j$.
{From the previous result $\widetilde{\theta}_{t}^j = \big(\widetilde{Q}_t^j\big)^{-1} \Phi_{t} \mathbf{Z}_{t}^{j}$,} we have the following expression:
\begin{align*}
    X^{*\top} \widetilde{\theta}_{t}^j = X^{*\top} {\big(\widetilde{Q}_t^j\big)}^{-1} \Phi_{t} \mathbf{Z}_{t}^{j}.
\end{align*}
Recall that from the definition \eqref{equ:q_tj_definition},
\begin{align*}
    \widetilde{Q}_t^j=Q(\theta_t^j, \bar{\theta}_t ; \cD_t)=\sum_{l=1}^t \alpha\big(X_l^{\top} \theta_t^j, X_l^{\top} \bar{\theta}_t\big) X_l X_l^{\top}+\lambda I_{d}.
\end{align*}
{Since $\widetilde{Q}_t^j$ contains the information of perturbation sequence $\{Z_{l}^{j}\}_{l=1}^{t}$ through the estimated parameter $\theta_{t}^{j}$, the randomness from perturbation is not fully decoupled into $\mathbf{Z}_{t}^{j}$. As a result, with this expression, we cannot use the tail distribution of $Z_{l}^{j}$ to obtain the probability of optimism in the following analysis.}
To overcome this difficulty, we define {$W_t^j := \widetilde{Q}_t^j - H_t$}, and we have the following relation:
\begin{align*}
    {
    \bigg(\widetilde{Q}_t^j\bigg)^{-1} = \bigg(H_t + W_t^j\bigg)^{-1}
    = H_t^{-1} - H_t^{-1} W_t^j \big(H_t + W_t^j\big)^{-1}.
    }
\end{align*}
{Note that $H_{t}$ is defined in \eqref{equ:ht_definition} and does not include information of perturbations and random noise.}
Based on \eqref{equ:expression_optimism}, to guarantee that optimism holds, it suffices to satisfy that
\begin{align*}
    &X^{*\top} \widetilde{\theta}_{t-1}^j + X^{*\top} (\bar{\theta}_{t-1} - \theta^{*}) \\
    &=X^{*\top} \big(\widetilde{Q}_{t-1}^j\big)^{-1} \Phi_{t-1} \mathbf{Z}_{t-1}^{j} + X^{*\top} (\bar{\theta}_{t-1} - \theta^{*})\\
    &= X^{*\top} H_{t-1}^{-1} \Phi_{t-1} \mathbf{Z}_{t-1}^{j} - X^{*\top} H_{t-1}^{-1} W_{t-1}^j \big(H_{t-1} + W_{t-1}^j\big)^{-1} \Phi_{t-1} \mathbf{Z}_{t-1}^{j} + X^{*\top} (\bar{\theta}_{t-1} - \theta^{*}) \\
    &\geq 0.
\end{align*}
Then it suffices to satisfy that
\begin{align}  \label{equ:optimism_sufficient_v0}
    {
    X^{*\top} H_{t-1}^{-1} \Phi_{t-1} \mathbf{Z}_{t-1}^{j}
    \geq \bigg| X^{*\top}\bigg(H_{t-1}^{-1} W_{t-1}^j \big(H_{t-1} + W_{t-1}^j\big)^{-1}\bigg) \Phi_{t-1} \mathbf{Z}_{t-1}^{j} \bigg| + \big| X^{*\top} (\bar{\theta}_{t-1} - \theta^{*}) \big|.
    }
\end{align}
{In the following analysis, we will use $(t-1) \rightarrow t$, which is equivalent to analyzing round $(t+1)$. Since the analysis in this section applies to any round $t \in [T]$, this change of notation does not affect our analysis.}
Note that in {$X^{*\top} H_t^{-1} \Phi_{t} \mathbf{Z}_{t}^{j}$}, the randomness of $\{Z_{l}^{j}\}_{l=1}^{t}$ only appears in $\mathbf{Z}_{t}^{j}$, thus we decoupled the randomness of perturbation.
{As a result, we can utilize the tail bound of $Z_{l}^{j} \sim \mathcal{N}(0, \sigma_{R}^{2})$ to prove the probability of optimism using the the anti-concentration inequality in \Cref{appendix_sub:optimism}.}

{To obtain a sufficient condition for optimism, we need to upper bound the right-hand-side of the inequality \eqref{equ:optimism_sufficient_v0}.}
We already have the result
\begin{align*}
    {
    \big| X^{*\top} (\bar{\theta}_{t} - \theta^{*}) \big| \leq \, \sqrt{\dot{\mu}_{\text{max}}} \|U_t\|_2 \vert\vert \bar{\theta}_{t} - \theta^{*} \vert\vert_{H_{t}}.
    }
\end{align*}
We upper bound $\big|X^{*\top} (H_t^{-1} W_t^j (H_t + W_t^j)^{-1}) \Phi_{t} \mathbf{Z}_{t}^{j}\big|$ as follows.
\begin{align*}
        \big|X^{*\top} \big(H_t^{-1} W_t^j\big(\widetilde{Q}_t^j\big)^{-1}\big) \Phi_{t} \mathbf{Z}_{t}^{j}\big|
        &= \big| \big(H_t^{-1/2} X^{*}\big)^{\top} \big(H_t^{-1/2} W_t^j H_t^{-1/2}\big) \big(H_t^{1/2} \big(\widetilde{Q}_t^j\big)^{-1} \Phi_{t} \mathbf{Z}_{t}^{j}\big) \big| \\
        &\leq \big\| H_t^{-1/2} X^{*} \big\|_{2}  \big\| H_t^{-1/2} W_t^j H_t^{-1/2} \big\|_{2} \big\| H_t^{1/2} \big(\widetilde{Q}_t^j\big)^{-1} \Phi_{t} \mathbf{Z}_{t}^{j} \big\|_{2} \\
        &= \| X^{*} \|_{H_t^{-1}} \big\| H_t^{-1/2} W_t^j H_t^{-1/2} \big\|_{2} \big\| \big(\widetilde{Q}_t^j\big)^{-1} \Phi_{t} \mathbf{Z}_{t}^{j} \big\|_{H_t} \\
        &= \| X^{*} \|_{H_t^{-1}}  \big\| H_t^{-1/2} \big(\widetilde{Q}_t^j - H_t\big) H_t^{-1/2} \big\|_{2}  \big\| \widetilde{\theta}_{t}^{j} \big\|_{H_t},
\end{align*}
where we used the relation $\widetilde{\theta}_{t}^j = \big(\widetilde{Q}_t^j\big)^{-1} \Phi_{t} \mathbf{Z}_{t}^{j}$.
{From the definitions \eqref{equ:hessian_definition} and \eqref{equ:q_definition},}
\begin{align*}
    {
    \widetilde{Q}_t^j - H_t =
    \sum_{l=1}^{t}\Big(\alpha\big(X_l^\top \bar{\theta}_t, X_l^\top \theta_t^j\big) - \dot{\mu}(X_l^\top \theta^{*})\Big) X_l X_l^\top.
    }
\end{align*}
We need $\bar{H}_{t}$ defined in \eqref{equ:ht_definition} to connect $\widetilde{Q}_t^j$ and $H_t$:
\begin{align*}
    &\qquad \widetilde{Q}_t^j - H_t = \widetilde{Q}_t^j - \bar{H}_{t} + \bar{H}_{t} - H_t \\
    &=\sum_{l=1}^{t}\Big(\alpha\big(X_l^\top \bar{\theta}_t, X_l^\top \theta_t^j\big) - \dot{\mu}(X_l^\top \bar{\theta}_{t})\Big) X_l X_l^\top
    +  \sum_{l=1}^{t}\Big(\dot{\mu}(X_l^\top \bar{\theta}_{t}) - \dot{\mu}(X_l^\top \theta^{*})\Big) X_l X_l^\top.
\end{align*}
Taking the absolute value of the coefficients, we obtain
\begin{align*}
    &H_t^{-1/2} \bigg( \sum_{l=1}^{t} -\Big( \Big|\alpha\big(X_l^\top \bar{\theta}_t, X_l^\top \theta_t^j\big) - \dot{\mu}\big(X_l^\top \bar{\theta}_{t}\big) \Big| + \Big|\dot{\mu}(X_l^\top \bar{\theta}_{t}) - \dot{\mu}(X_l^\top \theta^{*})\Big| \Big)X_l X_l^\top \bigg)  H_t^{-1/2} \\
    &\preceq H_t^{-1/2} \big(\widetilde{Q}_t^j - H_t\big)  H_t^{-1/2}  \\
    &\preceq H_t^{-1/2} \bigg( \sum_{l=1}^{t} \Big( \Big|\alpha\big(X_l^\top \bar{\theta}_t, X_l^\top \theta_t^j\big) - \dot{\mu}\big(X_l^\top \bar{\theta}_{t}\big) \Big| + \Big|\dot{\mu}(X_l^\top \bar{\theta}_{t}) - \dot{\mu}(X_l^\top \theta^{*})\Big| \Big)X_l X_l^\top \bigg)  H_t^{-1/2}.
\end{align*}
Note that this inequality holds in general without assuming that $\big(\widetilde{Q}_t^j - H_t\big)$ is positive semi-definite. To upper bound the distance between $\dot{\mu}(\cdot)$ and its secant approximation $\alpha(\cdot, \cdot)$, we need the following result from the $M$-self-concordant assumption.
From \Cref{lem:bounds_of_alpha}, we have
\begin{align*}
    {
    \alpha \big(X_l^\top \bar{\theta}_t, X_l^\top \theta_t^j\big)
    \leq \dot{\mu}\big(X_l^\top \bar{\theta}_t\big) \frac{\exp{\big(M \big|X_{l}^{\top}(\bar{\theta}_{t} - \theta_{t}^{j})\big|\big)}-1}{M \big|X_{l}^{\top}(\bar{\theta}_{t} - \theta_{t}^{j})\big|}
    \leq \dot{\mu}\big(X_l^\top \bar{\theta}_t\big) \frac{\exp{(M D_{t}^{j})}-1}{M D_{t}^{j}},
    }
\end{align*}
{where $D_{t}^{j} = \max_{l\leq t} \big|X_l^\top(\bar{\theta}_t - \theta_t^j)\big|$ and we used the fact that $h_{s}(x)$ (defined in \Cref{lem:bounds_of_alpha}) is strictly increasing for $x \in \mathbb{R}$.
Similarly, also from \Cref{lem:bounds_of_alpha}, we have}
\begin{align*}
    {
    \alpha \big(X_l^\top \bar{\theta}_t, X_l^\top \theta_t^j\big)
    \geq \dot{\mu}\big(X_l^\top \bar{\theta}_t\big) \frac{\exp{\big(-M \big|X_{l}^{\top}(\bar{\theta}_{t} - \theta_{t}^{j})\big|\big)}-1}{-M \big|X_{l}^{\top}(\bar{\theta}_{t} - \theta_{t}^{j})\big|}
    \geq \dot{\mu}\big(X_l^\top \bar{\theta}_t\big) \frac{1-\exp{(-M D_{t}^{j})}}{M D_{t}^{j}}.
    }
\end{align*}
{Note that for $x > 0$, we have the following relation}
\begin{align*}
    {
    \Big(\frac{e^{x}-1}{x}-1\Big) - \Big(1-\frac{1-e^{-x}}{x}\Big) = \frac{e^{x}+e^{-x}-2}{x} \geq 0
    \quad\rightarrow\quad
    \Big|\frac{e^{x}-1}{x}-1\Big| \geq \Big|1-\frac{1-e^{-x}}{x}\Big|.
    }
\end{align*}
{Therefore, we can bound the distance between $\dot{\mu}(\cdot)$ and its secant approximation $\alpha(\cdot, \cdot)$ as follows:}
\begin{align*}
    {
    \Big|\alpha\big(X_l^\top \bar{\theta}_t, X_l^\top \theta_t^j\big) - \dot{\mu}\big(X_l^\top \bar{\theta}_{t}\big)\Big| \leq \Big(\frac{\exp{(M D_{t}^{j})}-1}{M D_{t}^{j}} - 1\Big) \dot{\mu}\big(X_l^\top \bar{\theta}_{t}\big).
    }
\end{align*}
{Next, we find the upper bound for $\big|\dot{\mu}(X_l^\top \bar{\theta}_{t}) - \dot{\mu}(X_l^\top \theta^{*})\big|$. From \Cref{assum:M_self_concordant}, $|\ddot{\mu}(u)| \leq M \dot{\mu}(u)$, for any two points $u, v \in \mathbb{R}$, we have $e^{-M|u-v|} \leq \dot{\mu}(u) / \dot{\mu}(v) \leq e^{M|u-v|}$. Therefore, defining $\bar{D}_t = \max_{l\leq t} \big|X_l^\top(\bar{\theta}_t - \theta^{*})\big|$, we can write}
\begin{align}
    \label{equ:mu_dot_ineq}
    {
    \dot{\mu}(X_l^{\top} \bar{\theta}_{t}) e^{-M\bar{D}_t}
    \leq \dot{\mu}(X_l^{\top} \theta^{*})
    \leq \dot{\mu}(X_l^{\top} \bar{\theta}_{t}) e^{M\bar{D}_t}.
    }
\end{align}
Therefore, we have
\begin{align*}
    \Big|\dot{\mu}(X_l^\top \bar{\theta}_{t}) - \dot{\mu}(X_l^\top \theta^{*})\Big| \leq \big( e^{M\bar{D}_{t}} - 1 \big) \dot{\mu}(X_l^{\top} \bar{\theta}_{t}).
\end{align*}
Then, we have
\begin{align*}
    &H_t^{-1/2} \bigg(-\Big(\frac{\exp{(M D_{t}^{j})}-1}{M D_{t}^{j}} - 1 + e^{M\bar{D}_{t}} - 1 \Big) \sum_{l=1}^{t} \dot{\mu}(X_l^\top \bar{\theta}_{t}) X_l X_l^\top \bigg)  H_t^{-1/2} \\
    &\preceq H_t^{-1/2} \big(\widetilde{Q}_t^j - H_t\big) H_t^{-1/2}  \\
    &\preceq H_t^{-1/2} \bigg(\Big(\frac{\exp{(M D_{t}^{j})}-1}{M D_{t}^{j}} - 1 + e^{M\bar{D}_{t}} - 1 \Big) \sum_{l=1}^{t} \dot{\mu}(X_l^\top \bar{\theta}_{t}) X_l X_l^\top \bigg)  H_t^{-1/2}.
\end{align*}
Note that from the fact that $h_{s}(x)$ is strictly increasing and $h_{s}(0) = 1$, we have $\frac{e^{x} - 1}{x} > 1$ for $x > 0$. We define the following quantities for convenience:
\begin{align}  \label{equ:rt_define}
    {
    R_{t}^{j} = \frac{\exp{(M D_{t}^{j})}-1}{M D_{t}^{j}} -1,
    \quad \bar{R}_{t} = e^{M\bar{D}_{t}} - 1.
    }
\end{align}
{Then, for any $v \in \mathbb{R}^d$ with $\|v\|_{2}=1$, we have}
\begin{align*}
    &\Big|v^{\top} \Big(H_t^{-1/2} \big(\widetilde{Q}_t^j - H_t\big) H_t^{-1/2} \Big) v\Big| \\
    &\leq (R_{t}^{j} + \bar{R}_{t}) v^{\top} H_t^{-1/2} (\bar{H}_t - \lambda I_{d}) H_t^{-1/2} v \\
    &= (R_{t}^{j} + \bar{R}_{t}) v^{\top} \big(H_{t}^{-1/2} \bar{H}_{t} H_{t}^{-1/2} - \lambda H_{t}^{-1}\big) v,
\end{align*}
where we used $\sum_{l=1}^{t} \dot{\mu}(X_l^\top \bar{\theta}_{t}) X_l X_l^\top = \bar{H}_{t} - \lambda I_{d}$.
{Combing the inequality \eqref{equ:mu_dot_ineq} with the definitions of $H$-matrices \eqref{equ:hessian_definition} and \eqref{equ:ht_definition}, we have the following relation:}
\begin{align*}
    {
    e^{-M\bar{D}_{t}} \bar{H}_{t} \preceq H_{t} \preceq e^{M\bar{D}_{t}} \bar{H}_{t}
    \quad\rightarrow\quad H_{t}^{-1/2} \bar{H}_{t} H_{t}^{-1/2} \preceq e^{M\bar{D}_{t}} I_{d}.
    }
\end{align*}
Therefore, we have
\begin{align*}
    &\Big|v^{\top} \Big(H_t^{-1/2} \big(\widetilde{Q}_t^j - H_t\big) H_t^{-1/2} \Big) v\Big| \\
    &\leq (R_{t}^{j} + \bar{R}_{t}) v^{\top} \big(H_{t}^{-1/2} \bar{H}_{t} H_{t}^{-1/2} - \lambda H_{t}^{-1}\big) v \\
    &\leq (R_{t}^{j} + \bar{R}_{t}) e^{M\bar{D}_{t}},
\end{align*}
{where we used the fact that $H_{t}$ is symmetric positive definite.}
{According to the definition of spectral norm,} for symmetric matrix $A$, we have
\begin{align*}
    \| A \|_{2} = \text{max}_{\|v\|_{2} = 1} | v^{\top} A v |
    =\max| \lambda(A) |.
\end{align*}
We then obtain that
\begin{align*}
    \big\|H_t^{-1/2} \big(\widetilde{Q}_t^j - H_t\big) H_t^{-1/2} \big\|_{2} = \max_{\|v\|_{2}=1} \Big|v^{\top} \Big(H_t^{-1/2} \big(\widetilde{Q}_t^j - H_t\big) H_t^{-1/2} \Big) v\Big| \leq (R_{t}^{j} + \bar{R}_{t}) e^{M\bar{D}_{t}}.
\end{align*}
{Recall the definition $U_{t}^\top = X^{*\top}H_{t}^{-1} \Phi_{t}$ and the result}
\begin{align*}
    {
    \|X^*\|_{H_{t}^{-1}}^2 = {X^*}^\top H_{t} X^* \leq \dot{\mu}_{\text{max}} \|U_t\|_2^2.
    }
\end{align*}
{With these upper bounds, we can write the sufficient condition for optimism \eqref{equ:optimism_sufficient_v0} as}
\begin{align}  \label{equ:optimism_sufficient_v1}
    {
    X^{*\top} H_t^{-1} \Phi_{t} \mathbf{Z}_{t}^{j}
    \geq (R_{t}^{j} + \bar{R}_{t}) e^{M\bar{D}_{t}} \sqrt{\dot{\mu}_{\text{max}}} \|U_t\|_2 \big\| \widetilde{\theta}_{t}^{j} \big\|_{H_t} + \sqrt{\dot{\mu}_{\text{max}}} \|U_t\|_2 \vert\vert \bar{\theta}_{t} - \theta^{*} \vert\vert_{H_{t}}.
    }
\end{align}

{To further simplify the sufficient condition for optimism \eqref{equ:optimism_sufficient_v1}, we use the concentration results to upper bound $\big\| \widetilde{\theta}_{t}^{j} \big\|_{H_t}$ and $\vert\vert \bar{\theta}_{t} - \theta^{*} \vert\vert_{H_{t}}$.}
{From concentration results \Cref{lemma:concentration_mle} and \Cref{lemma:concentration_pert}, under event $\bar{\mathcal{E}}\bigcap\big(\cap_{t\in [T]} \widetilde{\mathcal{E}}_{1, t}\big)$,} we have
\begin{align*}
    \| \bar{\theta}_{t} - \theta^{*} \|_{H_t} &\leq (1+M \beta_T) \gamma_T, \\ %
    \|\widetilde{\theta}_{t}^{j} \|_{H_{t}} &\leq \dot{\mu}_{\text{max}}^{1/2} \, \dot{\mu}_{\text{min}}^{-1} \, \widetilde{\gamma}_T \sigma_{R},%
\end{align*}
for all $t \in [T]$, {and $\bar{\mathcal{E}}\bigcap\big(\cap_{t\in [T]} \widetilde{\mathcal{E}}_{1, t}\big)$ holds with probability at least $1-2\delta$.}

Then, we obtain the sufficient condition for optimism as follows:
\begin{align}
    \label{equ:optimism_sufficient_v2}
    {
    X^{*\top} H_{t}^{-1} \Phi_{t} \mathbf{Z}_{t}^{j}
    \geq \Big((1+M \beta_T) \gamma_T + (R_{t}^{j} + \bar{R}_{t}) e^{M\bar{D}_{t}} \dot{\mu}_{\text{max}}^{1/2} \, \dot{\mu}_{\text{min}}^{-1} \, \widetilde{\gamma}_T \sigma_{R} \Big) \,\,  \sqrt{\dot{\mu}_{\text{max}}} \|U_t\|_2.
    }
\end{align}
Using the definition $U_{t}^\top = X^{*\top}H_{t}^{-1} \Phi_{t}$ again, we obtain the following expression of the sufficient condition for optimism:
\begin{align*}
    {
    U_{t}^{\top} \mathbf{Z}_{t}^j \geq  \Big((1+M \beta_T) \gamma_T + (R_{t}^{j} + \bar{R}_{t}) e^{M\bar{D}_{t}} \dot{\mu}_{\text{max}}^{1/2} \, \dot{\mu}_{\text{min}}^{-1} \, \widetilde{\gamma}_T \sigma_{R} \Big) \,\,  \sqrt{\dot{\mu}_{\text{max}}} \|U_t\|_2.
    }
\end{align*}
To simplify the notations, define
\begin{align*}
    {
    b_{1} := \sqrt{\dot{\mu}_{\text{max}}} (1+M \beta_T) \gamma_T,
    \quad b_{2} := \dot{\mu}_{\text{max}} \dot{\mu}_{\text{min}}^{-1} \widetilde{\gamma}_T.
    }
\end{align*}
Then, we have the following sufficient condition:
\begin{align*}
    U_{t}^{\top} \mathbf{Z}_{t}^j \geq \Big(b_{1} + b_{2} (R_{t}^{j} + \bar{R}_{t}) e^{M\bar{D}_{t}} \sigma_R\Big) \| U_{t} \|_{2}.
\end{align*}
{This is the final expression of the sufficient condition for optimism.}

We summarize these discussions in the following lemma.
\begin{lemma} [Sufficient Condition for Optimism]
    \label{lemma:optimism_condition}
    {Fix $t \in [T]$ and $j \in [m]$. Define vector $U_{t}^\top := X^{*\top}H_{t}^{-1} \Phi_{t} \in \mathbb{R}^{d+t}$, $b_1 := \sqrt{\dot{\mu}_{\text{max}}} (1+M \beta_T) \gamma_T$ and $b_2 := \dot{\mu}_{\text{max}} \dot{\mu}_{\text{min}}^{-1} \widetilde{\gamma}_T$, define $R_{t}^{j}$ and $\bar{R}_{t}$ as in \eqref{equ:rt_define}. Then, under the concentration event $\bar{\mathcal{E}}\bigcap\big(\cap_{t\in [T]} \widetilde{\mathcal{E}}_{1, t}\big)$, at round $t$, optimism $X^{*\top} \theta^{*} \leq X_{t}^{\top} \theta_{t-1}^{j}$ holds if the following inequality is satisfied:}
    \begin{align}
    \label{equ:final_sufficient_condition_of_optimism}
        U_{t-1}^{\top} \mathbf{Z}_{t-1}^{j} \geq \Big(b_{1} + b_{2} (R_{t-1}^{j} + \bar{R}_{t-1}) e^{M\bar{D}_{t-1}} \sigma_R\Big) \| U_{t-1} \|_{2}.
    \end{align}
\end{lemma}

\subsection{Warm-up Procedure} \label{appendix_sub:warm_up_glm}

{In \Cref{appendix_sub:sufficient}, we obtained a sufficient condition for optimism \eqref{equ:final_sufficient_condition_of_optimism}. While the randomness of added perturbations only appears in $\mathbf{Z}_{t-1}^j$ and we are ready to apply the tail distribution of perturbations to obtain a lower bound of probability of optimism, we should first analyze the order of $e^{M\bar{D}_{t}}$,  $R_{t}^{j}$ and $\bar{R}_{t}$ before proceeding to the optimism analysis. As we will see in the analysis of this section, the order of these parameters are controlled by the warm-up process.}

{We should note that the following analysis is unique to the GLM setting. In the linear bandit setting, the $M$-self-concordant constant $M=0$, thus $R_{t}^{j} = \bar{R}_{t} = 0$, the sufficient condition \eqref{equ:final_sufficient_condition_of_optimism} becomes $U_{t-1}^{\top} \mathbf{Z}_{t-1}^{j} \geq b_1 \| U_{t-1} \|_{2}$. In this setup, we can directly proceed to the analysis of optimism and no warm-up is required.
For this reason, our analysis deviates considerably from the linear ensemble sampling work \citep{lee2024improved} and the analysis of the warm-up procedure is a significant contribution in our theoretical analysis.}

{Recall the following definitions:}
\begin{align*}
    {
    R_t^{j} = \frac{\exp{(M D_{t}^{j})}-1}{M D_{t}^{j}} -1,
    \quad \bar{R}_{t} = e^{M\bar{D}_{t}} - 1,
    }
\end{align*}
{where $D_{t}^{j}$ and $\bar{D}_{t}$ are defined as}
\begin{align*}
    {
    D_{t}^{j} = \max_{l\leq t} \big|X_l^\top(\bar{\theta}_t - \theta_t^j)\big|,
    \quad \bar{D}_t = \max_{l\leq t} \big|X_l^\top(\bar{\theta}_t - \theta^{*})\big|.
    }
\end{align*}
{We also define $\bar{D}_{T} = \max_{t < T} \bar{D}_{t}$.
From these expressions, upper bounding $R_{t}^{j}$ and $\bar{R}_{t}$ is equivalent to upper bounding $D_{t}$ and $\bar{D}_{t}$. With the concentration results \Cref{lemma:concentration_mle} and \Cref{lemma:concentration_pert}, we have}
\begin{align*}
    D_{t}^{j} =& {\,\max_{l<t} \big|X_l^\top(\bar{\theta}_t - \theta_t^j)\big| \leq \max_{l<t} \|X_l\|_{H_t^{-1}} \|\bar{\theta}_t - \theta_t^j\|_{H_t} \leq \dot{\mu}_{\text{max}}^{1/2} \, \dot{\mu}_{\text{min}}^{-1} \, \sigma_{R} \, \widetilde{\gamma}_{T} \,\, \max_{l<t}\|X_l\|_{H_t^{-1}},}\\
    \bar{D}_t =& {\,\max_{l<t} \big|X_l^\top(\bar{\theta}_t - \theta^{*})\big| \leq \max_{l<t} \|X_l\|_{H_t^{-1}} \|\bar{\theta}_t - \theta^{*}\|_{H_t} \leq (1+M \beta_T)\gamma_T \,\, \max_{l<t}\|X_l\|_{H_t^{-1}},}
\end{align*}
{where the first inequality holds due to Cauchy-Schwarz inequality. With these decompositions of $D_{t}$ and $\bar{D}_{t}$, their magnitudes are controlled by the term $\max_{l<t}\|X_l\|_{H_t^{-1}}$. We now use the following warm-up procedure \Cref{alg:warm_up_glm} to control the order of $\max_{l<t}\|X_l\|_{H_t^{-1}}$.}

\begin{algorithm}[H]
\caption{Warm-up of GLM-ES \label{alg:warm_up_glm}}
\textbf{Input:} arm set $\mathcal{X} \subset \mathbb{R}^d, {a} \in (0,1), \tau$
\begin{algorithmic}[1]
    \STATE Set $\zeta = \argmin_{\zeta \in \Delta_{\mathcal{X}}} \max_{X \in \mathcal{X}} \|X\|^2_{V(\zeta)^{-1}}$

    \STATE Pull $X_1, \dots, X_\tau$ \icml{according to}  $\textcolor{blue}{\text{round}(\tau, \zeta, a, \mathcal{X})}$ \label{line:rounding_procedure}

    \STATE Observe rewards $Y_1, \dots, Y_\tau$
\end{algorithmic}
\textbf{Return:} $\{(X_l, Y_l)\}_{l=1}^\tau$
\end{algorithm}

In \texttt{GLM-ES}, \Cref{alg:es_nonlinear} is initialized by calling the warm-up procedure \Cref{alg:warm_up_glm} that approximates a G-optimal design. %
Here, $\zeta \in \Delta_{\mathcal{X}}$ is a probability measure over arm set $\mathcal{X}$, which is set to minimize
\begin{align*}
    \max_{X \in \mathcal{X}} \|X\|^2_{V(\zeta)^{-1}}, \quad \text{where } V(\zeta) = \sum_{X \in \mathcal{X}} \zeta(X) X X^\top.
\end{align*}
We sample $X_1, \dots, X_\tau$ based on G-optimal design $\zeta$ by following the rounding procedure (Line \ref{line:rounding_procedure} in \Cref{alg:warm_up_glm}), \icml{where $N_{i}$ is the number of rounds to pull arm $i$}. The complete rounding procedure is given in \Cref{alg:rounding}, which is proposed in Chapter 12 of \cite{pukelsheim2006optimal} and is detailed in \cite{fiez2019sequential}.

\begin{algorithm}[H]
\caption{Rounding Procedure of
\Cref{alg:warm_up_glm}\label{alg:rounding}}
\textbf{Input:} {$\tau$, $\zeta$, $a \in (0, 1)$, arm set $\mathcal{X} \subset \mathbb{R}^d$}
\begin{algorithmic}[1]
    \STATE $r({a}) \leftarrow \big(d(d+1)/2 + 1\big) / {a}$
    \STATE $N_{i} \leftarrow \lceil (\tau - K/2) \zeta_{i} \rceil$, for all $i \leq K$
    \WHILE{$\sum_{i=1}^{K} N_{i} \neq \tau$}
        \IF{$\sum_{i=1}^{K}N_{i} < \tau$}
            \STATE $j \leftarrow \text{argmin}_{i \leq K} (N_{i} - 1) / \zeta_{i}$
            \STATE $N_{j} \leftarrow N_{j} + 1$
        \ENDIF
        \IF{$\sum_{i=1}^{K}N_{i} > \tau$}
            \STATE $j \leftarrow \text{argmax}_{i \leq K} (N_{i} - 1) / \zeta_{i}$
            \STATE $N_{j} \leftarrow N_{j} - 1$
        \ENDIF
    \ENDWHILE
    \STATE $N_{i} \leftarrow \text{max} \big( N_{i}, r({a})/K \big)$
\end{algorithmic}
\textbf{Return:} $N_{1}, ..., N_{K}$ for all $i \leq K$
\end{algorithm}

{The following technical lemma demonstrates how the warm-up procedure \Cref{alg:warm_up_glm} can control the upper bound of the term $\max_{l<t}\|X_l\|_{H_t^{-1}}$.}
\begin{lemma}[Warm-up]
\label{lem:warm_up_procedure}
Let $\zeta$ be the $G$-optimal design solution over a compact feature set $\mathcal{X}$ whose span is $\mathbb{R}^d$. Let $\iota>0$ and {$a=0.5$}. By setting the number of warm-up rounds as $\tau=\lceil 1.5 \iota^2 d / \dot{\mu}_{\text{min}}\rceil$, \Cref{alg:warm_up_glm} returns a dataset $\{(X_l, Y_l)\}_{l=1}^\tau$ such that for all $t>\tau$, we have
\begin{align}
\label{equ:warm_up_result}
    \max_{X \in \mathcal{X}}\|X\|_{\rev{H_t^{-1}}} \leq 1/\iota.
\end{align}
\end{lemma}

\begin{proof}[Proof of \Cref{lem:warm_up_procedure}]
The proof mainly follows the proof of \citep[Lemma 4.10]{liu2023glm}, we present it here for completeness of our analysis. According to definition \eqref{equ:hessian_definition}, for all $t>\tau$, we have $H_t^{-1} \preceq H_\tau^{-1}$.
Then, we have
\begin{align*}
    \rev{\max_{X \in \mathcal{X}} \|X\|_{H_t^{-1}}^2 \leq \max_{X \in \mathcal{X}} \|X\|_{H_\tau^{-1}}^2.}
\end{align*}
We further define $H(\zeta, \theta)=\sum_{X \in \mathcal{X}} \zeta(X) \dot{\mu}(X^{\top} \theta) X X^{\top}$ where $\zeta$ is a probability measure over $\mathcal{X}$. Then it is sufficient to prove that for all $t>\tau$, \rev{$\max_{X \in \mathcal{X}} \|X\|_{H_\tau^{-1}} \leq 1 / \iota$} holds. We have the following results:
\begin{align*}
    \rev{\max_{X \in \mathcal{X}} \|X\|_{H_\tau^{-1}}^2} &\leq \rev{\frac{1+\epsilon}{\tau} \max_{X \in \mathcal{X}} \|X\|_{H^{-1}(\zeta, \theta^{*})}^2} \\
    &\leq \frac{1+\epsilon}{\tau \dot{\mu}_{\text{min}}} \max_{X \in \mathcal{X}} \|X\|_{V^{-1}(\zeta)}^2 \\
    &\leq \frac{d(1+\epsilon)}{\tau \dot{\mu}_{\text{min}}},
\end{align*}
where the first inequality holds because of \citet[Lemma 13]{jun2021improved}, the second inequality holds because \rev{$H(\zeta,\theta^{*}) \succeq \dot{\mu}_{\text{min}} V(\zeta)$} and the last inequality holds due to \citet[Theorem 21.1 (Kiefer–Wolfowitz)]{lattimore2020bandit}. The proof is completed by setting $\tau=\iota^2 d(1+\epsilon)/\dot{\mu}_{\text{min}}$.
\end{proof}

{According to \Cref{lem:warm_up_procedure}, by setting the number of warm-up rounds $\tau=\lceil 1.5 \iota^2 d / \dot{\mu}_{\text{min}}\rceil$, we obtain the upper bound $\max_{X \in \mathcal{X}}\|X\|_{H_t^{-1}} \leq 1/\iota$. With this upper bound, we can write}
\begin{align}  
    \label{equ:dt_decomposition}
    D_{t}^{j} &\leq \,\dot{\mu}_{\text{max}}^{1/2} \, \dot{\mu}_{\text{min}}^{-1} \, \sigma_{R} \, \widetilde{\gamma}_{T} \,\, \max_{l<t}\|X_l\|_{H_t^{-1}} \leq \,\dot{\mu}_{\text{max}}^{1/2} \, \dot{\mu}_{\text{min}}^{-1} \, \sigma_{R} \, \widetilde{\gamma}_{T} / \iota, \\
    \label{equ:bar_dt_decomposition}
    \bar{D}_t &\leq \,(1+M \beta_T)\gamma_T \,\, \max_{l<t}\|X_l\|_{H_t^{-1}} \leq (1+M \beta_T)\gamma_T / \iota.
\end{align}
Using these expressions, we can first choose the desired order of $D_{t}^{j}$ and $\bar{D}_{t}$, then obtain the lower bound of $\iota$ according to \eqref{equ:dt_decomposition} and \eqref{equ:bar_dt_decomposition}. Then, using the relation $\tau=\lceil 1.5 \iota^2 d / \dot{\mu}_{\text{min}}\rceil$, we obtain the lower bound of the number of warm-up rounds.
In this section, we discuss the requirement for the order of $\bar{D}_{t}$ and the resulting lower bound for warm-up rounds. The requirement of $R_{t}^{j}$ and $\bar{R}_{t}$ appears in the analysis of optimism, thus we leave the discussion in the next section.

{To upper bound $e^{M\bar{D}_{t}}$, we desire that $\bar{D}_{t} = \mathcal{O}(1)$ so that $e^{M\bar{D}_t/2}$ can be upper bounded by a constant. According to \eqref{equ:bar_dt_decomposition}, for all $t < T$, $\bar{D}_{t}$ is upper bounded by $(1+M \beta_T)\gamma_T / \iota$. Therefore, we have}
\begin{align*}
    {
    \bar{D}_{T} = \max_{t < T} \bar{D}_{t} \leq (1+M \beta_T)\gamma_T / \iota.
    }
\end{align*}
{Then, we have the following requirement for $\iota$:}
\begin{align*}
    {
    \iota \geq M (1+M \beta_T)\gamma_T
    \quad\rightarrow\quad \iota = \Omega\big(d^{1/2}\big),
    }
\end{align*}
{where we used the fact that if we set $\lambda = 1 \vee d \vee \log(1/\delta)$, according to the definition of parameters $\gamma_{T}$ and $\beta_{T}$ as in \Cref{lemma:concentration_mle}, we have $\gamma_{T} \sim \widetilde{\mathcal{O}}\big(\sqrt{d}\big)$, $\beta_{T} \sim \widetilde{\mathcal{O}}\big(1\big)$.
Then, by setting $\tau=\lceil 1.5 \iota^2 d / \dot{\mu}_{\text{min}}\rceil$, we obtain the lower bound of warm-up rounds as}
\begin{align*}
    {
    \tau=\lceil 1.5 \iota^2 d / \dot{\mu}_{\text{min}}\rceil = \Omega\big( d^{2} \big).
    }
\end{align*}
{By setting warm-up rounds $\tau = \Omega(d^{2})$, $e^{M\bar{D}_T}$ is upper bounded by a constant $e$. We obtain the following adapted version of \Cref{lemma:optimism_condition}.}
\begin{lemma} [Sufficient Condition for Optimism with Warm-up]
    \label{lemma:optimism_condition_with_warm_up}
    {Fix $t \in [T]$, $j \in [m]$ and $\lambda = 1 \vee d \vee \log(1/\delta)$. Define vector $U_{t}^\top = X^{*\top} H_{t}^{-1} \Phi_{t}$, $b_1 = \sqrt{\dot{\mu}_{\text{max}}} (1+M \beta_T) \gamma_T$ and $b_2 = \dot{\mu}_{\text{max}} \dot{\mu}_{\text{min}}^{-1} \widetilde{\gamma}_T$.
    Then, with warm-up rounds $\tau = \Omega(d^{2})$ and under the concentration event $\bar{\mathcal{E}}\bigcap\big(\cap_{t\in [T]} \widetilde{\mathcal{E}}_{1, t}\big)$, optimism $X^{*\top} \theta^{*} \leq X_{t}^{\top} \theta_{t-1}^{j}$ holds if the following inequality is satisfied:}
    \begin{align}  \label{equ:sufficient_with_warm_up}
        {
        U_{t-1}^{\top} \mathbf{Z}_{t-1}^{j} \geq \Big(b_{1} + b_{2} (R_{t-1}^{j} + \bar{R}_{t-1}) e \sigma_R\Big) \| U_{t-1} \|_{2}.
        }
    \end{align}
\end{lemma}

\subsection{Proof of \Cref{lemma:optimism_glm} (Probability of Optimism)}
\label{appendix_sub:optimism}

{With the warm-up procedure \Cref{alg:warm_up_glm} and \Cref{lemma:optimism_condition_with_warm_up}, we are ready to prove the probability of optimism.
With the sufficient condition of optimism \Cref{lemma:optimism_condition_with_warm_up} and the fact that the randomness of added perturbations only appears in $\mathbf{Z}_{t}^{j}$, we can apply the tail bound of perturbations to obtain the probability of optimism.
The counting sequence procedure in this section is adapted from \cite{lee2024improved} and is a standard procedure in ensemble sampling analysis, we include the detailed derivations to make our proof self-contained. The core novelty of our proof of optimism is the analysis of the sufficient condition of optimism \Cref{lemma:optimism_condition_with_warm_up} and warm-up procedure. Note that the requirement \eqref{equ:sufficient_with_warm_up} is fundamentally different from linear bandit settings and we need novel warm-up procedure and analysis to guarantee the inequality.}

Recall that we use Gaussian noise $Z_l^j \sim \mathcal{N}(0, \sigma_{R}^{2})$ for reward perturbation, thus we can apply the following tail bound: for any $U \in \mathbb{R}^d$ and $\Zb \sim \mathcal{N}(0, \sigma_{R}^2 I_d)$,  we have the following inequality
\begin{align*}
    \mathbb{P} \big( U^{\top} \Zb \geq \sigma_{R} \| U \|_{2} \big)
    \geq p_{N},
\end{align*}
where $p_{N} = 1 - \Phi(1) \approx 0.16$.
{
Based on \eqref{equ:sufficient_with_warm_up}, with the definitions of upper bounds $R_{t-1}^{j} \leq R^{j}$, $\bar{R}_{t-1} \leq \bar{R}$, where $R^{j}$ and $\bar{R}$ are constants, we obtain the sufficient condition
\begin{align}  \label{equ:sufficient_with_constants}
    U_{t-1}^{\top} \mathbf{Z}_{t-1}^{j} \geq \Big(b_{1} + b_{2} (R^{j} + \bar{R}) e \sigma_R\Big) \| U_{t-1} \|_{2}.
\end{align}
}
From \eqref{equ:sufficient_with_constants}, if we fix the sequence of pulled arms $\{X_{l}\}_{l=1}^{t}$ and element $j \in [m]$, the only randomness comes from $\mathbf{Z}_{t-1}^{j}$. Then, \eqref{equ:sufficient_with_constants} holds with probability at least $p_N$ if we choose constant variance $\sigma_{R}$ that satisfies
\begin{align}
\label{equ:sigma_condition}
    {
    \sigma_{R} \geq b_{1} + b_{2} (R^{j} + \bar{R}) e \sigma_{R}.
    }
\end{align}
{We should note that this is a non-trivial requirement for perturbation variance $\sigma_{R}$ and we need to carefully control $R^{j}$ and $\bar{R}$ throught warm-up to guarantee that there is a solution of $\sigma_{R}$ in \eqref{equ:sigma_condition}. In the following discussions, we analyze how to choose $\sigma_{R}$ to satisfy \eqref{equ:sigma_condition} and the corresponding requirement for warm-up rounds $\tau$.}

Recall from the definition of $R_{t}^{j}$ and $\bar{R}_{t}$, we have
\begin{align*}
    R_{t}^{j} &= \frac{\exp{(M D_{t}^{j})}-1}{M D_{t}^{j}} -1 \leq \frac{1 + M D_{t}^{j} + (M D_{t}^{j})^2 - 1}{M D_{t}^{j}} -1 = M D_{t}^{j},  \\
    \bar{R}_{t} &= e^{M\bar{D}_{t}} - 1 \leq (e-1)M\bar{D}_{t},
\end{align*}
where the inequality holds when $M D_{t}^{j} \leq 1$, $M \bar{D}_{t} \leq 1$.
To guarantee that solution exists for \eqref{equ:sigma_condition}, we desire
\begin{align*}
    R_{t}^{j} \leq \sigma_{R}^{-\epsilon},\,\, \bar{R}_{t} \leq \sigma_{R}^{-\epsilon},\,\,\epsilon \in (0,1]
    \quad\rightarrow\quad MD_{t}^{j} \leq \sigma_{R}^{-\epsilon} \leq 1,\,\,
    (e-1)M\bar{D}_{t} \leq \sigma_{R}^{-\epsilon} \leq 1.
\end{align*}
{From the upper bounds \eqref{equ:dt_decomposition} and \eqref{equ:bar_dt_decomposition}, we immediately have}
\begin{align*}
    \iota = \Omega \Big( \max\big\{M\dot{\mu}_{\text{max}}^{1/2} \, \dot{\mu}_{\text{min}}^{-1} \, \sigma_{R}^{1 + \epsilon} \, \widetilde{\gamma}_{T}, (e-1)M((1+M\beta_{T})\gamma_{T} \sigma_{R}^{\epsilon} \big\} \Big) = \Omega\big( \sigma_{R}^{1+\epsilon} \, (d \log T)^{1/2} \big).
\end{align*}
{where we used the definition of $\gamma_{T}$, $\beta_{T}$ and $\widetilde{\gamma}_{T}$, as in \Cref{lemma:concentration_mle} and \Cref{lemma:concentration_pert}, and the chosen regularization parameter $\lambda$ as in \Cref{lemma:optimism_condition_with_warm_up}.}
{The corresponding warm-up rounds}
\begin{align*}
    \tau=\lceil 1.5 \iota^2 d / \dot{\mu}_{\text{min}}\rceil = \Omega\big( \sigma_{R}^{2+2\epsilon} d^{2} \log T \big).
\end{align*}
With $R^{j} \leq \sigma_{R}^{-\epsilon}$, $\bar{R} \leq \sigma_{R}^{-\epsilon}$, it suffices to satisfy
\begin{align}  \label{equ:sigma_R_eq}
    \sigma_{R} \geq b_{1} + 2eb_{2} \sigma_{R}^{1-\epsilon}.
\end{align}
{Since $0 < \epsilon \leq 1$, \eqref{equ:sigma_R_eq} is guaranteed to have a solution for $\sigma_{R}$.
In general, \eqref{equ:sigma_R_eq} does not have a closed-form solution, a sufficient condition for \eqref{equ:sigma_R_eq} to hold is:}
\begin{align*}
    \sigma_{R} \geq \max \big\{ 2b_{1},\,\, \big(4eb_{2}\big)^{1/\epsilon} \big\}.
\end{align*}
{From the definitions of $b_{1}$, $b_{2}$ in \Cref{lemma:optimism_condition_with_warm_up}, we have the following sufficient condition:}
\begin{align*}
    \sigma_{R} = \Omega\big((d \log T)^{1/(2\epsilon)}\big)
    \quad\rightarrow\quad
    \tau = \Omega\big(d^{3 + 1/\epsilon} (\log T)^{2 + 1/\epsilon}\big).
\end{align*}
In our algorithm design, we choose $\epsilon=1$, then we have the following requirements for $\sigma_{R}$ and $\tau$:
\begin{align*}
    \sigma_{R} = \Omega\big((d\log T)^{1/2}\big),
    \quad \tau = \Omega \big(d^{4} (\log T)^{3}\big).
\end{align*}
Combining the requirement of warm-up rounds $\Omega\big(d^{2}\big)$ in \Cref{lemma:optimism_condition_with_warm_up}, the final requirement for $\tau$ is $\tau = \Omega \big(d^{4} (\log T)^{3}\big)$.

{In summary, with warm-up rounds $\tau = \widetilde{\Omega}\big(d^{4}\big)$ and constant perturbation variance $\sigma_{R} = \widetilde{\Omega}\big(d^{1/2}\big)$, \eqref{equ:sigma_condition} is satisfied. We have the following lemma.}
\begin{lemma} [Probability of Optimism under Fixed Arm Sequence]
    \label{lemma:optimism_prob_fixed}
    {Fix $t \in [T]$, $j \in [m]$ and $\lambda = 1 \vee d \vee \log(1/\delta)$. Fix the sequence of pulled arms $\{X_{l}\}_{l=1}^{t-1}$. Then, with warm-up rounds $\tau = \widetilde{\Omega} \big(d^{4}\big)$, perturbation distribution $\mathcal{P}_{R} = \mathcal{N}(0, \sigma_{R}^{2})$ with $\sigma_{R} = \widetilde{\Omega} \big(d^{1/2}\big)$, the probability of optimism $X^{*\top} \theta^{*} \leq X_{t}^{\top} \theta_{t-1}^{j}$ is lower bounded as follows:}
    \begin{align*}
        {
        \mathbb{P} \big( X^{*\top} \theta^{*} \leq X_{t}^{\top} \theta_{t-1}^{j} \big)
        \geq p_{N},
        }
\end{align*}
{where $p_{N} = 1 - \Phi(1) \approx 0.16$, the probability is measured over the randomness of the perturbation sequence $\{Z_{l}^{j}\}_{l=1}^{t-1}$.}
\end{lemma}

{Starting from \Cref{lemma:optimism_prob_fixed}, we adapt the counting arm sequence number procedure as introduced by \citet{lee2024improved} to prove the probability of optimism without fixing arm sequence.}
We define the indicator function that optimism and concentration holds at the beginning of round $t$: $I_{t}^{j} := \ind\big\{(U_{t-1}^{\top} \mathbf{Z}_{t-1}^{j} \geq (b_1 + b_2(R^{j} + \bar{R}) e \sigma_{R}) \| U_{t-1} \|_{2}) \bigcap \widetilde{\mathcal{E}}_{1, t}\big\}$.
According to \Cref{lemma:concentration_pert}, $\widetilde{\mathcal{E}}_{1, t}$ holds with probability at least $1 - \delta/T$, thus we have
\begin{align*}
    \mathbb{P} \big(I_{t}^{j} = 1\big) \geq \mathbb{P} \big(U_{t-1}^{\top} \mathbf{Z}_{t-1}^{j} \geq (b_1 + b_2 (R^{j} + \bar{R}) e \sigma_{R}) \| U_{t-1} \|_{2}\big) - \frac{\delta}{T} \geq p_{N} - \frac{\delta}{T} \geq \frac{p_{N}}{2},
\end{align*}
where we assumed that $\delta / T \leq p_{N} / 2$.
Note that $I_{t}^{j}$ is $\mathcal{F}_{t-1}$-measurable: given the history up to round $(t-1)$, the value of $I_{t}^{j}$ is determined for each model $j\in[m]$.
Now we add the randomness of choosing arm $j_{t}$ at round $t$.
\rev{Conditioning on $\mathcal{F}_{t-1}$, the probability of optimism after uniformly randomly choosing \rev{model} $j_{t}$ is given by}
\begin{align*}
    \mathbb{P} \Big( \big(U_{t-1}^{\top} \mathbf{Z}_{t-1}^{j_{t}} \geq \big(b_1 + b_2 (R^{j} + \bar{R}) e \sigma_{R}\big) \| U_{t-1} \|_{2}\big) \bigcap \widetilde{\mathcal{E}}_{1, t} \,\,|\,\, \mathcal{F}_{t-1} \Big) = \frac{1}{m} \sum_{j=1}^{m} I_{t}^{j}.
\end{align*}
Using Azuma-Hoeffding inequality, we have the following result:
\begin{align*}
    \mathbb{P} \Big( \frac{1}{m} \sum_{j=1}^{m} I_{t}^{j} < \frac{p_{N}}{4}\Big) \leq \text{exp} \Big( -\frac{p_{N}^{2}m}{8} \Big).
\end{align*}

The results above are obtained under fixed sequence of pulled arms $\{X_{l}\}_{l=1}^{t}$, we use the following result to count equivalent permutations of $\{X_{l}\}_{l=1}^{t}$.
\begin{lemma} [Claim 1 in \cite{lee2024improved}]
    There exists an event $\mathcal{E}^{*}$ such that under $\mathcal{E}^{*}$, $\frac{1}{m}\sum_{j=1}^{m} I_{t}^{j} \geq p_{N} / 4$ holds for all $t \in [T]$, and $\mathbb{P}\big( \bar{\mathcal{E}}^{*} \big) \leq T^{K} \text{exp} \big(-p_{N}^{2}m / 8\big)$.
\end{lemma}

Therefore, by choosing ensemble size
\begin{align*}
    m \geq \frac{8}{p_{N}^{2}} \Big( K \text{log}T + \text{log}\frac{1}{\delta} \Big),
\end{align*}
we have $\mathbb{P}\big( \bar{\mathcal{E}}^{*} \big) \leq \delta$: with high probability, optimism holds with constant probability.
This completes the proof.

\subsection{Proof of Optimism without $K$-dependency}

From \Cref{lemma:optimism_condition}, the sufficient condition for optimism can be written as
\begin{align*}
    U_{t-1}^{\top} \mathbf{Z}_{t-1}^{j} \geq \Big(b_{1} + b_{2} (R_{t-1}^{j} + \bar{R}_{t-1}) e^{M\bar{D}_{t-1}} \sigma_R\Big) \| U_{t-1} \|_{2},
\end{align*}
where $U_{t}^\top := X^{*\top}H_{t}^{-1} \Phi_{t} \in \mathbb{R}^{d+t}$. To apply the technical lemmas in \citep{pmlr-v336-janz26a}, we introduce the following notations.
Recall that we use Gaussian initial perturbation $Z_{0}^{j} \sim \mathcal{N}(0, \sigma_{R}^{2} I_{d})$ for $j \in [m]$ and the definitions
\begin{align*}
    \mathbf{Z}_{t}^{j} := \big(Z_{0}^{j}, Z_{1}^{j}, \cdots, Z_{t}^{j} \big)^{\top} \in \mathbb{R}^{d+t},
    \quad
    \Phi_{t} := \big(\sqrt{\lambda/\dot{\mu}_{\text{min}}} I_d , X_{1}, \cdots, X_{t}\big) \in \mathbb{R}^{d \times (d+t)},
\end{align*}
we then define
\begin{align*}
    P_{t}^{j} := \Phi_{t} \mathbf{Z}_{t}^{j} = \sqrt{\lambda/\dot{\mu}_{\text{min}}} Z_{0}^{j} + \sum_{l=1}^{t} X_{l}Z_{l}^{j},
    \quad S_{t}^{j} := P_{t}^{j} / \sigma_{R} = \sqrt{\lambda/\dot{\mu}_{\text{min}}} \zeta^{j} + \sum_{l=1}^{t} X_{l} \eta_{l}^{j},
\end{align*}
where $\zeta^{j} \sim \mathcal{N}(0, I_{d})$, $\eta_{l}^{j} \sim \mathcal{N}(0, 1)$.
Next, we define
\begin{align*}
    u_{t} := H_{t}^{-1} X^{*} \in \mathbb{R}^{d}.
\end{align*}
Recall the relation from \eqref{equ:gt_definition}
\begin{align*}
    \Phi_{t}\Phi_{t}^{\top}
    = \sum_{l=1}^{t}X_{l}X_{l}^{\top} + \frac{\lambda}{\dot{\mu}_{\text{min}}} I_{d} = G_{t},
\end{align*}
we can rewrite the sufficient condition for optimism as
\begin{align*}
   \langle u_{t-1}, P_{t-1}^{j} \rangle \geq \Big(b_{1} + b_{2} (R_{t-1}^{j} + \bar{R}_{t-1}) e^{M\bar{D}_{t-1}} \sigma_R\Big) \| u_{t-1} \|_{G_{t-1}}.
\end{align*}
This is equivalent to
\begin{align*}
    \frac{\langle u_{t-1}, S_{t-1}^{j} \rangle}{\| u_{t-1} \|_{G_{t-1}}}
    \geq \frac{b_{1}}{\sigma_{R}} + b_{2} (R_{t-1}^{j} + \bar{R}_{t-1}) e^{M\bar{D}_{t-1}}.
\end{align*}
Starting from this new expression of sufficient condition for optimism, we can utilize the following theorem from \citep{pmlr-v336-janz26a} to lower bound the probability of optimism.
\begin{lemma} [Corollary 5.6 of \citep{pmlr-v336-janz26a}]
    \label{lem:optimism_janz}
    For $u \neq 0$, $c, t, m > 0$, define
    \begin{align*}
        E_{t, m}(u, c) = \frac{1}{m} \sum_{j=1}^{m} \ind \Big\{ \langle u, S_{t-1}^{j} \rangle / \|u\|_{V_{t-1}} \geq c \Big\},
    \end{align*}
    \begin{align*}
        \text{where}
        \quad S_{t-1}^{j} = \sqrt{\lambda} \zeta^{j} + \sum_{l=1}^{t-1} X_{l} \eta_{l}^{j}, \,\, j \in [m],
        \qquad V_{t-1} = \lambda I + \sum_{l=1}^{t-1} X_{l}X_{l}^{\top},
    \end{align*}
    $\zeta \sim \mathcal{N}(0, I_{d})$, $\eta_{l}^{j} \sim \mathcal{N}(0, 1)$ are i.i.d. sampled standard Gaussian noises.
    Assume $\lambda \geq 1$, $\bar{\gamma} = 40$, $T \geq \max(2, d)$, $\delta \in (0, 1/2)$. Write $l := \max \big\{1, \log \big(T / \delta\big) \big\}$. If $m \geq 2000 dl$, then with probability at least $1 - \delta$,
    \begin{align*}
        \min_{t \in [T]} \inf_{u \in \mathbb{S}_{2}^{d-1}} E_{t, m} (u, 1 / \bar{\gamma}) \geq \frac{1}{10}.
    \end{align*}
\end{lemma}

Intuitively, the quantity $E_{t, m}(u, c)$ gives the probability that optimism holds when selecting the ensemble member $j \in [m]$ uniformly randomly. When the right-hand-side parameter $c \leq 1/40$ and ensemble size $m \geq 200dl$, the probability of optimism is lower bounded as $1 / 10$.
Now we verify that the assumptions of \Cref{lem:optimism_janz} holds in our analysis.
First, $Z_{0}^{j}$ and $Z_{l}^{j}$ are independent Gaussian variables and $X_{l}$ is chosen before the fresh perturbations $Z_{l}^{j}$.
According to the definition of $S_{t}^{j}$, for every vector $u$,
\begin{align*}
    \langle u, S_{t}^{j}\rangle = \sqrt{\lambda/\dot{\mu}_{\text{min}}} \langle u, \zeta^{j}\rangle + \sum_{l=1}^{t} \langle u, X_{l} \rangle \eta_{l}^{j}.
\end{align*}
Then, the common quadratic-variation clock is
\begin{align*}
    (\lambda/\dot{\mu}_{\text{min}}) \|u\|_{2}^{2} + \sum_{l=1}^{t} \langle u, X_{l} \rangle^{2} = \|u\|_{G_{t}}^{2}.
\end{align*}
Therefore, \Cref{lem:optimism_janz} applies to our analysis.

Next, we consider the condition that
\begin{align*}
    \frac{b_{1}}{\sigma_{R}} + b_{2} (R_{t-1}^{j} + \bar{R}_{t-1}) e^{M\bar{D}_{t-1}} \leq c = 1 / 40.
\end{align*}
From previous results,
\begin{align*}
    R_{t}^{j} \leq \sigma_{R}^{-1},\quad
    \bar{R}_{t} \leq \sigma_{R}^{-1},\quad
    (e-1)M\bar{D}_{t} \leq \sigma_{R}^{-1} \leq 1.
\end{align*}
Then, a sufficient condition can be written as
\begin{align*}
    \frac{b_{1}}{\sigma_{R}} + \frac{2 b_{2}}{\sigma_{R}} e \leq 1 / 40.
\end{align*}
We can choose $\sigma_{R}$ as follows to satisfy this condition:
\begin{align*}
    \sigma_{R} \geq 40 (b_{1} + 2eb_{2}).
\end{align*}
Compared to the requirement \eqref{equ:sigma_R_eq} with $\epsilon = 1$, we have identical requirement for $\sigma_{R}$ except for the constant factors.
Therefore, we have the following requirements for perturbation variance $\sigma_{R}$ and warm-up round $\tau$:
\begin{align*}
    \sigma_{R} = \Omega\big((d\log T)^{1/2}\big),
    \quad \tau = \Omega \big(d^{4} (\log T)^{3}\big).
\end{align*}

Applying \Cref{lem:optimism_janz}, given that
\begin{align*}
    m \geq 2000 d \log (T / \delta),
\end{align*}
we have
\begin{align*}
    \min_{t \leq T} \inf_{u \neq 0} \frac{1}{m} \sum_{j=1}^{m} \ind \Big\{ \frac{\langle u, S_{t}^{j} \rangle}{\| u \|_{G_{t}}} \geq \frac{1}{40} \Big\} \geq \frac{1}{10},
\end{align*}
with probability at least $1 - \delta$.
Although $u_{t}$ is random and depends on the adaptive history, this result is uniform over all directions. Therefore, it can be substituted after observing the trajectory.

Recall that we defined constant $p_{N} = 1 - \Phi(1) \approx 0.16$ in \Cref{lemma:concentration_glm_es}. Since $p_{N} / 4 \leq 1 / 10$, we have the following lemma.

\begin{lemma}[Optimism without $K$-dependency]
Let constant $p_{N} = 1 - \Phi(1) \approx 0.16$, and define a sequence of events as
\begin{align*}
    \mathcal{E}_{2, t} = \Big\{ \mathbb{P} \Big(X^{*\top}\theta^{*} \leq X_{t}^{\top} \rev{\theta_{t-1}^{j_{t}}} \text{ and } \widetilde{\mathcal{E}}_{1, t}  \,\vert\, \mathcal{F}_{t-1}\Big) \geq p_{N} / 4 \Big\},
\end{align*}
and their intersection $\mathcal{E}_{2} = \bigcap_{t=\tau}^{T} \mathcal{E}_{2, t}$.
Then, with ensemble size $m \geq 2000 d \log (T / \delta)$, regularization $\lambda = 1 \vee \dot{\mu}_{\min} \vee d \vee \log(1/\delta)$, perturbation variance $\sigma_{R} = \widetilde{\Omega} \big(d^{1/2}\big)$ and warm-up rounds $\tau = \widetilde{\Omega}\big(d^{4}\big)$, $\mathcal{E}_{2}$ holds with probability at least $1 - \delta$.
\end{lemma}

\subsection{Proof of \Cref{lemma:gt_define}}

The proof of this technical lemma mostly follows the proof of Lemma 6 in \citet{lee2024improved}, we add the detailed proof here for completeness.
{While the technical lemma in \citet{lee2024improved} applies to linear bandit sertting, we extend the result to generalized linear setting.}
From the definition of $g_{t}(\theta^{*}, \mathcal{E}')$:
\begin{align*}
    g_{t}(\theta^{*}, \mathcal{E}') = (J(\theta^{*}) - J(\theta_{t}^{-})) \ind\{\mathcal{E}'\},
    \quad J(\theta_{t}^{-}) = \text{inf}_{\theta \in \Theta_{t}} J(\theta).
\end{align*}
Therefore, we have
\begin{align*}
    J(\theta^{*}) \geq J(\theta_{t}^{-}) \quad \rightarrow \quad g_{t}(\theta^{*}, \mathcal{E}') \geq 0,
\end{align*}
holds for any event $\mathcal{E}^{\prime} \in \mathcal{F}_{t}$.
Now we consider event $\mathcal{E}^{\prime\prime} \subset \mathcal{E}_{1, t}$.
If $\mathcal{E}^{\prime\prime}$ does not hold, $ g_{t}(\theta_{t}, \mathcal{E}^{\prime\prime})=0 \geq 0$.
If $\mathcal{E}^{\prime\prime}$ holds, from the concentration result \Cref{lemma:concentration_glm_es}, we have $\theta_{t} \in \Theta_{t}$. Then we have
\begin{align*}
    J(\theta_{t}) \geq J(\theta_{t}^{-}) \quad \rightarrow \quad g_{t}(\theta_{t}, \mathcal{E}^{\prime\prime}) \geq 0.
\end{align*}
Therefore, $g_{t} (\theta_{t}, \mathcal{E}'') \geq 0$ holds almost surely for any event such that $\mathcal{E}'' \subset \mathcal{E}_{1, t}$.
This completes the proof.

\vspace{3mm}

\section{Proof of Regret Bound of Neural-ES}  \label{appendix:neural_es_bound}

{In this section, we prove the high-probability regret bound of \texttt{Neural-ES} \Cref{theorem:neural-es}. The structure of this section is as follows.
In \Cref{appendix_sub:neural_es_preliminary}, we introduce the standard notations and assumptions in the neural contextual bandit literature, as well as the corresponding basic technical results required in the following analysis.
In \Cref{appendix_sub:neural_es_lemmas}, we list the technical lemmas required in the proof of the high-probability regret bound. These technical lemmas are proved in \Cref{appendix:neural_es_lemmas}.
With the technical lemmas listed in \Cref{appendix_sub:neural_es_lemmas}, we present the proof of regret bound \Cref{theorem:neural-es} in \Cref{appendix_sub:neural_es_proof}.}

\subsection{Preliminary Analysis}  \label{appendix_sub:neural_es_preliminary}

In this section, we introduce the standard notations and definitions in the neural contextual bandit literature.
We adopt a fully connected neural network $f(X; \theta): \RR^{d} \rightarrow \RR$ to approximate the true reward model $h(X)$:
\begin{align*}
    f(X; \theta) := \sqrt{N} \, W_{L} \phi \bigg( W_{L - 1} \phi \big( \cdot\cdot\cdot \phi(W_{1}X) \big) \bigg),
\end{align*}
where $N$ is the width of each layer, $\theta = [\text{vec}(W_{1}), \cdot\cdot\cdot, \text{vec}(W_{L})] \in \mathbb{R}^{d'}$ is the concatenation of all learnable parameters, {$\phi(x) = \text{ReLU}(x)$ is the activation function. For convenience of our analysis, we assume that the widths of each layer are the same. The dimensions of the matrices are as follows: $W_{1} \in \mathbb{R}^{N \times d}$, $W_{L} \in \mathbb{R}^{1 \times N}$, $W_{2}, ..., W_{L-1} \in \mathbb{R}^{N \times N}$. Therefore, the dimension of $\theta$ is given by $d' = N + Nd + N^{2}(L-2)$.
We use the following expression:}
\begin{align*}
g(X; \theta_{0}) := \nabla_{\theta} f(X; \theta) \big|_{\theta=\theta_{0}} \in \mathbb{R}^{d'}
\end{align*}
to denote the gradient at the initial neural network parameter.
For convenience and comparison with linear contextual bandit, we define the following normalized gradient mapping:
\begin{align}
    \psi(X) := \frac{g(X;\theta_{0})}{\sqrt{N}}:
    \quad \RR^{d} \rightarrow \RR^{d'}.
\end{align}
{As we will see in the following analysis, after we map the $d$-dimensional feature vector $X$ into the $d'$-dimensional $\psi(X)$, the analysis of \texttt{Neural-ES} follows similar principles as in \texttt{Lin-ES}.}

{Recall that $X_{t} \in \mathbb{R}^{d}$ is the pulled arm at round $t$, $Y_{t} \in \mathbb{R}$ is the observed reward at round $t$.
Given input data set $\mathcal{D}_{t}=\{(X_l,Y_l)\}_{l=1}^t$, we define the following $\lambda$-regularized loss function:}
\begin{align}  \label{equ:neural_likelihood}
    {
    L_{\text{Neural}}(\theta; \mathcal{D}_{t})
    := \frac{\lambda}{2}N\|\theta - \theta_{0}\|_{2}^{2} + \sum_{l=1}^{t} \frac{1}{2} \Big( f(X_{l}; \theta) - Y_{l} \Big)^{2}.
    }
\end{align}
{Then we have the gradient}
\begin{align*}
    {
    \nabla_{\theta} L_{\text{Neural}}(\theta; \mathcal{D}_{t})
    = \lambda N (\theta - \theta_{0}) + \sum_{l=1}^{t} \Big( f(X_{l}; \theta) - Y_{l} \Big) g(X_{l}; \theta).
    }
\end{align*}
{Due to non-linearity, in general, we do not have a closed-form solution for the MLE of loss function \eqref{equ:neural_likelihood}.}
{For the convenience of further analysis, we derive the following approximations of the MLE of $\theta$.
Under linear approximation, $f(X; \theta) \approx f(X; \theta_{0}) + g(X; \theta_{0})^{\top} (\theta - \theta_{0})$. Then, the loss function can be approximated as:}
\begin{align*}
    {
    \widetilde{L}_{\text{Neural}}(\theta; \mathcal{D}_{t})
    := \frac{\lambda}{2}N\|\theta - \theta_{0}\|_{2}^{2} + \sum_{l=1}^{t} \frac{1}{2} \Big( f(X_{l}; \theta_{0}) + g(X_{l}; \theta_{0})^{\top} (\theta - \theta_{0}) - Y_{l} \Big)^{2}.
    }
\end{align*}
{Note that if the learnable parameter $\theta_{0}$ is initialized as described in \Cref{section_sub:neural_es_design}, then under \Cref{assumption:neural_x}, we have $f(X;  \theta_{0}) = 0$ for all $X \in \mathcal{X}$. Therefore, we have}
\begin{align*}
    {
    \widetilde{L}_{\text{Neural}}(\theta; \mathcal{D}_{t})
    = \frac{\lambda}{2}N\|\theta - \theta_{0}\|_{2}^{2} + \sum_{l=1}^{t} \frac{1}{2} \Big( g(X_{l}; \theta_{0})^{\top} (\theta - \theta_{0}) - Y_{l} \Big)^{2}.
    }
\end{align*}

{To obtain the expression of the MLE, we introduce the following notations.}
We define the covariance matrix $A_{t}$ at the end round $t$ as
\begin{align}
    \label{equ:at_definition}
    A_{t} = \frac{1}{N} \sum_{l=1}^{t} g(X_{l}; \theta_{0}) g(X_{l}; \theta_{0})^{\top} + \lambda \, I_{d'} = \sum_{l=1}^{t} \psi(X_{l}) \psi(X_{l})^{\top} + \lambda I_{d'},
\end{align}
and the following auxiliary vectors
\begin{align}
    \label{equ:bar_bt_definition}
    \bar{\mathbf{b}}_{t}
    = \frac{1}{\sqrt{N}} \sum_{l=1}^{t} Y_{l} \,g(X_{l}; \theta_{0})
    = \sum_{l=1}^{t} Y_{l} \, \psi(X_{l}),
\end{align}
\begin{align}
    \label{equ:btj_definition}
    \mathbf{b}_{t}^{j}
    = \frac{1}{\sqrt{N}} \sum_{l=1}^{t} (Y_{l} + Z_{l}^{j}) \,g(X_{l}; \theta_{0})
    = \sum_{l=1}^{t} (Y_{l} + Z_{l}^{j}) \, \psi(X_{l}),
\end{align}
{where $j \in [m]$ is the index ensemble element and $\{Z_{l}^{j}\}_{l=1}^{t}$ is the perturbation sequence. Then, the MLE under linear approximation are given by}
\begin{align}
    \label{equ:mle_theta_neural}
    {
    \bar{\theta}_{t} = \theta_{0} + A_{t}^{-1} \, \bar{\mathbf{b}}_{t} / \sqrt{N},
    \quad \bar{\theta}_{t}^{j} = \theta_{0} + A_{t}^{-1} \, \mathbf{b}_{t}^{j} / \sqrt{N},
    }
\end{align}
{where $\bar{\theta}_{t}$ is the estimate without added perturbations, $\bar{\theta}_{t}^{j}$ is the estimate based on observed rewards and added perturbations for ensemble element $j \in [m]$. Note that $\bar{\theta}_{t}$ and $\bar{\theta}_{t}^{j}$ are obtained under linear approximation and are introduced as reference points for convenience of analysis. We do not make the linear approximation assumption in our algorithm design and proof.  As we will observe in the following analysis, after the gradient descent process, the learned parameter $\theta_{t-1}^{j}$ at the beginning of round $t$ for ensemble element $j \in [m]$ can be well approximated by $\bar{\theta}_{t-1}^{j} = \theta_{0} + A_{t-1}^{-1} \, \mathbf{b}_{t-1}^{j} / \sqrt{N}$.
Note that the estimates at the beginning of round $t$ is based on interaction history up to $t-1$.}
These definitions and results are standard and well-known in the neural bandit literature and will be utilized extensively in the following analysis.

\subsection{Technical Lemmas}  \label{appendix_sub:neural_es_lemmas}

In this section, we list the main technical lemmas required to obtain the high-probability regret bound of \texttt{Neural-ES}. We present the proof of the technical lemmas in Appendix \ref{appendix:neural_es_lemmas}.

First, we need the following common assumption and well-known result to demonstrate that our neural network $f(X; \theta)$ and its gradient $g(X; \theta_{0})$ can accurately approximate the true reward model $h(X)$. {For convenience, we repeat \Cref{assumption:sub_gaussian} and \Cref{assumption:neural_x} in here.}
\begin{assumption}
{
Given feature vector $X \in \RR^{d}$, the reward $Y_{t}$ is given by $Y_{t} = h(X_{t}) + \eta_{t}$. There exists $\sigma > 0$ such that $\eta_{t}$ is $\mathcal{F}_{t}^{-}$-conditionally $\sigma$-sub-Gaussian:
\begin{align*}
    \mathbb{E} [\exp(s\eta_{t}) | \mathcal{F}_{t}^{-}] \leq \exp(\sigma^{2}s^{2}/2),
    \quad \forall t \in [T]
\end{align*}
holds almost surely for all $s \in \mathbb{R}$.
}
\end{assumption}
\begin{assumption}
    We use $\mathbf{H}$ to denote the neural tangent kernel (NTK) matrix on the context set. We assume that $\mathbf{H} \succeq \lambda_{0} I$.
    Moreover, assume that for any $X \in \mathcal{X}$, $||X||_{2} \leq 1$ and $[X]_{j} = [X]_{j+d/2}$.
\end{assumption}
\begin{lemma}  \label{lem:neural_approximation}
    (Lemma 4.1 in \cite{jia2022learning})
    There exists a positive constant $\bar{C}$ such that for any $\delta \in (0, 1)$, with probability at least $1 - \delta$, when $N\geq \bar{C}K^{4}L^{6} \, \log(K^{2}L/\delta)/\lambda_{0}^{4}$, there exists a $\theta^{*} \in \mathbb{R}^{d'}$ such that for all $X \in \mathcal{X}$,
    \begin{align*}
        h(X) = \langle g(X; \theta_{0}), \theta^{*} - \theta_{0} \rangle,
        \quad \sqrt{N} ||\theta^{*} - \theta_{0}||_{2} \leq \sqrt{2 \mathbf{h}^{\top} \mathbf{H}^{-1} \mathbf{h}} \leq S,
    \end{align*}
    where $\mathbf{H}$ is the NTK matrix defined on the context set $\mathcal{X}$ and $\mathbf{h} = \big( h(X_{1}), ..., h(X_{K}) \big)$, $S$ is the upper bound of $\sqrt{2\mathbf{h}^{\top}\mathbf{H}^{-1}\mathbf{h}}$.
\end{lemma}
Compare this result with linear contextual bandit, where the mean reward is given by $X^{\top}\theta^{*}$, we can see that when the neural network is wide enough, the mean reward for context $X \in \mathcal{X}$ can be well approximated by {a linear function of $g(X; \theta_{0})$, which is a non-linear transformation of the arm context $X$.}
{For arm $X \in \mathcal{X}$, the distance between neural network approximation $f(X; \theta)$ and true mean reward $h(X)$ can be controlled by the distance $\big| f(X; \theta) - \langle g(X; \theta_{0}), \theta - \theta_{0} \rangle \big|$ and $\big| \langle g(X; \theta_{0}), \theta - \theta_{0} \rangle - h(X) \big|$.}
Following these intuitions, we need the following lemmas to provide concentration of these distances.

\begin{lemma}  \label{lemma:neural_mle_concentration}
    Define a sequence of events $\mathcal{E}_{1, t}$ as follows:
    \begin{align*}
        \mathcal{E}_{1, t} :=
        \Big\{ {\forall X \in \mathcal{X}, \big| \langle g(X; \theta_{0}), \bar{\theta}_{t-1} - \theta_{0} \rangle - h(X) \big| \leq \alpha_{t-1} \big\| g (X; \theta_{0}) / \sqrt{N} \big\|_{A_{t-1}^{-1}} } \Big\},
    \end{align*}
    where $\alpha_{t}$ is a scalar scale of deviation.
    {Define their intersection $\mathcal{E}_{1} = \bigcap_{t=1}^{T} \mathcal{E}_{1, t}$.}
    Then, for any $\lambda > 0$ and $\delta > 0$, with $\alpha_{t}$ set as
    \begin{align*}
        {
        \alpha_{t} = \sigma \sqrt{\log\big( \det(A_{t}) / (\delta^{2} \det(\lambda I_{d'})) \big)} + \sqrt{\lambda} S,
        }
    \end{align*}
    where $S$ is the upper bound of $\sqrt{2\mathbf{h}^{\top}\mathbf{H}^{-1}\mathbf{h}}$, we have {$\mathbb{P} (\mathcal{E}_{1}) \geq 1 - 2\delta$}.
\end{lemma}

Under event {$\mathcal{E}_{1, t}$}, for round $t \in [T]$, the distance $\big| \langle g(X; \theta_{0}), \bar{\theta}_{t-1} - \theta_{0} \rangle - h(X) \big|$ is bounded by the term $\alpha_{t-1} || g (X; \theta_{0}) / \sqrt{N} ||_{A_{t-1}^{-1}}$.
Recall that we assume that the random noise $\eta_{t}$ is $\sigma$-sub-Gaussian. The parameter $\alpha_{t}$ reflects the deviation caused by the noise in the observed reward.
This result is valid for any algorithms in the neural contextual bandit setting, as it does not involve any added perturbations.
The estimate $\bar{\theta}_{t-1}$ is only based on the sequence of pulled arms $\{X_{l}\}_{l=1}^{t-1}$ and observed rewards $\{Y_{l}\}_{l=1}^{t-1}$.

Next, we need the following lemma to upper bound the distance between neural network output and linear approximation $\big| f(X; \theta_{t-1}^{j}) -  \langle g(X; \theta_{0}), \bar{\theta}_{t-1} - \theta_{0} \rangle \big|$.
{The following lemma is inspired by Lemma 4.3 in \citet{jia2022learning}. Due to the fact that in ensemble sampling, the perturbation sequence is updated incrementally rather than fully re-sampled at each round, the proof of this technical lemma deviates considerably from Lemma 4.3 in \citet{jia2022learning}, and the final concentration bound has an extra $\widetilde{\mathcal{O}}\big(\widetilde{d}^{1/2}\big)$ factor. We present the detailed proof in \Cref{appendix_sub:neural_concentration_pert}.} %

\begin{lemma}
    \label{lemma:neural_concentration_pert}
    Fix ensemble element $j \in [m]$. Define a sequence of events $\mathcal{E}_{2, t}^{j}$ as follows:
    \begin{align*}
        \mathcal{E}_{2, t}^{j} = \Big\{ \forall X \in \mathcal{X}, \big| f(X; \theta_{t-1}^{j}) - \langle g(X; \theta_{0}), \bar{\theta}_{t-1} - \theta_{0} \rangle \big| \leq \epsilon(N) + \beta_{t-1} || g(X; \theta_{0}) / \sqrt{N} ||_{A_{t-1}^{-1}} \Big\},
    \end{align*}
    where $\beta_{t}$ is a scalar scale of the deviation, and $\epsilon(N)$ is defined as
    \begin{align}
        \epsilon(N) &= C_{\epsilon, 1} N^{-1/6}T^{2/3}\lambda^{-2/3}L^{3} \sqrt{\log \, N} + C_{\epsilon, 2} (1 - \eta N \lambda)^{J} \sqrt{TL / \lambda} \\
        &\qquad+ C_{\epsilon, 3} N^{-1/6}T^{5/3}\lambda^{-5/3}L^{4} \sqrt{\log \, N} \big(1 + \sqrt{T / \lambda} \big),
    \end{align}
    where $\eta$ is learning rate, $J$ is the number of steps for gradient descent in neural network learning, $\{ C_{\epsilon, i} \}_{i=1}^{3}$ are constants.
    Then, there exist positive constants $\{ C_{i} \}_{i=1}^{3}$, such that with step size $\eta$ and the neural network width satisfy the following conditions:
    \begin{align*}
        \eta &= C_{1} (N\lambda + NLT)^{-1}, \\
        N &\geq C_{2} \sqrt{\lambda} L^{-3/2} \big[\log (TKL^{2} / \delta)\big]^{3/2}, \\
        N[\log N]^{-3} &\geq C_{3} \, \max \Big\{ TL^{12}\lambda^{-1}, T^{7}\lambda^{-8}L^{18}(\lambda + LT)^{6}, L^{21}T^{7}\lambda^{-7}\big(1 + \sqrt{T/\lambda}\big)^{6} \Big\},
    \end{align*}
    and set $\beta_{t}$ as
    \begin{align*}
        {
        \beta_{t} = \sigma_{R} \sqrt{\log \big( \det(A_{t}) / \det(\lambda I_{d'})\big) + 4\log t}\,\,,
        }
    \end{align*}
    we have $\mathbb{P} \big(\mathcal{E}_{2, t}^{j} \big) \geq 1 - t^{-2}$.
\end{lemma}

Recall that the perturbations $Z_{t}^{j}$ are sampled from $\mathcal{N}(0, \sigma_{R}^{2})$. The parameter $\beta_{t}$ reflects the deviation caused by added perturbations.
After uniformly randomly choosing model $j_{t} \in [m]$ from the ensemble, we further define the following event for round $t$:
\begin{align*}
    {
    \mathcal{E}_{2, t} = \Big\{ \forall X \in \mathcal{X},\,\, \big| f(X; \theta_{t-1}^{j_{t}}) - \langle g(X; \theta_{0}), \bar{\theta}_{t-1} - \theta_{0} \rangle \big| \leq \epsilon(N) + \beta_{t-1} || g(X; \theta_{0}) / \sqrt{N} ||_{A_{t-1}^{-1}} \Big\}.
    }
\end{align*}
Then, at round $t$, we have the probability for $\mathcal{E}_{2, t}$ as:
\begin{align*}
    \mathbb{P} \big(\mathcal{E}_{2, t}\big)
    = \sum_{j=1}^{m} \mathbb{P} \big(j_{t} = j,\,\, \mathcal{E}_{2, t}^{j}\big)
    \geq 1 - \frac{1}{t^{2}}.
\end{align*}

Next, we present the optimism (anti-concentration) condition in the following two lemmas.
First, we fix the sequence of pulled arms $\{X_{l}\}_{l=1}^{t}$ and model $j \in [m]$, the only source of randomness is the perturbations $\{Z_{l}^{j}\}$, we can apply tail bound of Gaussian distribution to lower bound the probability of optimism.
Then, we add randomness of choosing arm $j_{t}$ at round $t$ and apply the Azuma-Hoeffding inequality.
Finally, we count equivalent sequences of $\{X_{l}\}_{l=1}^{t}$ and prove the probability of optimism considering all randomness of the procedure.

\begin{lemma}
    \label{lemma:neural_optimism_fixed}
    Fix the model $j \in [m]$ and $t \in [T]$. Fix the sequence of pulled arms $\{X_{l}\}_{l=1}^{t}$.
    Set the variance of perturbations $\sigma_{R}$ as
    \begin{align}  \label{equ:sigma_R}
        \sigma_{R} = \alpha_{T} \sqrt{\frac{\lambda + \rho}{\rho}},
        \quad \alpha_{T} = \sigma \sqrt{\widetilde{d} \log \Big(1 + \frac{TL^{2}}{\lambda \widetilde{d}}\Big)} + \sqrt{\lambda} S,
    \end{align}
    where $\rho$ is defined in \Cref{assumption:coverage}.
    Then, the probability of optimism is lower bounded as follows:
    \begin{align*}
        \mathbb{P}\Big(f (X^{*}; \theta_{t-1}^{j}) \geq h(X^{*}) - \epsilon(N)\Big) \geq p_{N}',
    \end{align*}
    where $p_{N}' = (4e\sqrt{\pi})^{-1}$, the probability is measured over the randomness of the perturbation sequence.
\end{lemma}

\begin{lemma}  \label{lemma:neural_optimism}
    Fix the ensemble size $m$ satisfying
    \begin{align*}
        m \geq \frac{8}{p_{N}^{\prime2}} \Big( K \text{log}T + \text{log}\frac{1}{\delta} \Big).
    \end{align*}
    Define a sequence of events $\mathcal{E}_{3, t}$ as follows:
    \begin{align*}
        \mathcal{E}_{3, t} :=
        \Big\{ \mathbb{P} \Big(f (X^{*}; \theta_{t-1}^{j_{t}}) \geq h(X^{*}) - \epsilon(N) \,\,\text{and}\,\, \mathcal{E}_{2, t} \,|\, \mathcal{F}_{t-1}\Big) \geq p_{N}' / 4 \Big\},
    \end{align*}
    and their intersections $\mathcal{E}_{3} = \bigcap_{t=K+1}^{T} \mathcal{E}_{3, t}$.
    Then, with $\sigma_{R}$ set as \eqref{equ:sigma_R}, $\mathcal{E}_{3}$ holds with probability at least $1 - \delta$.
\end{lemma}

We need the following lemma to calculate the summation of the per-round regret.
\begin{lemma} [Lemma 4.6 in \cite{jia2022learning}]  \label{lemma:neural_sum}
    With $\widetilde{d}$ as the effective dimension of the NTK matrix $\mathbf{H}$, we have
    \begin{equation}
        \sum_{t=1}^{T} \min \Big\{\ || g(X_{t}; \theta_{0}) / \sqrt{N}||_{A_{t-1}^{-1}}^{2}, 1 \Big\}
        \leq 2 \big( \widetilde{d} \, \text{log}(1 + TK/\lambda) + 1 \big).
    \end{equation}
\end{lemma}

{
To prove the high-probability concentration bounds, we need the following version of the well-known result from \cite{abbasi2011improved}.
\begin{lemma}[Theorem 1 of \cite{abbasi2011improved}]  
\label{lemma:martingale_v2}
Let $\{\mathcal{F}_{t}\}_{t=0}^{\infty}$ be a filtration. Let $\{\eta_{t}\}_{t=1}^{\infty}$ be a sequence of real-valued random variables such that $\eta_{t}$ is $\mathcal{F}_{t}$-measurable and $\mathcal{F}_{t-1}$-conditionally $\sigma$-sub-Gaussian for some $\sigma \geq 0$. Let $\{X_{t}\}_{t=1}^{\infty}$ be a sequence of $\mathbb{R}^{d}$-valued random vectors such that $X_{t}$ is $\mathcal{F}_{t-1}$-measurable and $\vert\vert X_{t} \vert\vert_{2} \leq 1$ almost surely for all $t \geq 1$. Fix $\lambda \geq 1$. Let $V_{t} = \lambda I_{d} + \sum_{l=1}^{t} X_{l}X_{l}^{\top}$. Then, for any $\delta \in (0, 1]$, with probability at least $1 - \delta$, the following inequality holds for all $t \geq 0$:
\begin{align*}
    \bigg\| \sum_{l=1}^{t} \eta_{l} X_{l} \bigg\|_{V_{t}^{-1}}^{2}
    \leq 2\sigma^{2} \log \Big(\frac{\det (V_{t})^{1/2} \det (\lambda I_{d})^{-1/2}}{\delta}\Big).
\end{align*}
\end{lemma}
}

\subsection{Regret Analysis}  \label{appendix_sub:neural_es_proof}

{In this section, we derive the high-probability upper bound of the cumulative regret of \texttt{Neural-ES} \Cref{theorem:neural-es}. The framework of the proof is inspired by \citet{zhang2021neural} and follows the common practice in the randomized exploration literature.
The core novelty in our analysis is the proof of concentration \Cref{lemma:neural_concentration_pert} and constant probability of optimism \Cref{lemma:neural_optimism} with incrementally updated perturbations, as well as integrating the new technical lemmas with the neural bandit analysis framework.}

From the design of \texttt{Neural-ES} (\Cref{section_sub:neural_es_design}), we use the first $K$ rounds to pull each arm once. Therefore, the cumulative regret can be decomposed as (assuming that $T > K$):
\begin{align*}
    R(T) \leq \sum_{t=K+1}^{T} \big(h(X^{*}) - h(X_{t})\big) + K.
\end{align*}
Next, we analyze the per-round regret bound for $t \in [K+1, T]$.
We assume that $\mathcal{E}_{1, t}$ holds for the rest of the analysis.
{Recall from \Cref{lemma:neural_mle_concentration}, with probability at least $1 - \delta$, $\mathcal{E}_{1, t}$ holds for all $t\in[T]$.}
Under event $\mathcal{E}_{1, t}$, for each round $t$, the regret can be written as
\begin{align*}
    h(X^{*}) - h(X_{t})
    \leq \big(h(X^{*}) - h(X_{t})\big) \ind\{ \mathcal{E}_{2, t} \} + \mathbb{P} \big( \bar{\mathcal{E}}_{2, t} \big),
\end{align*}
where we used the assumption that $h(X)$ in bounded in $[0, 1]$.
{From \Cref{lemma:neural_concentration_pert}, $\mathbb{P} \big(\bar{\mathcal{E}}_{2, t}\big) \leq t^{-2}$, thus we only need to focus on $\big(h(X^{*}) - h(X_{t})\big)$ under event $\mathcal{E}_{1, t} \cap \mathcal{E}_{2, t}$.}

For the following analysis, we introduce the set of sufficiently sampled arms in round $t$, {which is a standard procedure in contextual bandit and randomized exploration literature}:
\begin{align*}
    \Omega_{t} = \big\{ \forall X \in {\mathcal{X}}: 2 \epsilon(N) + (\alpha_{t} + \beta_{t}) || g(X; \theta_{0}) / \sqrt{N} ||_{A_{t-1}^{-1}} \leq h(X^{*}) - h(X) \big\}.
\end{align*}
We also define the set of under sampled arms $\bar{\Omega}_{t} = \mathcal{X} \backslash \Omega_{t}$.
From the definition, the per-round regret of $X_{t} \in \bar{\Omega}_{t}$ is upper bounded by $\big(2 \epsilon(N) + (\alpha_{t} + \beta_{t}) || g(X_{t}; \theta_{0}) / \sqrt{N} ||_{A_{t-1}^{-1}}\big)$. We expect that the pulled arm $X_{t}$ should come from $\bar{\Omega}_{t}$ with finite probability as the algorithm proceeds.
We further define the least uncertain and under sampled arm $X_{t}^{(e)}$ at round $t$ as
\begin{align*}
    X_{t}^{(e)} = \text{argmin}_{X \in \bar{\Omega}_{t}} ||g(X; \theta_{0}) / \sqrt{N}||_{A_{t-1}^{-1}}.
\end{align*}

Then, we can write the per-round regret
\begin{align*}
    &\big(h(X^{*}) - h(X_{t})\big) \ind\{\mathcal{E}_{2, t} \} \\
    =& \big(h(X^{*}) - h(X_{t}^{(e)}) + h(X_{t}^{(e)}) - h(X_{t})\big) \ind\{\mathcal{E}_{2, t} \} \\
    &\leq \bigg(2 \epsilon(N) + (\alpha_{t} + \beta_{t}) || g(X_{t}; \theta_{0}) ||_{A_{t}^{-1}}\bigg) + \bigg( h(X_{t}^{(e)}) - h(X_{t}) \bigg) \ind\{\mathcal{E}_{2, t} \}  \\
    &\leq \, 4\epsilon(N) + (\alpha_{t} + \beta_{t}) \big( 2 ||g(X_{t}^{(e)}; \theta_{0}) / \sqrt{N}||_{A_{t}^{-1}} + ||g(X_{t}; \theta_{0}) / \sqrt{N}||_{A_{t}^{-1}} \big).
\end{align*}
From the per-round regret bound, we should upper bound $||g(X_{t}^{(e)}; \theta_{0}) / \sqrt{N}||_{A_{t}^{-1}}$ using expression of $||g(X_{t}; \theta_{0}) / \sqrt{N}||_{A_{t}^{-1}}$, then we can apply \Cref{lemma:neural_sum} to compute the cumulative regret.
We have the following result
\begin{align*}
    \mathbb{E} \big[ ||g(X_{t}; \theta_{0}) / \sqrt{N}||_{A_{t}^{-1}} \,|\, \mathcal{F}_{t-1} \big]
    \geq ||g(X_{t}^{(e)}; \theta_{0}) / \sqrt{N}||_{A_{t}^{-1}} \mathbb{P}\big(X_{t} \in \bar{\Omega}_{t} \,|\, \mathcal{F}_{t-1}\big), 
\end{align*}
Therefore, we have
\begin{align*}   
    ||g(X_{t}^{(e)}; \theta_{0}) / \sqrt{N}||_{A_{t}^{-1}}
    \leq \frac{\rev{\mathbb{E} \big[ ||g(X_{t}; \theta_{0}) / \sqrt{N}||_{A_{t}^{-1}} \,|\, \mathcal{F}_{t-1}\big]}}{\mathbb{P}(X_{t} \in \bar{\Omega}_{t} \,|\, \mathcal{F}_{t-1})}.
\end{align*}
{The lower bound the probability of $\mathbb{P}(X_{t} \in \bar{\Omega}_{t})$ is given by:}
\begin{align*}
    \mathbb{P}(X_{t} \in \bar{\Omega}_{t})
    \geq& \, \mathbb{P}\big( \exists X \in \bar{\Omega}_{t}: f(X; \theta_{t}) > \text{max}_{X'\in \Omega_{t}} f(X'; \theta_{t}) \big) \\
    \geq& \, \mathbb{P} \big( f(X^{*}; \theta_{t}) > \text{max}_{X'\in \Omega_{t}} f(X'; \theta_{t}), \mathcal{E}_{2, t} \big) \\
    \geq& \, \mathbb{P} \big( f(X^{*}; \theta_{t}) > h(X^{*}) - \epsilon(N), \mathcal{E}_{2, t} \big) \\
    \geq& \, p_{N}^{\prime} / 4.
\end{align*}
where we used that under $\mathcal{E}_{3}$, the probability is bounded by $p_{N}^{\prime} / 4$.
Therefore, {combining with the fact that $0 \leq h(X) \leq 1 \rightarrow h(X^{*}) - h(X_{t}) < 1$,} the per-round regret can be bounded as
\begin{align*}
    &\mathbb{E}\big[h(X^{*}) - h(X_{t}) \,|\, \mathcal{F}_{t-1}\big] \\
    &\leq t^{-2} + 4\epsilon(N) + (\alpha_{t} + \beta_{t}) \Big(1 + \frac{2}{p_{N}^{\prime} / 4} \Big) \min \big\{||g(X_{t};\theta_{0}) / \sqrt{N}||_{A_{t}^{-1}}, 1 \big\}.
\end{align*}

Starting from the per-round expected regret, we define the following quantity:
\begin{align*}
    \Delta_{t} = \sum_{l=K+1}^{t} \Big(h(X^{*}) - h(X_{l})\Big) - \sum_{l=K+1}^{t}\Big( 4\epsilon(N) + (\alpha_{l} + \beta_{l}) \big(1 + 8/p_{N}'\big) ||g(X_{l};\theta_{0}) / \sqrt{N}||_{A_{l-1}^{-1}} + l^{-2} \Big).
\end{align*}
Then, with probability at least $1 - \delta$, $\{\Delta_{t}\}$ forms a super martingale sequence since $\mathbb{E}[\Delta_{t} - \Delta_{t-1}] \leq 0$.
Then, by Azuma-Hoeffding inequality, the following bound holds with probability at least $1 - \delta$:
\begin{align*}
    &\sum_{t=K+1}^{T} h(X^{*}) - h(X_{t}) \\
    &\leq 4T\epsilon(N) + \frac{\pi^{2}}{6}
    + (\alpha_{T} + \beta_{T}) \big(1 + 8/p_{N}'\big) \sum_{t=K+1}^{T} \text{min} \big\{||g(X_{l};\theta_{0}) / \sqrt{N}||_{A_{l-1}^{-1}}, 1\big\} \\
    &\qquad + \Big(4 + (\alpha_{T} + \beta_{T}) \big(1 + 8/p_{N}'\big) + 4\epsilon(N)\Big) \sqrt{2\,T\, \text{log} \frac{1}{\delta}},
\end{align*}
where we used
\begin{align*}
    \sum_{t=K + 1}^{T} \frac{1}{t^{2}} \leq \frac{\pi^{2}}{6}.
\end{align*}
Applying the summation bound \Cref{lemma:neural_sum}, we have
\begin{align*}
    R(T) &\leq \,K + 4T\epsilon(N) + \frac{\pi^{2}}{6} + \Big(1 + \frac{8}{p_{N}^{\prime}}\Big) (\alpha_{T} + \beta_{T}) \sqrt{T} \sqrt{ \sum_{t=1}^{T} \min \Big\{||g(X_{t};\theta_{0}) / \sqrt{N}||_{A_{t}^{-1}}^{2}, 1 \Big\} } \\
    &\qquad+ \Big(4 + (\alpha_{T} + \beta_{T}) \big(1 + 8/p_{N}'\big) + 4\epsilon(N)\Big) \sqrt{2\,T\, \text{log} \frac{1}{\delta}} \\
    &\leq \, K + 4T\epsilon(N) + \frac{\pi^{2}}{6} + \Big(1 + \frac{8}{p_{N}^{\prime}}\Big) (\alpha_{T} + \beta_{T}) \sqrt{T} \sqrt{ 2\widetilde{d}\, \text{log}(1 + TK/\lambda) + 2} \\
    &\qquad+ \Big(4 + (\alpha_{T} + \beta_{T}) \big(1 + 8/p_{N}'\big) + 4\epsilon(N)\Big) \sqrt{2\,T\, \text{log} \frac{1}{\delta}}.
\end{align*}
The regret bound holds with probability at least $1 - 4\delta$.
With parameters $\eta = C_{1} (NTL + N\lambda)^{-1}$, $N[\text{log} N]^{-3} \geq C_{3} \, \text{max} \big\{ TL^{12}\lambda^{-1}, T^{7}\lambda^{-8}L^{18}(\lambda + LT)^{6}, L^{21}T^{7}\lambda^{-7}(1 + \sqrt{T/\lambda})^{6} \big\}$ and $\lambda \geq \text{max}\{1, S^{-2}\}$, we have $T\epsilon(N) \sim \widetilde{\mathcal{O}}(1)$, $\alpha_{T} \sim \widetilde{\mathcal{O}}\big(\widetilde{d}^{1/2}\big)$, $\beta_{T} \sim \widetilde{\mathcal{O}}\big(\widetilde{d}\big)$. {Therefore, the regret bound $R(T) = \widetilde{\mathcal{O}} \big(\widetilde{d}^{3/2} \sqrt{T}\big)$ holds with probability at least $1 - 4\delta$.
This completes the proof.}

\vspace{3mm}

\section{Proof of Technical Lemmas in Neural-ES}  \label{appendix:neural_es_lemmas}

\subsection{Proof of \Cref{lemma:neural_mle_concentration}}
\label{appendix_sub:neural_concentration_mle}

In this section, we prove a high-probability upper bound for the distance between true mean reward $h(X)$ and the estimate without perturbations $\langle g(X; \theta_{0}), \bar{\theta}_{t} - \theta_{0} \rangle$, where $\bar{\theta}_{t}$ is the MLE under linear approximation at the end of round $t$, as defined in \eqref{equ:mle_theta_neural}.

According to \Cref{lem:neural_approximation}, when choosing the network width $N\geq \bar{C}K^{4}L^{6} \, \log(K^{2}L/\delta)/\lambda_{0}^{4}$, with probability at least $1 - \delta$, we have for all $X \in \mathcal{X}$,
\begin{align*}
    h(X) = \langle g(X; \theta_{0}), \theta^{*} - \theta_{0} \rangle,
    \quad \sqrt{N} ||\theta^{*} - \theta_{0}||_{2} \leq \sqrt{2 \mathbf{h}^{\top} \mathbf{H}^{-1} \mathbf{h}} \leq S.
\end{align*}
Under this condition, we can write
\begin{align*}
    \langle g(X; \theta_{0}), \bar{\theta}_{t} - \theta_{0} \rangle - h(X)
     = \langle g(X; \theta_{0}), \bar{\theta}_{t} - \theta^{*} \rangle.
\end{align*}
Applying Cauchy-Schwarz inequality, we have
\begin{align*}
    \big|\langle g(X; \theta_{0}), \bar{\theta}_{t} - \theta^{*} \rangle\big|
    \leq \big\| \sqrt{N} (\bar{\theta}_{t} - \theta^{*}) \big\|_{A_{t}} \big\| g(X; \theta_{0}) / \sqrt{N} \big\|_{A_{t}^{-1}}.
\end{align*}

{
Now we bound the distance between $\bar{\theta}_{t}$ and $\theta^{*}$. First, according to the definition \eqref{equ:mle_theta_neural}, we have
\begin{align*}
    \big\| \sqrt{N} (\bar{\theta}_{t} - \theta^{*}) \big\|_{A_{t}}
    = \big\| A_{t}^{-1} \, \bar{\mathbf{b}}_{t} - \sqrt{N} (\theta^{*} - \theta_{0}) \big\|_{A_{t}}.
\end{align*}
Since the observed reward $Y_{l}$ can be decomposed as $Y_{l} = h(X_{l}) + \eta_{l}$, we have
\begin{align*}
    \bar{\mathbf{b}}_{t}
    = \frac{1}{\sqrt{N}} \sum_{l=1}^{t} Y_{l} \,g(X_{l}; \theta_{0})
    = \frac{1}{\sqrt{N}} \sum_{l=1}^{t} \big(h(X_{l}) + \eta_{l}\big) \,g(X_{l}; \theta_{0}).
\end{align*}
With the expression of $h(X)$ and $\theta^{*}$, we can write
\begin{align*}
    \bar{\mathbf{b}}_{t} &= \frac{1}{\sqrt{N}} \sum_{l=1}^{t} g(X_{l}; \theta_{0}) g(X_{l}; \theta_{0})^{\top} (\theta^{*} - \theta_{0}) + \sum_{l=1}^{t} \eta_{l} g(X_{l};\theta_{0}) / \sqrt{N}  \\
    &= (A_{t} - \lambda I_{d'})\sqrt{N}(\theta^{*} - \theta_{0}) + \sum_{l=1}^{t} \eta_{l} g(X_{l};\theta_{0}) / \sqrt{N}.
\end{align*}
With this decomposition of $\bar{\mathbf{b}}_{t}$, we have
\begin{align*}
    \big\| \sqrt{N} (\bar{\theta}_{t} - \theta^{*}) \big\|_{A_{t}}
    =& \Big\| -\lambda A_{t}^{-1} \sqrt{N} (\theta^{*} - \theta_{0}) + A_{t}^{-1} \sum_{l=1}^{t} \eta_{l} g(X_{l};\theta_{0}) / \sqrt{N} \Big\|_{A_{t}} \\
    \leq& \Big\| \sum_{l=1}^{t} \eta_{l} g(X_{l};\theta_{0}) / \sqrt{N} \Big\|_{A_{t}^{-1}} + \lambda \big\|\sqrt{N}(\theta^{*} - \theta_{0}) \big\|_{A_{t}^{-1}}.
\end{align*}
Applying the bound of $\|\theta^{*} - \theta_{0}\|_{2}$ and the fact that $\lambda I_{d'} \preceq A_t$, we have
\begin{align*}
    \lambda \big\|\sqrt{N}(\theta^{*} - \theta_{0})\big\|_{A_{t}^{-1}} \leq \sqrt{\lambda}S
\end{align*}
Applying \Cref{lemma:martingale_v2}, we have with probability at least $1 - \delta$,
\begin{align*}
    \Big\| \sum_{l=1}^{t} \eta_{l} g(X_{l};\theta_{0}) / \sqrt{N} \Big\|_{A_{t}^{-1}}
    \leq \sigma \sqrt{\log \frac{\det(A_{t})}{\delta^{2}\det (\lambda I_{d'})}}.
\end{align*}
Therefore, define
\begin{align*}
    {
    \alpha_{t} = \sigma \sqrt{\log\big( \det(A_{t}) / (\delta^{2} \det(\lambda I_{d'})) \big)} + \sqrt{\lambda} S,
    }
\end{align*}
after choosing the network width $N\geq \bar{C}K^{4}L^{6} \, \log(K^{2}L/\delta)/\lambda_{0}^{4}$, with probability at least $1 - 2\delta$, we have for all $X \in \mathcal{X}$ and $t \in [T]$,
\begin{align*}
    \big| \langle g(X; \theta_{0}), \bar{\theta}_{t} - \theta_{0} \rangle - h(X) \big| \leq \alpha_{t} \big\| g (X; \theta_{0}) / \sqrt{N} \big\|_{A_{t}^{-1}}.
\end{align*}
This completes the proof.
}

\subsection{{Proof of \Cref{lemma:neural_concentration_pert}}}
\label{appendix_sub:neural_concentration_pert}

In this section, we prove a high-probability upper bound for the distance between neural network output and linear approximation $\big| f(X; \theta_{t}^{j}) - \langle g(X; \theta_{0}), \bar{\theta}_{t} - \theta_{0} \rangle \big|$, where $j \in [m]$ is the ensemble element, $\bar{\theta}_{t}$ is the MLE without perturbations and $\theta_{t}^{j}$ is the learned parameter through gradient descent at the end of round $t$.
The derivation of the decomposition \eqref{equ:neural_concentration_pert_decomposition} follows the proof of Lemma 4.3 in \citet{jia2022learning} and we add the calculations here for completeness, while the following analysis is unique to our work because the perturbations are now stored rather than freshly sampled in each round.
First, we decompose the distance as
\begin{align*}
    &\big| f(X; \theta_{t-1}^{j}) - \langle g(X; \theta_{0}), \bar{\theta}_{t-1} - \theta_{0} \rangle \big| \\
    &\leq \big| f(X; \theta_{t-1}^{j}) - \langle g(X; \theta_{0}), \theta_{t-1}^{j} - \theta_{0} \rangle \big| + \big|\langle g(X; \theta_{0}), \bar{\theta}_{t-1} - \theta_{t-1}^{j} \rangle \big| \\
    &\leq C_{\epsilon, 1} N^{-1/6} T^{2/3} \lambda^{-2/3} L^{3} \sqrt{\text{log}N} + \big|\langle g(X; \theta_{0}), \theta_{t-1}^{j} - \bar{\theta}_{t-1} \rangle\big|,
\end{align*}
where the second inequality holds due to Lemma B.2 in \citet{zhou2020neural}.
Applying Lemma B.3 and Lemma B.4 in \citet{zhou2020neural}, we further obtain the following upper bound for the second term:
\begin{align*}
    &\big|\langle g(X; \theta_{0}), \theta_{t-1}^{j} - \bar{\theta}_{t-1} \rangle\big| \\
    &\leq \big|\langle g(X; \theta_{0}), \theta_{t-1}^{j} - \bar{\theta}_{t-1}^{j} \rangle\big| + \big|\langle g(X; \theta_{0}), \bar{\theta}_{t-1}^{j} - \bar{\theta}_{t-1} \rangle\big| \\
    &\leq \big\| g(X; \theta_{0})\big\|_{2} \big\| \theta_{t-1}^{j} - \bar{\theta}_{t-1}^{j} \big\|_{2} + \big|\langle g(X; \theta_{0}), \bar{\theta}_{t-1}^{j} - \bar{\theta}_{t-1} \rangle\big| \\
    &\leq C_{\epsilon, 2} (1 - \eta N\lambda)^{J} \sqrt{TL/\lambda} + C_{\epsilon, 3} N^{-1/6}T^{5/3}\lambda^{-5/3} L^{4} \sqrt{\text{log}N} \big(1 + \sqrt{T/\lambda}\big) \\
    &\qquad+ |\langle g(X; \theta_{0}), \bar{\theta}_{t-1}^{j} - \bar{\theta}_{t-1} \rangle|.
\end{align*}
Following these decompositions, define
\begin{align*}
    \epsilon(N) &= C_{\epsilon, 1} N^{-1/6}T^{2/3}\lambda^{-2/3}L^{3} \sqrt{\text{log} \, N} + C_{\epsilon, 2} (1 - \eta N \lambda)^{J} \sqrt{TL / \lambda} \\
    &\qquad+ C_{\epsilon, 3} N^{-1/6}T^{5/3}\lambda^{-5/3}L^{4} \sqrt{\text{log} \, N} \big(1 + \sqrt{T / \lambda} \big),
\end{align*}
then we can write
\begin{align}
    \label{equ:neural_concentration_pert_decomposition}
    \big| f(X; \theta_{t-1}^{j}) - \langle g(X; \theta_{0}), \bar{\theta}_{t-1} - \theta_{0} \rangle \big|
    \leq \epsilon(N) + |\langle g(X; \theta_{0}), \bar{\theta}_{t-1}^{j} - \bar{\theta}_{t-1} \rangle|.
\end{align}
According to definitions of $\bar{\mathbf{b}}_{t}$ \eqref{equ:bar_bt_definition}, $\mathbf{b}_{t}^{j}$ \eqref{equ:btj_definition} and definition of MLE \eqref{equ:mle_theta_neural}, the last term can be written as
\begin{align*}
    \big|\langle g(X; \theta_{0}), \bar{\theta}_{t-1}^{j} - \bar{\theta}_{t-1} \rangle\big|
    &= \Big|\big\langle g(X ; \theta_{0}) / \sqrt{N}, A_{t-1}^{-1} \sum_{l=1}^{t-1} Z_{l}^{j} g(X_{l}; \theta_{0}) / \sqrt{N} \big\rangle\Big| \\
    &\leq ||g(X; \theta_{0}) / \sqrt{N}||_{A_{t-1}^{-1}} \Big|\Big|\sum_{l=1}^{t-1} Z_{l}^{j}g(X_{l}; \theta_{0})/\sqrt{N}\Big|\Big|_{A_{t-1}^{-1}},
\end{align*}
where we applied the Cauchy–Schwarz inequality.
Applying \Cref{lemma:martingale_v2}, with probability at least $1 - \delta$, the following upper bound holds:
\begin{align*}
    \Big\|\sum_{l=1}^{t} Z_{l}^{j}g(X_{l}; \theta_{0})/\sqrt{N}\Big\|_{A_{t}^{-1}}
    \leq \sqrt{2} \sigma_{R} \sqrt{\log \frac{\det(A_{t})^{1/2}}{\det(\lambda I_{d'})^{1/2} \delta}}
    \leq \sigma_{R} \sqrt{\log \frac{\det(A_{t})}{\det(\lambda I_{d'})} + 2\log \frac{1}{\delta}}.
\end{align*}
Therefore, we have the following upper bound:
\begin{align*}
    \big| f(X; \theta_{t}^{j}) - \langle g(X; \theta_{0}), \bar{\theta}_{t} - \theta_{0} \rangle \big|
    \leq \epsilon(N) + \sigma_{R} \sqrt{\log \frac{\det(A_{t})}{\det(\lambda I_{d'})} + 2\log \frac{1}{\delta}} \,\, \big\| g(X; \theta_{0}) / \sqrt{N} \big\|_{A_{t}^{-1}}.
\end{align*}
Setting $\delta = t^{-2}$, we obtain the following upper bound that holds with probability at least $1 - t^{-2}$:
\begin{align*}
    \big| f(X; \theta_{t}^{j}) - \langle g(X; \theta_{0}), \bar{\theta}_{t} - \theta_{0} \rangle \big|
    \leq \epsilon(N) + \sigma_{R} \sqrt{\log \frac{\det(A_{t})}{\det(\lambda I_{d'})} + 4\log t} \,\, \big\| g(X; \theta_{0}) / \sqrt{N} \big\|_{A_{t}^{-1}}.
\end{align*}
Define
\begin{align*}
    \beta_{t} = \sigma_{R} \sqrt{\log \frac{\det(A_{t})}{\det(\lambda I_{d'})} + 4\log t}\,\,,
\end{align*}
then, with probability at least $1 - t^{-2}$, we have the upper bound
\begin{align*}
    \big| f(X; \theta_{t}^{j}) - \langle g(X; \theta_{0}), \bar{\theta}_{t} - \theta_{0} \rangle \big|
    \leq \epsilon(N) + \beta_{t} \big\| g(X; \theta_{0}) / \sqrt{N} \big\|_{A_{t}^{-1}}.
\end{align*}
This completes the proof.

\subsection{Proof of \Cref{lemma:neural_optimism_fixed}}

In this section, we prove that the optimism condition is satisfied at constant probability. The definition and variance calculation of $U_{t}$ is adapted from \cite{jia2022learning}. Note that in our analysis, we need to explicitly assume that the sequence of pulled arms $\{X_{l}\}_{l=1}^{t}$ is fixed such that the only source of randomness is the perturbation sequence.

Under event $\mathcal{E}_{1, t}$, we have
\begin{align*}
    {
    h(X) \leq \langle g(X; \theta_{0}), \bar{\theta}_{t} - \theta_{0} \rangle + \alpha_{t} || g (X; \theta_{0}) / \sqrt{N} ||_{A_{t}^{-1}}.
    }
\end{align*}
Following the same procedure of obtaining \eqref{equ:neural_concentration_pert_decomposition} in \Cref{appendix_sub:neural_concentration_pert}, we have
\begin{align*}
    \big| f(X; \theta_{t}^{j}) - \langle g(X; \theta_{0}), \bar{\theta}_{t}^{j} - \theta_{0} \rangle \big| \leq \epsilon(N).
\end{align*}
Therefore, we have
\begin{align*}
    f(X; \theta_{t}^{j}) \geq \langle g(X; \theta_{0}), \bar{\theta}_{t} - \theta_{0} \rangle - \epsilon(N) + \langle g(X; \theta_{0}), \bar{\theta}_{t}^{j} - \bar{\theta}_{t} \rangle.
\end{align*}
Then, we have the following sufficient condition for optimism $f(X^{*}; \theta_{t}^{j}) > h(X^{*}) - \epsilon(N)$:
\begin{equation}
    \langle g(X^{*}; \theta_{0}), \bar{\theta}_{t}^{j} - \bar{\theta}_{t} \rangle
    \geq \alpha_{t} \big\| g (X^{*}; \theta_{0}) / \sqrt{N} \big\|_{A^{-1}}.
\end{equation}
From the definitions \eqref{equ:mle_theta_neural}, we have
\begin{equation}
    \langle g(X^{*}; \theta_{0}),\bar{\theta}_{t}^{j} - \bar{\theta}_{t} \rangle
    = \frac{1}{N} \sum_{l=1}^{t} Z_{l}^{j} g(X^{*}; \theta_{0})^{\top} A_{t}^{-1} g(X_{l}; \theta_{0}) =: U_{t}.
\end{equation}
Fix the sequence of pulled arms $\{ X_{l} \}_{l=1}^{t}$, then the only randomness comes from the perturbation sequence $\{Z_{l}^{j}\}_{l=1}^{t}$, other factors are constant coefficients. Therefore, $U_{t}$ follows a Gaussian distribution.
Since the perturbations are i.i.d., we can use the following tail bound: for a Gaussian random variable $X \sim \mathcal{N}(\mu, \sigma^{2})$ and $\beta > 0$, we have
\begin{align*}
    \mathbb{P} \Big( \frac{X - \mu}{\sigma} > \beta \Big) \geq \frac{\text{exp}(-\beta^{2})}{4\sqrt{\pi}\beta}.
\end{align*}
The mean of $U_{t}$ is zero. Recall that perturbations are sampled from $\mathcal{N}(0, \sigma_{R}^{2})$, the variance of $U_{t}$ is given by
\begin{align*}
    \text{Var}\big[U_{t}\big]
    =& \,\sigma_{R}^{2} \sum_{l=1}^{t} \Big( \frac{1}{N} g(X^{*}; \theta_{0})^{\top} A_{t}^{-1} g(X_{l}; \theta_{0}) \Big)^{2} \\
    =& \, \sigma_{R}^{2} \frac{1}{N^{2}} g(X^{*}; \theta_{0})^{\top} A_{t}^{-1} \Big(\sum_{l=1}^{t} g(X_{l}; \theta_{0}) g(X_{l}; \theta_{0})^{\top} \Big) A_{t}^{-1} g(X^{*}; \theta_{0})  \\
    =& \sigma_{R}^{2} \frac{g(X^{*}; \theta_{0})^{\top}}{\sqrt{N}} A_{t}^{-1} \big(A_{t} - \lambda I_{d'}\big) A_{t}^{-1} \frac{g(X^{*}; \theta_{0})}{\sqrt{N}}   \\
    =& \,\sigma_{R}^{2} || g(X^{*}; \theta_{0}) / \sqrt{N} ||_{A_{t}^{-1}}^{2} - \lambda \sigma_{R}^{2} || g(X^{*}; \theta_{0}) / \sqrt{N} ||_{A_{t}^{-2}}^{2}.
\end{align*}
Under \Cref{assumption:coverage}, during the initial warm-up, each arm is pulled once, we have 
\begin{align*}
    \sum_{l=1}^{K} g(X_{l}; \theta_{0})g(X_{l}; \theta_{0})^{\top} / N \succeq \rho I_{\mathcal{S}}.
\end{align*}
Therefore, we have $A_{t} |_{\mathcal{S}} \succeq (\lambda + \rho) I_{\mathcal{S}}$. According to the definition $\mathcal{S} = \text{span} \{g(X; \theta_{0})/\sqrt{N}: X \in \mathcal{X}\} \subseteq \mathbb{R}^{d'}$, we have $g(X; \theta_{0})/\sqrt{N} \in \mathcal{S}$. Therefore, we have
\begin{align*}
     || g(X^{*}; \theta_{0}) / \sqrt{N} ||_{A_{t}^{-2}}^{2} \leq \frac{1}{\lambda + \rho} || g(X^{*}; \theta_{0}) / \sqrt{N} ||_{A_{t}^{-1}}^{2}.
\end{align*}
Then, we have
\begin{align*}
    \text{Var}\big[U_{t}\big] \geq \sigma_{R}^{2} \frac{\rho}{\lambda + \rho} || g(X^{*}; \theta_{0}) / \sqrt{N} ||_{A_{t}^{-1}}^{2}.
\end{align*}

Therefore, by choosing variance of perturbation
\begin{align*}
       \sigma_{R} = \alpha_{t} \sqrt{\frac{\lambda + \rho}{\rho}},
\end{align*}
we have the probability of optimism:
\begin{equation}
    \mathbb{P} \bigg( f(X^{*}; \theta_{t}^{j}) > h(X^{*}) - \epsilon(N) \bigg)
    \geq \mathbb{P} \bigg( U_{t} > \alpha_{t} || g (X_{t}; \theta_{0}) / \sqrt{N} ||_{A_{t}^{-1}} \bigg) \geq \frac{1}{4e\sqrt{\pi}}.
\end{equation}
This completes the proof.

\subsection{Proof of \Cref{lemma:neural_optimism}}

Define random variable $I_{t}^{j}$ for round $t$ and model $j$ as $I_{t}^{j} = \ind \{\mathcal{E}_{3, t}^{j} \bigcap \mathcal{E}_{2, t}\}$. Then, from previous technical lemmas, we have $\mathbb{P} \big(I_{t}^{j} = 1) \geq p_{N}^{\prime} - \mathbb{P}(\bar{\mathcal{E}}_{1, t}) \geq p_{N}' / 2$, where we assumed $\delta / T \leq p_{N}' / 2$.
For each round $t$, we further define $I_{t} = I_{t}^{j_{t}}$. After uniformly choosing the model $j \in [m]$, we have
\begin{equation}
    \mathbb{P} \big( I_{t} = 1 \,|\, \mathcal{F}_{t-1} \big) = \frac{1}{m} \sum_{j=1}^{m} I_{t}^{j}.
\end{equation}
Apply Azuma-Hoeffding inequality, we have the following probability of optimism:
\begin{equation}
    \mathbb{P} \Big( \frac{1}{m} \sum_{j=1}^{m} I_{t}^{j} < \frac{p_{N}^{\prime}}{4} \Big) \leq \text{exp} \Big( -\frac{p_{N}^{\prime2}m}{8} \Big).
\end{equation}
Following the discussions in \cite{lee2024improved}, we have the following lemma, which is adapted from Claim 1 in the paper.
\begin{lemma}
    There exists an event $\mathcal{E}^{*}$ such that under $\mathcal{E}^{*}$, $\frac{1}{m}\sum_{j=1}^{m} I_{t}^{j} \geq p_{N}^{\prime} / 4$ holds for all $t \in [T]$, and $\mathbb{P}\big( \bar{\mathcal{E}}^{*} \big) \leq T^{K} \text{exp} \big(-p_{N}^{\prime2}m / 8\big)$.
\end{lemma}
Therefore, setting ensemble size as follows:
\begin{align*}
    m \geq \frac{8}{p_{N}^{\prime2}} \Big( K \text{log}T + \text{log}\frac{1}{\delta} \Big),
\end{align*}
then we have $\mathbb{P}(\mathcal{E}^{*}) \geq 1 - \delta$.
In summary, by setting the ensemble size $m = \Omega(K \text{log}T)$, we have
\begin{equation}
    \mathbb{P}(\mathcal{E}_{3, t}) = 
    \mathbb{P} \Big(\big(f (X^{*}; \theta_{t-1}^{j_{t}}) \geq h(X^{*}) - \epsilon(N)\big) \bigcap \mathcal{E}_{2, t} \, | \, \mathcal{F}_{t-1} \Big) \geq p_{N}^{\prime} / 2.
\end{equation}
This completes the proof.

\section{Analysis of Doubling Trick}  \label{appendix:doubling_trick}

In this section, we provide theoretical analysis of regret bound with doubling trick.
To preserve the asymptotic regret bound of ensemble sampling, we need to properly set the schedule $\{T_{i}\}$.
We need the following result of the regret bound of doubling trick.
\begin{lemma}[Theorem 4 in \cite{besson2018doubling}]
\label{lemma:anytime_bound}
If an algorithm $\mathcal{A}$ satisfies $R_{T}(\mathcal{A}_{T}) \leq cT^{\gamma}(\text{log}T)^{\delta} + f(T)$, for $0 < \gamma < 1$, $\delta \geq 0$ and for $c > 0$, and an increasing function $f(t) = o(t^{\gamma}(\text{log}t)^{\delta})$ (at $t \rightarrow \infty$), then the anytime version $\mathcal{A'} := \mathcal{DT}(\mathcal{A}, (T_{i})_{i \in \mathbb{N}})$ with the geometric sequence $(T_{i})_{i \in \mathbb{N}}$ of parameters $T_{0} \in \mathbb{N}^{*}$, $b > 1$ (i.e., $T_{i} = \lfloor T_{0} b^{i} \rfloor$) with the condition $T_{0}(b-1) > 1$ if $\delta > 0$, satisfies,
 \begin{align*}
    R_{T}(\mathcal{A}') \leq l(\gamma, \delta, T_{0}, b) cT^{\gamma}(\text{log}T)^{\delta} + g(T),
\end{align*}
with a increasing function $g(t) = o(t^{\gamma}(\text{log}t)^{\delta})$, and a constant loss $l(\gamma, \delta, T_{0}, b) > 1$,
\begin{align*}
    l(\gamma, \delta, T_{0}, b) := \bigg(\frac{\text{log}(T_{0}(b-1) + 1)}{\text{log}(T_{0}(b-1))}\bigg)^{\delta} \times \frac{b^{\gamma}(b-1)^{\gamma}}{b^{\gamma} - 1}.
\end{align*}
\end{lemma}

With \Cref{lemma:anytime_bound}, the regret bound analysis is straightforward.
From the regret bound of ensemble sampling, the dependency on $T$ is given by $\mathcal{O} \Big((\text{log}T)^{\frac{3}{2}}\sqrt{T}\Big)$.
Therefore, apply \Cref{lemma:anytime_bound} to ensemble sampling algorithms (\texttt{GLM-ES} or \texttt{Neural-ES}), we have $\gamma = 1/2$, $\delta = 3/2$.
We can minimize the expression of $l(\gamma, \delta, T_{0}, b)$ by properly selecting the parameter of the doubling sequence $T_{0}$ and $b$.
The optimal choice of $b$ is $(3+\sqrt{5})/2 \approx 2.6$ and we can choose $T_{0}$ large enough to reduce the other factor.
For example, if we choose $T_{0} = 100$, the extra factor $l(\gamma, \delta, T_{0}, b) \approx 3.3$.

From the discussions above, we can choose the sequence
\begin{align*}
    \{T_{i}\} = \{T_{0},\, T_{0}b,\, T_{0}b^{2},...\}.
\end{align*}
The number of rounds follows the sequence:
\begin{align*}
    \{\tau_{i}\} = \{T_{0},\, T_{0}(b-1),\, T_{0}(b-1)b,\, T_{0}(b-1)b^{2}, ...\}.
\end{align*}
From the discussions above, by directly applying doubling trick, we obtain the same asymptotic regret bound with the cost of a constant factor.

\section{Proof of Anytime GLM-ES without Fully Reset}  \label{appendix:doubling_trick_new}

\subsection{Technical Lemmas}

In anytime \texttt{GLM-ES}, we choose a sequence of time steps $\{T_{i}\} = \{T_{0},\, T_{1},\, T_{2}, \, ...\}$ as final time step for epoch $i$, then adjust the algorithm parameters when we reach $t = T_{i} + 1$.
We use $\mathcal{W} \subset [T]$ to denote the set of rounds for warm-up. Recall from \Cref{alg:es_anytime}, for epoch $i$, we use the first $\tau_{i} - \tau_{i-1}$ rounds for warm-up.
In this algorithm design, due to the fact that the perturbations are re-sampled after entering the next epoch, the following technical lemmas still hold.

\begin{lemma}[Concentration of MLE]
Fix $\delta \in (0, 1)$. Define the following parameters:
\begin{align*}
    \gamma_{t}(\delta, \lambda) &:= \sqrt{\lambda} \bigg( \frac{1}{2M} + S\bigg) + \frac{\rev{4M}}{\sqrt{\lambda}} \bigg(d + \frac{d}{2} \, \text{log} \bigg( 1 + \frac{t \dot{\mu}_{\text{max}}}{d \lambda} \bigg) + \text{log}\frac{1}{\delta} \bigg), \\
    \beta_{t}(\delta, \lambda) &:= M \frac{\gamma_{t}(\delta, \lambda)^2}{\lambda} +  \sqrt{\frac{\gamma_t(\delta, \lambda)^2}{\lambda}}.
\end{align*}
For $t \in {[T]}$, define a sequence of events
\begin{align*}
    \bar{\mathcal{E}}_{t} := \big\{ \|\bar{\theta}_{{t-1}}-\theta^*\|_{H_{{t-1}}} \leq (1+M \beta_T) \,\gamma_T \big\},
\end{align*}
and their intersection $\bar{\mathcal{E}} = \bigcap_{t=1}^{T} \bar{\mathcal{E}}_{t}$. Then, we have $\mathbb{P}(\bar{\mathcal{E}}) \geq 1 - \delta$.
\end{lemma}

\begin{lemma}[Concentration of Perturbations - Anytime]
Fix $\delta \in (0, 1)$ and $\lambda \geq 1$. Define the following parameter:
\begin{align*}
    \widetilde{\gamma}_{t}^{(i)} (\delta, \lambda) = \sqrt{d \log \bigg(1+\frac{t \dot{\mu}_{\text{min}}}{d \lambda}\bigg)+2 \log \frac{2T_{i}}{\delta}} + \sqrt{d} + \sqrt{2 \log \frac{2T_{i}}{\delta}}.
\end{align*}
For each round $t \in [T_{i-1}, T_{i}]$, define the following event
\begin{align*}
    \widetilde{\mathcal{E}}_{1, t} := \big\{ \big\|\rev{\theta_{t-1}^{j_{t}}- \bar{\theta}_{t-1}\big\|_{H_{t-1}}} \leq \dot{\mu}_{\text{max}}^{1/2} \, \dot{\mu}_{\text{min}}^{-1} \, \sigma_{R} \, \widetilde{\gamma}_{T}^{(i)} \big\}.
\end{align*}
Then, $\widetilde{\mathcal{E}}_{1, t}$ holds with probability at least $1 - \delta/T_{i}$.
\end{lemma}

\begin{lemma}[Optimism - Anytime]
\label{lem:glm_optimism_anytime_k}
Consider round $t$ in epoch $i$.
Let constant $p_{N} = 1 - \Phi(1) \approx 0.16$, and define a sequence of events as
\begin{align*}
    \mathcal{E}_{2, t} = \Big\{ \mathbb{P} \Big(X^{*\top}\theta^{*} \leq X_{t}^{\top} \rev{\theta_{t-1}^{j_{t}}} \text{ and } \widetilde{\mathcal{E}}_{1, t}  \,\vert\, \mathcal{F}_{t-1}\Big) \geq p_{N}/4 \Big\},
\end{align*}
and their intersection within the epoch $\mathcal{E}_{2}^{(i)} = \bigcap_{t = T_{i-1}}^{T_{i}} \mathcal{E}_{2, t}$.
Then, with ensemble size $m_{i} \geq \frac{8}{p_{N}^{2}} \Big( K \text{log}T_{i} + \text{log}\frac{1}{\delta} \Big)$, regularization $\lambda = 1 \vee d \vee \log(1/\delta)$, perturbation variance $\sigma_{R}^{(i)} = \Theta\big((d\log T_{i})^{1/2}\big)$ and warm-up rounds $\tau_{i} = \Theta\big(d^{4} (\log T_{i})^{3}\big)$, $\mathcal{E}_{2}^{(i)}$ holds with probability at least $1 - \delta$.
\end{lemma}

\begin{lemma}[Optimism without $K$-dependency - Anytime]
\label{lem:glm_optimism_anytime_d}
Consider round $t$ in epoch $i$.
Let constant $p_{N} = 1 - \Phi(1) \approx 0.16$, and define a sequence of events as
\begin{align*}
    \mathcal{E}_{2, t} = \Big\{ \mathbb{P} \Big(X^{*\top}\theta^{*} \leq X_{t}^{\top} \rev{\theta_{t-1}^{j_{t}}} \text{ and } \widetilde{\mathcal{E}}_{1, t}  \,\vert\, \mathcal{F}_{t-1}\Big) \geq p_{N} / 4 \Big\},
\end{align*}
and their intersection within the epoch $\mathcal{E}_{2}^{(i)} = \bigcap_{t = T_{i-1}}^{T_{i}} \mathcal{E}_{2, t}$.
Then, with ensemble size $m_{i} \geq 2000 d \log (T_{i} / \delta)$, regularization $\lambda = 1 \vee \dot{\mu}_{\min} \vee d \vee \log(1/\delta)$, perturbation variance $\sigma_{R}^{(i)} = \Theta\big((d\log T_{i})^{1/2}\big)$ and warm-up rounds $\tau_{i} = \Theta\big(d^{4} (\log T_{i})^{3}\big)$, $\mathcal{E}_{2}^{(i)}$ holds with probability at least $1 - \delta$.
\end{lemma}

With these technical lemmas, we present the proof of the regret bound of anytime \texttt{GLM-ES} in the next section. The proof of the technical lemmas are presented in following sections.

\subsection{Regret Analysis}

In the GLM setting, the cumulative regret
\begin{align*}
    R(T) = \sum_{t=1}^{T}
    \Big(\mu \big(X^{*\top}\theta^{*} \big) - \mu \big(X_{t}^{\top}\theta^{*} \big)\Big).
\end{align*}
In the warm-up procedure, the regret for each round is bounded as $\Delta_{\max}$. Assume that the epoch number $i \leq I$. The minimum accumulated warm-up round until epoch $I$ is given by $\tau_{I} = \Theta\big(d^{4} (\log T_{I})^{3}\big)$. We use $\mathcal{W} \subset [T]$ to denote the set where $t \in \mathcal{W}$ are warm-up rounds. Then, we have
\begin{align*}
    R(T) \leq \tau_{I} \Delta_{\max} +  \sum_{t \notin \mathcal{W}}
    \Big( \mu(X^{*\top}\theta^{*} \big) - \mu \big(X_{t}^{\top}\theta^{*} \big)\Big).
\end{align*}

Define event of concentration and optimism $\mathcal{E}$ as $\mathcal{E} = \bigcap_{t\notin \mathcal{W}} \big( \mathcal{E}_{1, t} \cap \mathcal{E}_{2, t} \big)$.
Following the proof of non-anytime \texttt{GLM-ES} \Cref{appendix_sub:glm_es_proof}, the instantaneous regret at time $t \notin \mathcal{W}$ is bounded as:
\begin{align*}
    &\Big(\mu(X^{*\top}\theta^{*}) - \mu(X_{t}^{\top}\theta^{*})\Big)
    \ind\{\mathcal{E}\} \\
    &\leq \gamma_{I}\dot{\mu}_{\text{max}} \vert\vert X_{t} \vert\vert_{H_{t-1}^{-1}}
    + \frac{8\gamma_{I}\dot{\mu}_{\text{max}}}{p_{N}} \mathbb{E}\Big[\vert\vert X_{t} \vert\vert_{H_{t-1}^{-1}}  \vert  \mathcal{F}_{t-1}\Big] \\
    &= \gamma_{I}\dot{\mu}_{\text{max}} \bigg(1 + \frac{8}{p_{N}}\bigg) \vert\vert X_{t} \vert\vert_{H_{t-1}^{-1}}
    + \frac{8\gamma_{I}\dot{\mu}_{\text{max}}}{p_{N}}
    \Big(\mathbb{E}\Big[\vert\vert X_{t} \vert\vert_{H_{t-1}^{-1}}  \vert  \mathcal{F}_{t-1}\Big] - \vert\vert X_{t} \vert\vert_{H_{t-1}^{-1}}\Big).
\end{align*}
Then, following the summation as in the proof of \texttt{GLM-ES}, we have
\begin{align*}
    R(T) \leq \gamma_{I} \frac{\dot{\mu}_{\max}}{\sqrt{\dot{\mu}_{\min}}} \bigg(1 + \frac{8}{p_{N}}\bigg)
    \sqrt{2d(T-\tau_{I}) \,\, \text{log}\bigg(1 + \frac{T}{d \lambda}\bigg)}
    + \frac{4\gamma_{I}\dot{\mu}_{\text{max}}}{p_{N}}
    \sqrt{2(T-\tau_{I}) \text{log}\frac{1}{\delta}}
    + \tau_{I} \Delta_{max}
\end{align*}
holds with probability at least $1 - 4\delta \log T$.

Recall the definitions
\begin{align*}
    \gamma_{I} &:= (1+M \beta_{T_{I}})\gamma_{T_{I}} + \dot{\mu}_{\text{max}}^{1/2} \, \dot{\mu}_{\text{min}}^{-1} \, \sigma_{R}^{(I)} \, \widetilde{\gamma}_{T_{I}}, \\
    \gamma_{t}(\delta, \lambda) &:= \sqrt{\lambda} \bigg( \frac{1}{2M} + S\bigg) + \frac{\rev{4M}}{\sqrt{\lambda}} \bigg(d + \frac{d}{2} \, \text{log} \bigg( 1 + \frac{t \dot{\mu}_{\text{max}}}{d \lambda} \bigg) + \text{log}\frac{1}{\delta} \bigg), \\
    \beta_{t}(\delta, \lambda) &:= M \frac{\gamma_{t}(\delta, \lambda)^2}{\lambda} +  \sqrt{\frac{\gamma_t(\delta, \lambda)^2}{\lambda}}, \\
    \widetilde{\gamma}_{t}^{(i)} (\delta, \lambda) &:= \sqrt{d \log \bigg(1+\frac{t \dot{\mu}_{\text{min}}}{d \lambda}\bigg)+2 \log \frac{2T_{i}}{\delta}} + \sqrt{d} + \sqrt{2 \log \frac{2T_{i}}{\delta}},
\end{align*}
and the anytime schedule $T_{i} = \lfloor T_{0} b^{i} \rfloor$. From the fact that the horizon $T$ is within the $I$-th epoch, we have
\begin{align*}
    b_{0}^{I-1} T_{0} \leq T \leq b_{0}^{I} T_{0} = T_{I}.
\end{align*}
Since the parameters $\gamma_{I}$, $\tau_{I}$ for anytime \texttt{GLM-ES} are determined by $T_{I}$, compared to the regret bound of non-anytime \texttt{GLM-ES} \eqref{equ:glm_bound_exact}, we have the extra constant factor
\begin{align*}
    \Big( \frac{\log T_{I}}{\log T} \Big)^{3} \leq \Big( \frac{\log (b T_{0})}{\log T_{0}} \Big)^{3}.
\end{align*}
By choosing anytime parameters $T_{0} \geq 100$ and $b = (3 + \sqrt{5})/2 \approx 2.6$, the constant factor $\big(\log(b T_{0}) / \log T_{0}\big)^{3} \leq 1.78$.
As expected, the constant factor for anytime \texttt{GLM-ES} is greater than one but considerably lower than the doubling trick result.
The flexibility of anytime extension comes with a small constant factor in the regret bound.

\subsection{Proof of \Cref{lem:glm_optimism_anytime_k}}

We prove that the optimism lemma holds for the non-reset anytime \texttt{GLM-ES}.
Fix an arbitrary arm sequence $\{X_{l}\}_{l < t}$.
After resampling,
\begin{align*}
    U_{t-1}^{\top} Z_{t-1, i}^{j} \sim \mathcal{N} \big(0, \sigma_{R}^{(i)2} \| U_{t-1} \|^{2} \big).
\end{align*}
Therefore,
\begin{align*}
    P \big( U_{t-1}^{\top} Z_{t-1, i}^{j} \geq \sigma_{R, i} \| U_{t-1} \|_{2} \big) = p_{N}.
\end{align*}
The condition the solution of $\sigma_{R}^{(i)}$ exists is
\begin{align*}
    \sigma_{R}^{(i)} \geq b_{1, i} + b_{2, i} (R_{i}^{j} + \bar{R}_{i}) e \sigma_{R}^{(i)},
\end{align*}
which is exactly the same sufficient condition used in the finite-horizon proof. Hence, for every fixed arm sequence,
\begin{align*}
    P \big( X^{* \top} \theta^{*} \leq X_{t}^{\top} \theta_{t-1}^{j} \big) \geq p_{N}.
\end{align*}
From these results, resampling all historical perturbations is essential.

The counting argument can be repeated epoch-wise:
\begin{align*}
    P \Big( \frac{1}{m_{i}} \sum_{j=1}^{m} I_{t}^{j} < \frac{p_{N}}{4} \Big) \leq \exp \Big( -\frac{p_{N}^{2} m_{i}}{8} \Big).
\end{align*}
Union-bounding over the at most $\tau_{i}^{K}$ equivalent arm sequences yields
\begin{align*}
    \tau_{i}^{K} \exp \Big( -\frac{p_{N}^{2} m_{i}}{8} \Big) \leq \delta,
\end{align*}
given that we choose
\begin{align*}
    m_{i} \geq \frac{8}{p_{N}^{2}} \Big( K \log \tau_{i} + \log \frac{1}{\delta} \Big).
\end{align*}
Because this event is uniform over arm sequences, the historical arms need not have been generated using the current epoch’s perturbations.

\subsection{Proof of \Cref{lem:glm_optimism_anytime_d}}

For epoch $i$, define
\begin{align*}
    S_{t, i}^{j} = \sqrt{\lambda / \dot{\mu}_{\min}} \zeta_{i}^{j} + \sum_{l=1}^{t} X_{l}\eta_{l, i}^{j}.
\end{align*}
For every direction $u$, we have
\begin{align*}
    \langle u, S_{t, i}^{j} \rangle = \sqrt{\lambda / \dot{\mu}_{\min}} \langle u, \zeta_{i}^{j} \rangle + \sum_{l=1}^{t} \langle u, X_{l} \rangle \eta_{l, i}^{j},
\end{align*}
and its common quadratic-variation clock is
\begin{align*}
    \frac{\lambda}{\dot{\mu}_{\min}} \|u\|_{2}^{2} + \sum_{l=1}^{t} \langle u, X_{l} \rangle^{2} = \|u\|_{G_{t}}^{2}.
\end{align*}
Thus the same scalar-common-clock Gaussian martingale structure used in the proof of \Cref{lemma:optimism_glm_d} is present within every epoch. Therefore, for
\begin{align*}
    m_{i} \geq 2000 dl_{i},
    \quad l_{i} = \max \{ 1, \log \big(T_{i} / \delta\big) \},
\end{align*}
optimism holds with probability at least $1 - \delta$.

\vspace{3mm}

\section{Additional Experiments}  \label{appendix:experiments}

\begin{figure}[t]
\centering
    \subfigure[Logistic bandit environment with different ensemble size $m$.]{\includegraphics[width=.48\textwidth]{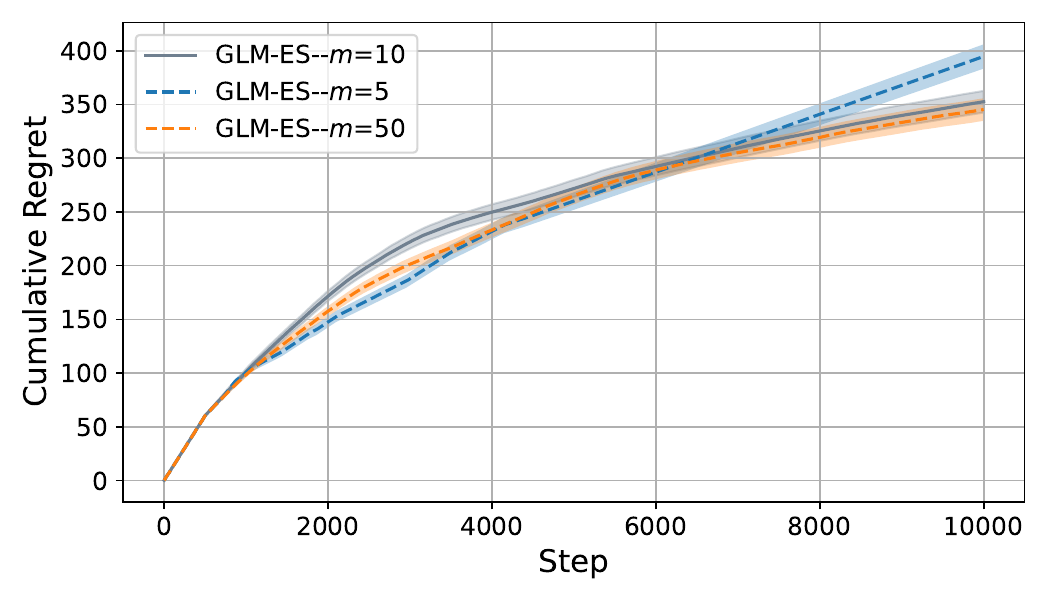}}
    \subfigure[Logistic bandit environment with different regularization $\lambda$.]{\includegraphics[width=.48\textwidth]{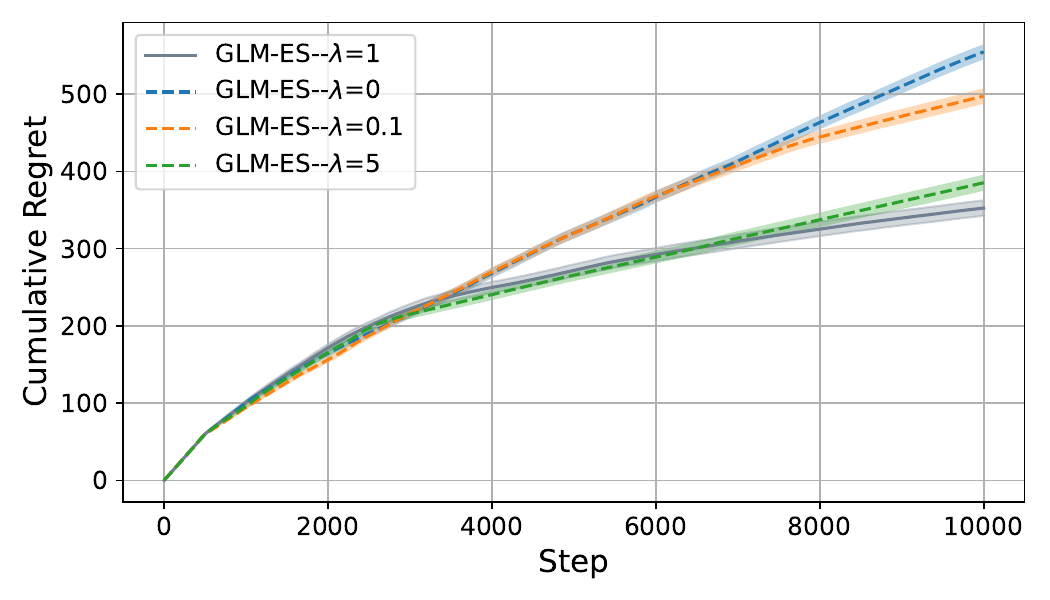}}
    \subfigure[Logistic bandit environment with different perturbation variance $\sigma_{R}$.]{\includegraphics[width=.48\textwidth]{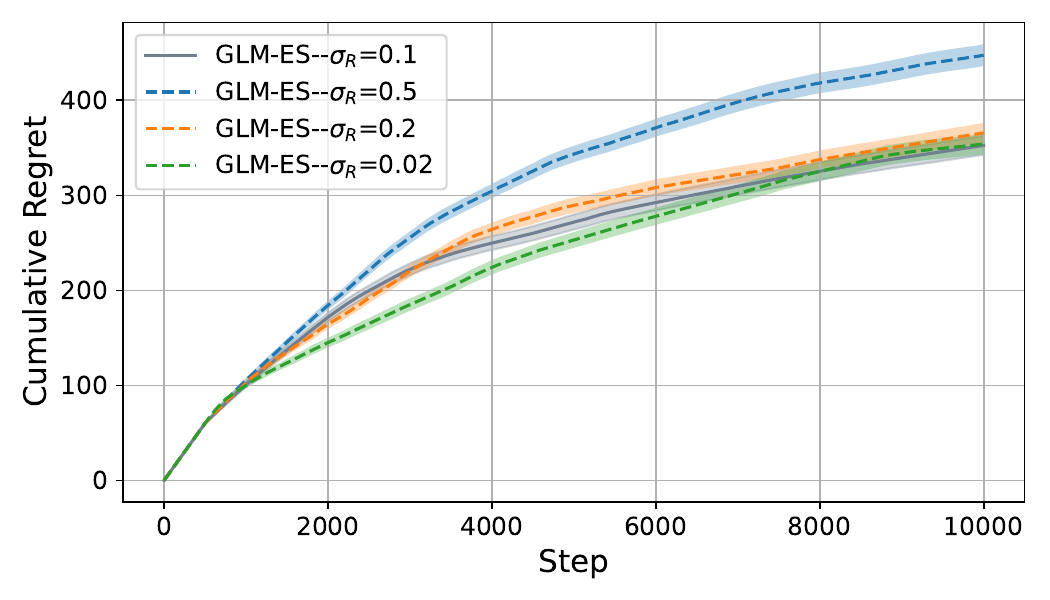}}
    \subfigure[Quadratic bandit environment with different network structure.]{\includegraphics[width=.48\textwidth]{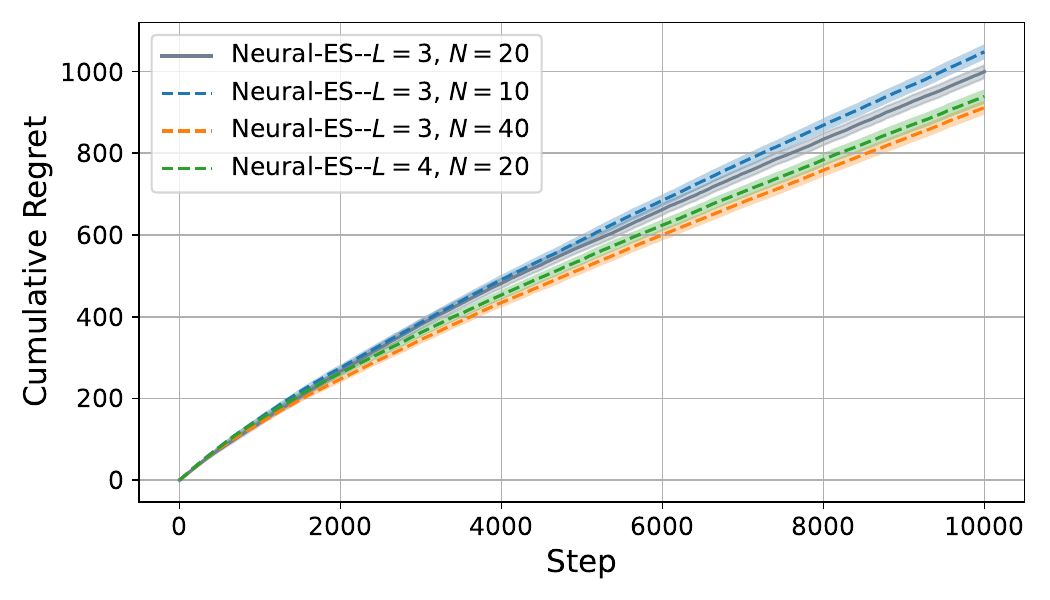}}
    \caption{Experiment results with different parameter setup and model structures.}
    \label{fig:diff_params}
\end{figure}

In this section, we present additional experiment results to explain how the performance of ensemble sampling is affected by hyper-parameters of the algorithm and the environment.
We also add experiment results for the UCI datasets to demonstrate the practicality of ensemble sampling in real-world environments.

\subsection{Experiments on Simulated Environments}

\subsubsection{Performance Test with Different Model Hyper-parameters}  \label{appendix:hyper_parameters}

First, we consider the performance of \texttt{GLM-ES} and \texttt{Neural-ES} with different ensemble size, regularization, perturbation distribution and network structure. The experiment results are plotted in \Cref{fig:diff_params}.
For \texttt{GLM-ES}, our baseline algorithm uses the following parameters: we use 100 iterations of gradient descent with step size 0.01, ensemble size $m = 10$, perturbation distribution $\mathcal{N}(0, \sigma_{R}^{2})$ with $\sigma_{R} = 0.1$, warm-up steps $\tau = 500$ and regularization $\lambda = 1.0$.
For \texttt{Neural-ES}, our baseline algorithm uses the following parameters: we use fully connected neural network with $L=3$ layers and width $N = 20$, we optimize the loss function using gradient descent with 100 steps and learning rate 0.01, ensemble size $m = 10$, perturbation distribution $\mathcal{N}(0, \sigma_{R}^{2})$ with $\sigma_{R} = 0.1$, warm-up steps $\tau = 50$ and regularization $\lambda = 1.0$.
To study the effects of $m$, $\lambda$, $\sigma_{R}$ and network structure, in each experiment in \Cref{fig:diff_params}, we change one hyper-parameter while keeping other parameters the same as the baseline algorithm.
In \Cref{fig:diff_params}, the baseline algorithm is plotted in gray solid line, additional experiment results are plotted in dashed lines.
We discuss the experiment results and effect of each hyper-parameter as follows.

(i) Ensemble size $m$.
For very small ensemble size $m$, it is likely that most models are trapped in sub-optimal arms, thus the cumulative regret grows fast for large $T$ compared to experiments with moderate or large ensemble sizes. We also observe higher variance for $m=5$, thus the algorithm becomes unstable at very small ensemble size.
For very large ensemble size ($m=50$), we observe marginal improvements on the cumulative regret compared to $m=10$. Since the computational cost grows linearly with $m$, moderate ensemble size such as $m=10$ suffices to reach competitive performance for \texttt{GLM-ES}.

(ii) Regularization $\lambda$.
For very small regularization $\lambda \sim 0$, the MLE estimates of the true parameter becomes unstable. Therefore, the predictions of reward become noisy and could cause unnecessary exploration.
For very large regularization, the loss function is dominated by the regularization term and the estimated parameter becomes very small. This could cause estimated rewards to be very similar across different arms, making identifying the optimal arm more difficult.
From the experiment results, $\lambda \sim 1$ is a proper choice, the cumulative regret becomes higher when regularization is too small or too large.

(iii) Perturbation variance $\sigma_{R}$.
For very small perturbation variance $\sigma_{R}$, the algorithm behaves greedy and the perturbed rewards for each model are nearly identical. Without sufficient exploration from perturbations, the algorithm is likely to be trapped in sub-optimal arms, resulting in fast growth of cumulative regret for large $T$.
For very large perturbations, the true reward is dominated by random perturbations and the model cannot effectively learn the true parameters of the environment. This results in large cumulative regret, especially at small $t$.
From the experiment results, $\sigma_{R} \sim 0.1$ is a proper choice of perturbation variance that provides sufficient exploration without introducing too much noise.

(iv) Network structure $L$ and $N$.
For very small neural networks, the model has limited capacity to approximate the true reward model, the cumulative regret grows fast compared to larger networks.
For very large networks, while the neural network can accurately approximate the true reward model, with limited data from interactions, the model could overfit the random noise and added perturbations. From experiments results, the performance is not significantly improved with larger network structures.
We also note that expanding the width of the neural network is generally more effective than adding more layers. This agrees with our theoretical analysis. As in \Cref{theorem:neural-es}, wider network can considerably reduce the cumulative regret.

\subsubsection{Performance Test with Different Anytime Schedules}

\begin{figure}[t]
\centering
    \subfigure[Logistic bandit with different anytime schedules (GLM-ES-$T_{0}$-$b$).]{\includegraphics[width=.45\textwidth]{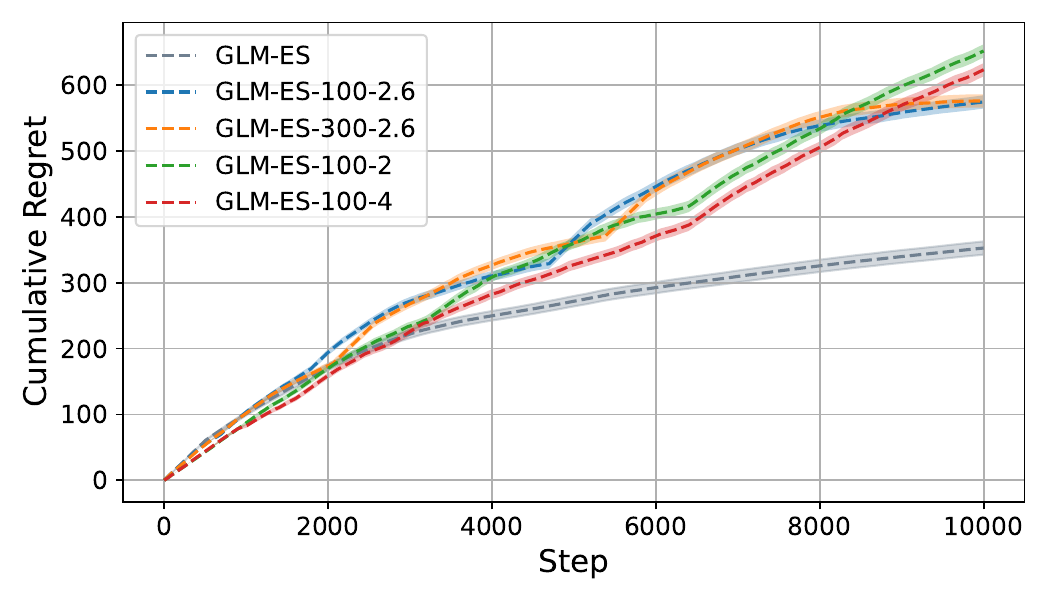}}
    \subfigure[Constant factor $l(T_{0}, b)$ in doubling trick.]{\includegraphics[width=.45\textwidth]{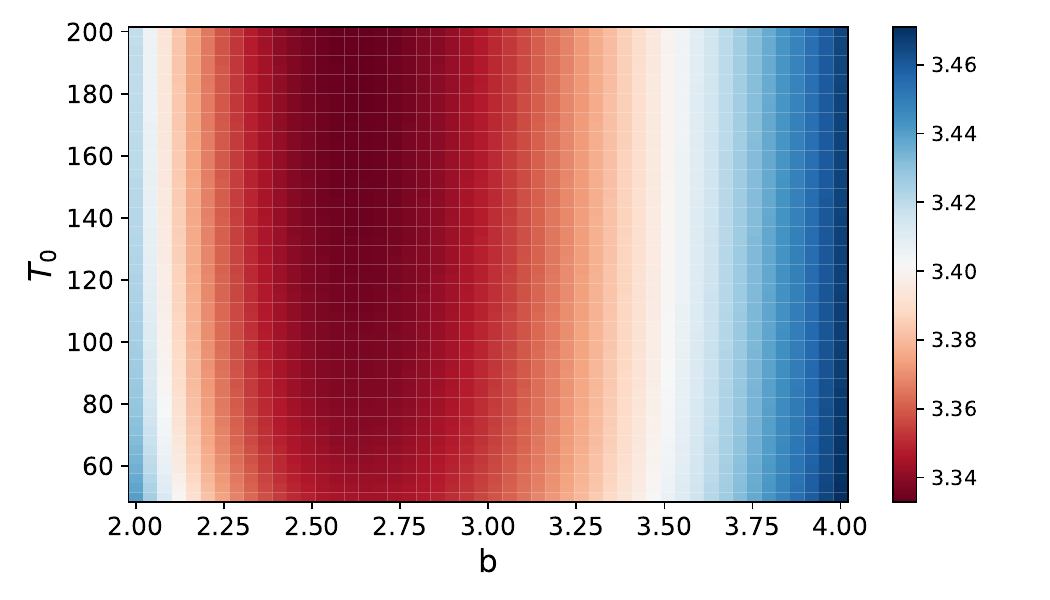}}
    \caption{Experiment results with different anytime schedules. \icml{The constant factor $l(T_{0}, b)$ is sensitive to $b$ but not sensitive to $T_{0}$, which is verified in the numerical results in logistic bandit settings. Therefore, we should keep $b$ at the optimal value $b \approx 2.6$, the choice of $T_{0}$ can be very flexible depending on practical requirements.}}
    \label{fig:anytime_schedule}
\end{figure}

Next, we consider the performance of anytime versions of ensemble sampling with different schedules $b$ and $T_{0}$. The experiment results are plotted in \Cref{fig:anytime_schedule}.
From theoretical analysis, for anytime versions of ensemble sampling, the regret bound increases by a constant factor $l(T_{0}, b)$:
\begin{align*}
    l(T_{0}, b) := \bigg(\frac{\text{log}(T_{0}(b-1) + 1)}{\text{log}(T_{0}(b-1))}\bigg)^{\delta} \times \frac{b^{\gamma}(b-1)^{\gamma}}{b^{\gamma} - 1},
\end{align*}
where $\gamma = 1/2$, $\delta = 3/2$.
The value of $l(T_{0}, b)$ is plotted in \Cref{fig:anytime_schedule}.
From this plot, the regret bound becomes smaller as we choose higher $T_{0}$ and $b = (3 + \sqrt{5})/2 \approx 2.6$. We also observe that $T_{0}$ has minor effect on the performance of the algorithm, while moving away from the optimal value of $b$ can considerably cause greater cumulative regret.
This agrees with the experiment results. Setting different $T_{0}$ results in nearly identical cumulative regret, while setting $b$ away from $(3 + \sqrt{5})/2$ has more significant impact on the performance of the algorithm.
Therefore, for practicality, we should keep $b$ at the optimal value $b = (3 + \sqrt{5})/2 \approx 2.6$ and set $T_{0}$ in the order of $10^{2}$.

\subsubsection{Performance Test on Larger Feature Spaces}

\begin{figure}[t]
\centering
    \subfigure[Logistic bandit with larger feature spaces.]{\includegraphics[width=.45\textwidth]{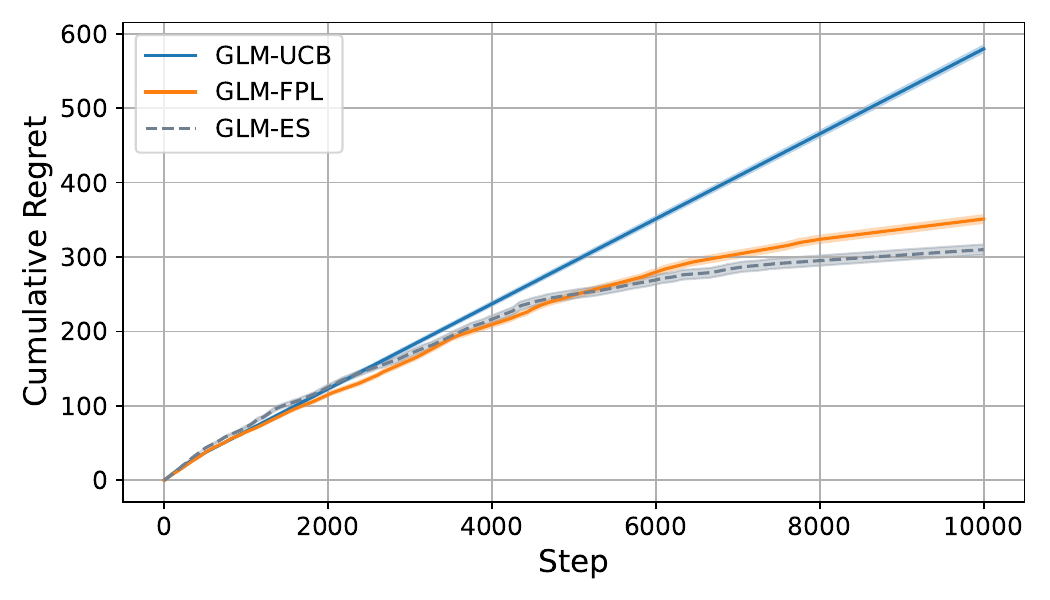}}
    \subfigure[Quadratic bandit with larger feature spaces.]{\includegraphics[width=.45\textwidth]{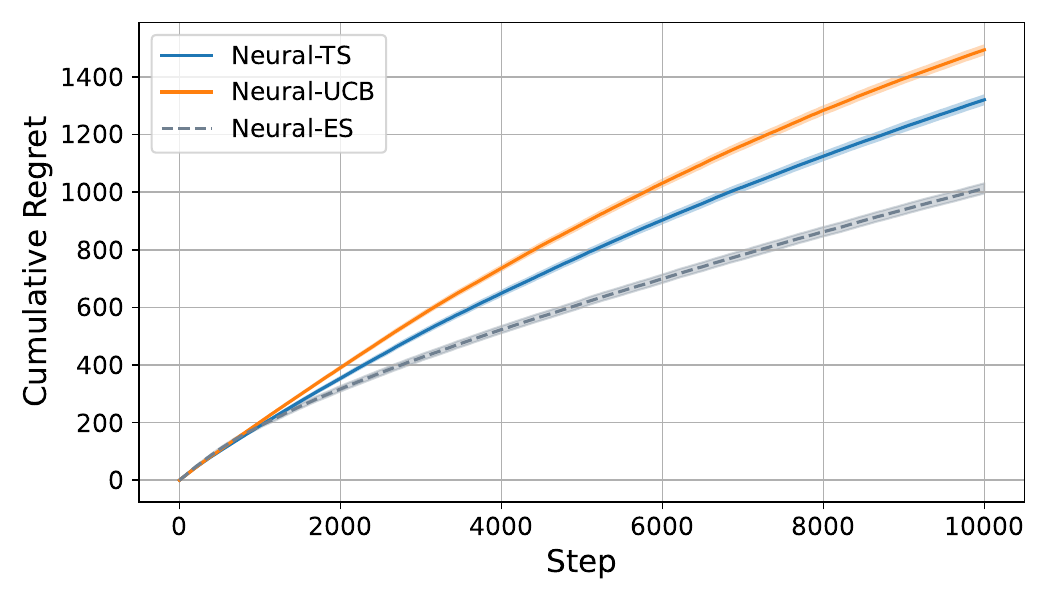}}
    \caption{Experiment results with larger feature spaces.}
    \label{fig:larger}
\end{figure}

We use synthetic environment with higher-dimensional features and more arms in the arm set $\mathcal{X}$.
Specifically, we use $d = 50$, $K = 500$ in the experiments. The results are plotted in \Cref{fig:larger}.
From the experiment results, the ensemble sampling algorithms remain competitive in larger feature spaces compared to baseline algorithms.

We further add experiments with a longer horizon $T=5 \times 10^{4}$ and four combinations of the number of arms and feature dimension: $K\in\{200,1000\}$ and $d\in\{100,500\}$. Each experiment result is averaged across 5 independently generated, fixed arm sets. We report the mean and standard error of cumulative regret at $t=T$ as follows, each result is computed with 50 runs.

\begin{table}[h]
\centering
\caption{Experiment results for linear bandit with large feature space.}
\begin{tabular}{lccc}
\toprule
Setting & LinTS & LinUCB & LinPHE \\
\midrule
$K=200, d=100$ & 991.51 ± 72.56 & 1080.93 ± 76.48 & 1033.28 ± 80.12 \\
$K=200, d=500$ & 1323.83 ± 120.32 & 1451.66 ± 107.57 & 1436.25 ± 121.22 \\
$K=1000, d=100$ & 1311.44 ± 110.92 & 1326.64 ± 96.21 & 1298.86 ± 121.51 \\
$K=1000, d=500$ & 2038.38 ± 186.04 & 2238.69 ± 181.33 & 2163.83 ± 166.40 \\
\bottomrule
\end{tabular}
\begin{tabular}{lcc}
\toprule
Setting & epsilon-greedy & LinES \\
\midrule
$K=200, d=100$ & 1456.60 ± 169.10 & 951.49 ± 75.62 \\
$K=200, d=500$ & 1855.27 ± 389.02 & 1272.87 ± 102.45 \\
$K=1000, d=100$ & 2240.90 ± 476.45 & 1194.22 ± 97.19 \\
$K=1000, d=500$ & 2879.40 ± 588.30 & 1811.42 ± 139.66 \\
\bottomrule
\end{tabular}
\end{table}

In all four larger settings, every structured linear-bandit method outperforms $\epsilon$-greedy, while Lin-ES achieves the lowest mean regret.
Compared to the previous setup where $K=50$, $d=20$, $T=10^{4}$, these additional experiments are considerably more challenging to learn due the following reasons.
(1) We increased $T$ from $10^{4}$ to $5 \times 10^{4}$. Since the cumulative regret of $\epsilon$-greedy scales linearly in $T$ while all other baseline algorithms scales as $\sqrt{T}$, $\epsilon$-greedy incurs considerably higher regret at longer horizons.
(2) We increased $K$ and $d$, the environment now requires more structured explorations to learn efficiently. Compared to $\epsilon$-greedy, other baseline algorithms achieved considerably lower cumulative regret. Increasing $K$ and $d$ raises the regret of all methods because more observations are required to estimate the unknown parameter.

\subsubsection{Performance Test with Large Ensemble Size}

We add large ensemble-size experiments for Lin-ES, GLM-ES, and Neural-ES, increasing $m$ from the default value 10 to 10,000. We report the cumulative regret at $t=T$ as follows, we use $T=10^{4}$ and each result is computed with 50 runs.

\begin{table}[h]
\centering
\caption{Experiment results with large ensemble size.}
\begin{tabular}{lccccccc}
\toprule
Setting & $m=10$ & $m=25$ & $m=50$ \\
\midrule
linear, Lin-ES, $K=200$, $d=100$ & 433.52 ± 38.89 & 421.75 ± 38.75 & 408.38 ± 33.73 \\
logistic, GLM-ES, $K=50$, $d=20$ & 369.39 ± 39.45 & 361.33 ± 34.56 & 365.96 ± 37.57 \\
quadratic, Neural-ES, $K=50$, $d=20$ & 1075.40 ± 54.34 & 1071.00 ± 44.98 & 1076.40 ± 40.48 \\
\midrule
\end{tabular}
\begin{tabular}{cccc}
$m=200$ & $m=1000$ & $m=10000$ & Best Baseline \\
\midrule
407.28 ± 39.02 & 405.69 ± 33.02 & 401.32 ± 28.40 & Lin-TS, 476.74 ± 33.84 \\
358.18 ± 33.81 & 347.36 ± 32.31 & 352.65 ± 30.81 & GLM-UCB, 401.48 ± 39.26 \\
1025.79 ± 42.47 & 1044.65 ± 35.39 & 1018.39 ± 34.33 & Neural-PHE, 1187.25 ± 40.70 \\
\bottomrule
\end{tabular}
\end{table}

Increasing $m$ from 10 to 10,000 reduces the mean regret by approximately $7.4\%$, $4.5\%$, and $5.3\%$ for Lin-ES, GLM-ES, and Neural-ES, respectively. Thus, a larger ensemble provides a modest benefit, consistent with increased ensemble diversity and reduced over-exploitation. However, the improvement largely saturates: beyond approximately $m=200$, the results fluctuate within a relatively narrow range and do not improve significantly. The observed non-monotonicity is expected from finite-sample variability rather than an adverse effect of larger ensembles. We also note that every tested ensemble size outperforms the corresponding baseline.

We emphasize that the theorem’s ensemble-size requirement is a conservative condition to ensure optimism, rather than a prediction of the empirically optimal $m$. The empirical results show that practical performance stabilizes much earlier than the theoretically required values on these instances.

\subsubsection{Performance Test with Large Arm Set}

We add a large arm-set study in logistic bandit with $d=50$, $T=10^{4}$, and $K\in\{500, 10000\}$. Crucially, we keep the ensemble size fixed at $m=25$ as $K$ increases twentyfold. We report the cumulative regret at $t=T$ as follows, each result is computed with 50 runs.

\begin{table}[h]
\centering
\caption{Experiment results with large arm set.}
\begin{tabular}{lcccc}
\toprule
Setting & GLM-TSL & GLM-UCB & GLM-FPL & GLM-ES \\
\midrule
$K=500$ & 525.64 ± 60.04 & 542.89 ± 50.72 & 448.39 ± 35.86 & 403.15 ± 27.78 \\
$K=10,000$ & 788.12 ± 55.37 & 727.09 ± 53.56 & 632.61 ± 37.87 & 556.61 ± 54.78 \\
\bottomrule
\end{tabular}
\end{table}

GLM-ES outperforms the logistic-bandit baselines at $K=500$ and $K=10,000$. The advantage therefore remains when $K$ increases twentyfold without increasing $m$.
These experiments support the practical implication of our revised analysis: for GLM-ES, the required ensemble size is governed by $d$ and $T$, $m=\Theta(d \log T)$, rather than by $K$.

\subsubsection{Comparison with Additional Baseline Algorithms}

We add NTK-UCB \citep{kassraie2022neural} to the quadratic bandit experiment under the same setting and evaluation protocol, with $K=50$, $d=20$, and $T = 10^{4}$. We report the cumulative regret at $t=T$ as follows, each result is computed with 50 runs.

\begin{table}[h]
\centering
\caption{Comparison with NTK-UCB algorithm in quadratic bandit.}
\begin{tabular}{lcc}
\toprule
Metric & NTK-UCB & Neural-ES \\
\midrule
Cumulative regret & 1349.78 ± 39.25 & 1020.63 ± 32.04 \\
\bottomrule
\end{tabular}
\end{table}

Compared to NTK-UCB, Neural-ES considerably reduces the mean cumulative regret while also exhibits slightly lower variability.
The new experiment therefore shows that the ensemble-based exploration mechanism performs favorably against explicit NTK-based optimism in the quadratic setting.

We also add EE-Net \citep{ban2022eenet}, NeuSquareCB, and NeuFastCB \citep{deb2024contextual} as baselines under the same two neural-bandit environments and evaluation protocol. The reported values are final cumulative regret at $T=10^{4}$, averaged over 50 repetitions.

\begin{table}[h]
\centering
\caption{Comparisons with additional baseline algorithms.}
\begin{tabular}{lcccc}
\toprule
Setting & EE-Net & NeuSquareCB & NeuFastCB & Neural-ES \\
\midrule
distance & 1198.73 ± 63.61 & 1102.65 ± 70.40 & 1087.98 ± 68.95 & 627.09 ± 57.78 \\
quadratic & 1211.95 ± 76.85 & 1116.24 ± 85.92 & 1121.50 ± 72.91 & 1050.01 ± 68.74 \\
\midrule
\end{tabular}
\begin{tabular}{lccc}
\midrule
Setting & Neural-TS & Neural-UCB & Neural-PHE \\
\midrule
distance & 926.19 ± 69.14 & 1269.63 ± 81.14 & 814.32 ± 71.70 \\
quadratic & 1151.73 ± 56.99 & 1173.64 ± 58.27 & 1149.88 ± 65.43 \\
\bottomrule
\end{tabular}
\end{table}

According to the experiment reuslts, EE-Net, NeuSquareCB and NeuFastCB achieves competitive cumulative regret in both settings, while Neural-ES achieves the lowest mean cumulative regret.
The empirical advantage is more pronounced in the distance environment and narrower in the quadratic environment.
These additional comparisons strengthen the empirical conclusion that Neural-ES is practically competitive across different exploration paradigms in neural bandit.

\subsection{Experiments on UCI Datasets}

\begin{figure}[t]
\centering
    \subfigure[UCI Shuttle Dataset.]{\includegraphics[width=.45\textwidth]{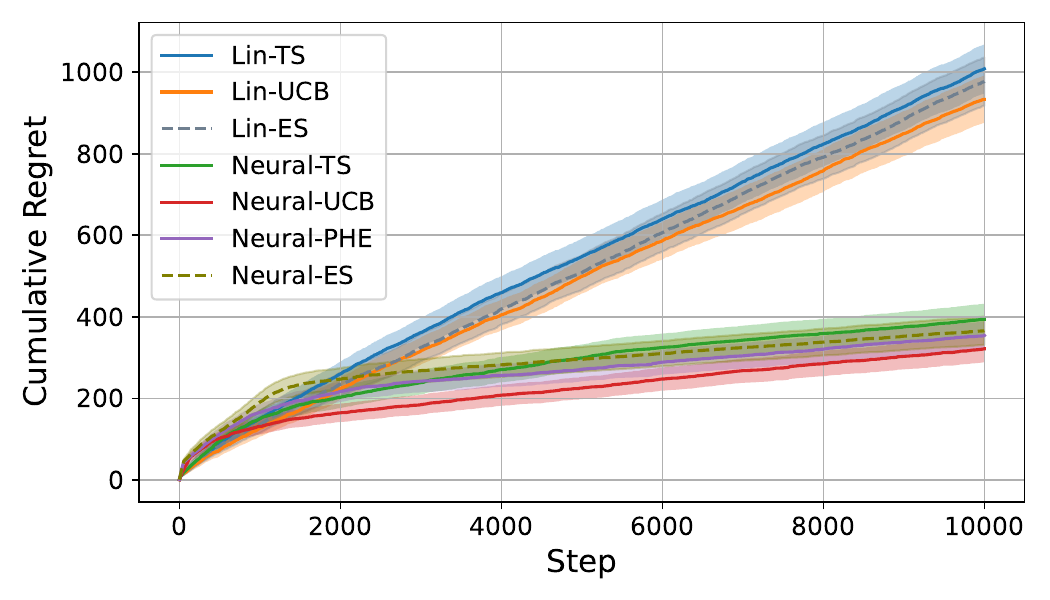}}
    \subfigure[UCI Mushroom Dataset.]{\includegraphics[width=.45\textwidth]{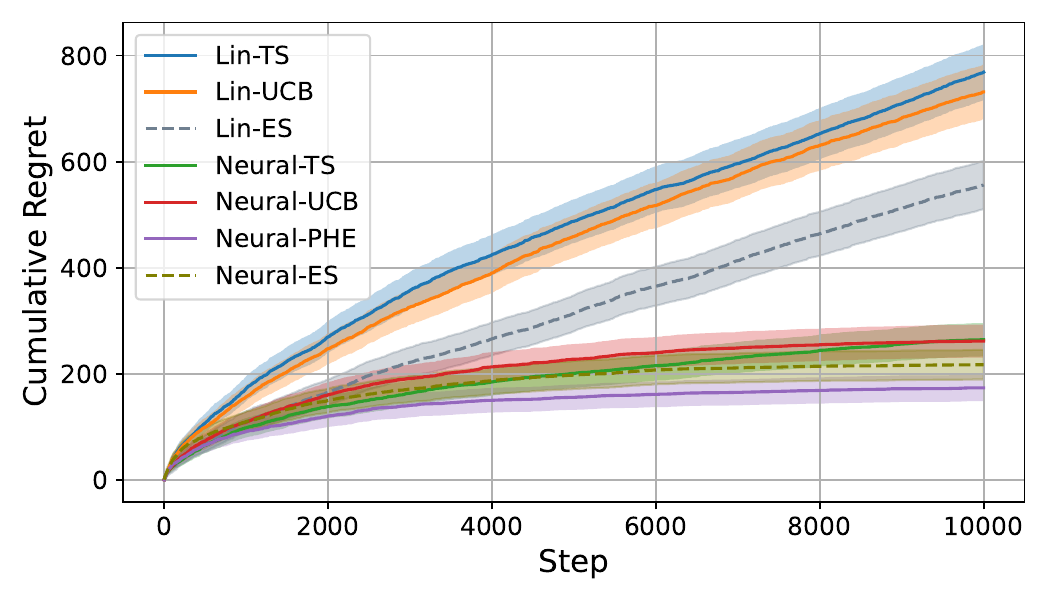}}
    \subfigure[UCI Adult Dataset.]{\includegraphics[width=.45\textwidth]{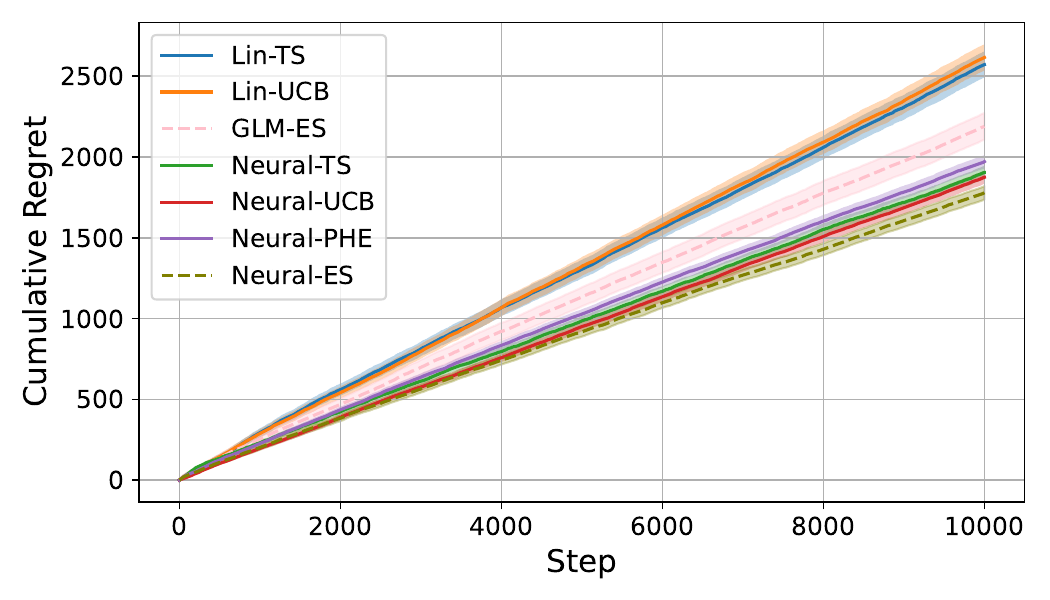}}
    \subfigure[UCI Covertype Dataset.]{\includegraphics[width=.45\textwidth]{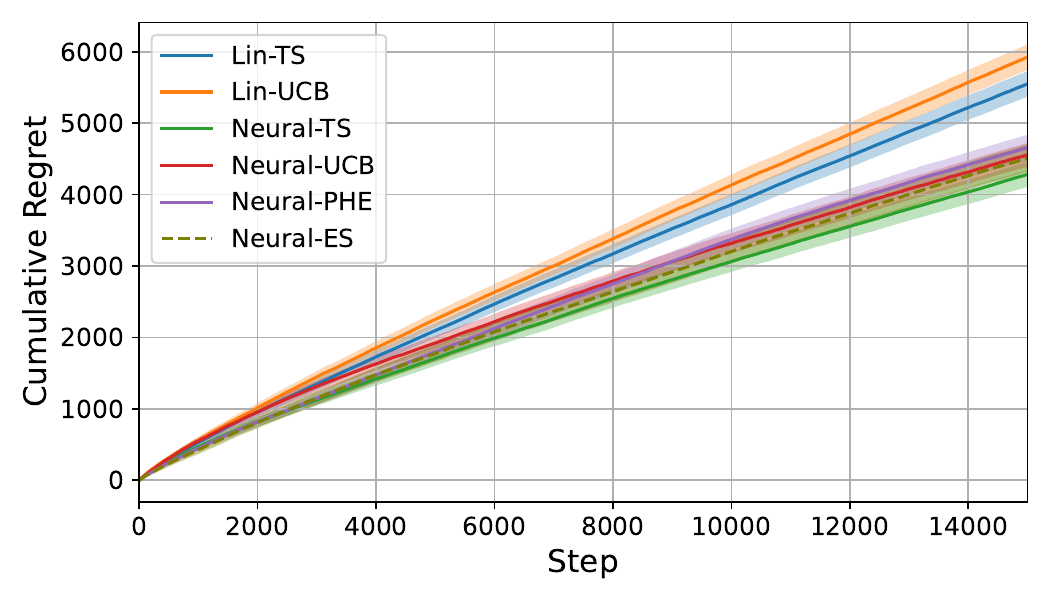}}
    \caption{Experiment results in UCI datasets. The results are averaged over 10 repetitions. The shaded area represents the standard deviation of the cumulative regret.}
    \label{fig:uci}
\end{figure}

We present the experiment results on classification tasks with the UCI datasets in \Cref{fig:uci}.
The context dimension and number of arms for each environment is listed below:

\begin{table}[h]
\centering
\caption{Context dimension and number of arms for UCI datasets.}
\label{tab:icml_example}
\begin{tabular}{lcccc}
\toprule
Environment & \texttt{Shuttle} & \texttt{Mushroom} & \texttt{Adult} & \texttt{Covertype} \\
\midrule
Context Dimension $d$        & 9 & 22 & 14 & 54 \\
Number of Arms $K$ & 7 & 2 & 2 & 7 \\
\bottomrule
\end{tabular}
\end{table}

The hyper-parameter setup for each model is the same as in \Cref{section:exp}. Compared to baseline algorithms \texttt{Lin-TS} and \texttt{Lin-UCB}, \texttt{Lin-ES} consistently achieves similar cumulative regret with reduced computational cost.
In these non-linear real-world datasets, linear models cannot capture the non-linearity effectively, thus algorithms designed for generalized linear bandit and neural bandit demonstrate considerably stronger empirical performance.
Compared to baseline algorithms \texttt{Neural-TS}, \texttt{Neural-UCB} and \texttt{Neural-PHE}, \texttt{Neural-ES} consistently achieved better or similar cumulative regret with significantly improved computational cost.

Overall, our experiments on UCI datasets verifies the competence and efficiency of ensemble sampling-based algorithms in real-world environments.

\bibliography{reference}
\bibliographystyle{ims}

\end{document}